\documentclass[twoside,11pt]{article}

\usepackage{xr}

\usepackage[english]{babel}

\usepackage[utf8]{inputenc}
\usepackage{amsmath}
\usepackage{amssymb}
\usepackage{amsthm}
\usepackage{physics}
\usepackage{textcomp}
\usepackage{gensymb}
\usepackage{verbatim}
\usepackage{amsfonts}
\usepackage{mathrsfs}
\usepackage{tikz-cd}
\usepackage{graphicx}
\usepackage{grffile} 
\usepackage{microtype}
\usepackage{centernot}
\usepackage{mathtools}
\usepackage{bm}
\usepackage{caption}
\usepackage{enumitem}
\usepackage{subcaption}
\usepackage{algorithm}
\usepackage[noend]{algpseudocode}
\usepackage{array}

\usepackage{pgfplots}
\pgfplotsset{compat=newest}
\usepackage{tikz}
\usepackage{multirow}
\usepackage{booktabs}

\usepackage[preprint]{jmlr2e}
\usepackage{cleveref}
\setcitestyle{authoryear,open={(},close={)}}

\usetikzlibrary{shapes, positioning, quotes, matrix}

\pgfplotsset{height=8cm, width=15cm,compat=1.9}

 \DeclareMathOperator*{\argmin}{arg\,min}
 \DeclareMathOperator*{\argmax}{arg\,max}
\newcommand{\Ex}{\mathbb{E}}
\newcommand{\R}{\mathbb{R}}

\newcommand*\diff{\mathop{}\!\mathrm{d}}

\newcommand{\statespace}{\mathcal{S}}
\newcommand{\actionspace}{\mathcal{A}}

\newcommand{\vpi}{V^\pi}
\newcommand{\qpi}{Q^\pi}

\newcommand{\Qhat}{{Q}}

\newcommand{\KL}[2]{\mathrm{KL}^{#1}_{#2}}
\newcommand{\RKL}[2]{\mathrm{RKL}^{#1}_{#2}}
\newcommand{\FKL}[2]{\mathrm{FKL}^{#1}_{#2}}
\newcommand{\policyparams}{\theta}
\newcommand{\boltzmannQ}{\mathcal{B}_\tau Q}
\newcommand{\entropy}{\mathcal{H}}
\newcommand{\pinew}{{\pi_\mathrm{new}}}
\newcommand{\piold}{{\pi_\mathrm{old}}}

\newcommand{\done}{\texttt{done}}
\newcommand{\Diag}{\mathrm{Diag}}
\newcommand{\intd}{\mathrm{d}}
\newcommand{\Etraj}{\mathbb{E}_{\rho_0}^{\pi_{new}}}

\newcommand{\defeq}{:=}

\usepackage{thmtools}
\usepackage{thm-restate}
\usepackage{cancel}
\usepackage{svg}

\newcommand{\BQ}{\boltzmannQ}
\newcommand{\BQt}{\boltzmannQ_\tau}
\newcommand{\Qt}{Q_{\tau}}
\newcommand{\cS}{\statespace}
\newcommand{\cA}{\actionspace}
\newcommand{\Qpitau}[1]{Q^{#1}_\tau}
\newcommand{\Vpitau}[1]{V^{#1}_\tau}
\newcommand{\cSA}{\cS\times\cA}
\newcommand{\advold}{A_\tau^{\piold}}
\newcommand{\advnew}{A_\tau^{\pinew}}
\newcommand{\qold}{Q_\tau^{\piold}}
\newcommand{\qnew}{Q_\tau^{\pinew}}
\newcommand{\vold}{V_\tau^{\piold}}
\newcommand{\vnew}{V_\tau^{\pinew}}
\newcommand{\bpinew}{\boldsymbol{\pi}_{new}}
\newcommand{\bpiold}{\boldsymbol{\pi}_{old}}
\newcommand{\bpi}{\boldsymbol{\pi}}

\newcommand{\BQold}{\mathcal{B}_{\tau}Q_{\tau}^{\pi_\mathrm{old}}}
\newcommand{\bv}{\mathbf{q}} 
\newcommand{\Ed}[1]{\Ex_{d^{\pinew}}\left[#1\right]}
\newcommand{\EdBigg}[1]{\Ex_{d^{\pinew}}\Bigg[#1\Bigg]}
\newcommand{\EdBiggAct}[1]{\Ex_{d^{\pinew},\pinew}\Bigg[#1\Bigg]}
\newcommand{\EdBiggActOld}[1]{\Ex_{d^{\pinew},\piold}\Bigg[#1\Bigg]}

\newcommand{\eqcomment}[1]{\quad\quad\quad\triangleright \text{#1}}

\newcommand{\BlackBox}{\rule{1.5ex}{1.5ex}}  
\renewenvironment{proof}{\par\noindent{\bf Proof\ }}{\hfill\BlackBox\\[2mm]}
 
\newtheorem{theorem}{Theorem}
\newtheorem{lemma}[theorem]{Lemma}

\newtheorem{corollary}[theorem]{Corollary}
\newtheorem{definition}[theorem]{Definition}

\newtheorem{assumption}[theorem]{Assumption}

\jmlrheading{1}{2021}{x-xx}{x/xx}{xx/xx}{xxx}{Alan Chan, Hugo Silva, Sungsu Lim, Tadashi Kozuno, A. Rupam Mahmood, Martha White}

\ShortHeadings{Greedification Operators for Policy Optimization}{Chan, Silva, Lim, Kozuno, Mahmood, and White}
\firstpageno{1}

\begin{document}

\title{Greedification Operators for Policy Optimization: Investigating Forward and Reverse KL Divergences}

\author{\name Alan Chan \email achan4@ualberta.ca \\
        \name Hugo Silva \email hugoluis@ualberta.ca\\
        \name Sungsu Lim \email sungsu@ualberta.ca \\
        \name Tadashi Kozuno \email kozuno@ualberta.ca \\        
        \name A. Rupam Mahmood \email armahmood@ualberta.ca\\
        \name Martha White \email whitem@ualberta.ca \\
       \addr Department of Computing Science, Alberta Machine Intelligence Institute (Amii)\\
       University of Alberta\\
       Edmonton, Alberta, Canada
       }

\editor{Editors}

\maketitle

\begin{abstract}%
Approximate Policy Iteration (API) algorithms alternate between (approximate) policy evaluation and (approximate) greedification. Many different approaches have been explored for approximate policy evaluation, but less is understood about approximate greedification and what choices guarantee policy improvement. In this work, we investigate approximate greedification when reducing the KL divergence between the parameterized policy and the Boltzmann distribution over action values. In particular, we investigate the difference between the forward and reverse KL divergences, with varying degrees of entropy regularization; these are chosen because they underlie many existing policy optimization approaches, as we highlight in this work. We show that the reverse KL has stronger policy improvement guarantees, and that reducing the forward KL can result in a worse policy. We also demonstrate, however, that a large enough reduction of the forward KL can induce improvement under additional assumptions. Empirically, we show on simple continuous-action environments that the forward KL can induce more exploration, but at the cost of a more suboptimal policy. No significant differences were observed in the discrete-action setting or on a suite of benchmark problems. This work provides novel theoretical and empirical insights about the forward KL and reverse KL for greedification, and clear next steps for understanding and improving our policy optimization algorithms. 
\end{abstract}


\section{Introduction}

A canonical approach to learn policies in reinforcement learning (RL) is Policy Iteration (PI). PI interleaves policy evaluation---understanding how a policy is currently performing by computing a value function---and policy improvement---making the current policy better based on the value function. The policy improvement step is sometimes called the \textit{greedification} step, because typically the policy is set to a greedy policy. That is, the policy is set to take the action that maximizes the current action-value function, in each state. 
In the tabular setting, this procedure is guaranteed to result in iteratively better policies and converge to the optimal policy \citep{bertsekas2019reinforcement}. The greedification step can also be soft, in that some some probability is placed on all other actions. In certain such cases, like with entropy regularization, PI converges to the optimal soft policy \citep{geist2019theory}.

Practically, however, it is not always feasible to perform each step to completion. Approximate PI (API) \citep{bertsekas2011approximate,scherrer2014approximate} allows for each step to be done incompletely, and still maintain convergence guarantees. The agent can perform an approximate policy evaluation step, where it obtains an improved estimate of the values without achieving the true values. The agent can also only perform approximate greedification by updating the policy to be closer to the (soft) greedy policy under the current values. The first approximation underlies algorithms like Sarsa, where the action-value estimates are updated with one new sample, upon which the new policy is immediately set to the soft greedy policy (approximate evaluation, exact greedification). 

It is not as common to consider approximate greedification. One of the reasons is that obtaining the (soft) greedy policy is straightforward for discrete actions.\footnote{Even under discrete actions, there is a reasonable argument that approximate greedification may be preferable, even if exact greedification is possible. We typically only have estimates of the value function, and exact greedification on the estimates can potentially harm the agent's performance \citep{kakade2002approximately}. Further, having an explicit parameterized policy, even under discrete actions, can be beneficial to avoid an effect known as delusional bias \citep{lu2018non}, where directly computing the greedy value in action-value updates can result in inconsistent action choices.} For continuous actions, however, obtaining the greedy action for given action-values is non-trivial, requiring the computation of the maximum value (or supremum) over the continuous domain. Some methods have considered optimization approaches to compute it, to get continuous-action Q-learning methods \citep{amos2016icnn,gu2016naf,kalashnikov2018qt,ryu2019caql,gu2019sample}. It is more common, though, to instead turn to policy gradient methods and learn a parameterized policy. 

This switch to parameterized policies, however, does not evade the question of how to perform approximate greedification. Indeed, many policy gradient (PG) methods---those approximating a gradient of the policy objective---can actually be seen as instances of API.
The connection between PG and API arises because efficient implementation of PG methods requires the estimation of a value function. 
Actor-critic methods estimate value functions through temporal-difference methods \citep{sutton2018reinforcement}. We explicitly show in this work that the basic actor-critic method can be seen as API with a particular approximate greedification step. In general, numerous papers have already linked PG methods to policy iteration \citep{sutton2000policy,kakade2002approximately,perkins2002existence,perkins2003convergent,wagner2011reinterpretation,wagner2013optimistic,scherrer2014local,bhandari2019global,ghosh2020operator,vieillard2020leverage}, including recent work connecting maximum-entropy PG and value-based methods \citep{o2016combining,nachum2017bridging, schulman2017equivalence, nachum2019algaedice}.

Moreover, most so-called PG methods used in practice are better thought of as API methods, rather than as PG methods. Many PG methods use a biased estimate of the policy gradient. The correct state weighting is not used in either the on-policy setting \citep{thomas2014bias,nota2019policy} or the off-policy setting \citep{imani2018off}. Additionally, the use of function approximators to estimate action-values generally results in biased gradient estimates without any further guarantees, such as a compatibility condition \citep{sutton2000policy}. This bias can be reduced by using $n$-step return estimates for the policy update, but is not completely removed. Understanding approximate greedification within API, therefore, is one direction for better understanding the PG methods actually used in practice.  

The question is what approximate greedification approach should be used. One answer is to define a target policy, that would provide policy improvement if we could represent it, and learn a policy to approximate that target.
The classical policy improvement theorem \citep{sutton2018reinforcement} guarantees that if a new policy is greedy with respect to the action-value function of an old policy, then the new policy is at least as good as the old policy. For parameterized policies (e.g., neural-network policies), exact greedification in each state is rarely possible as not all policies will be representable by a given function class. 
Instead of the greedy policy, we can use the Boltzmann distribution over the action values as the target policy, which is known to provide policy improvement \citep{haarnoja2018soft}. Of particular importance for us and as we show in this work, stepping towards this target policy---approximate greedification---on average across the state space---rather than for each state---is also guaranteed to provide policy improvement. As such, it is a reasonable target policy to explore for approximate greedification under function approximation.  

We explore minimizing the Kullback-Leibler (KL) divergence to this target policy.\footnote{We use the KL divergence to project the target policy to the space of parameterized policies. This differs from using the KL (or Bregman) divergence to regularize policy updates to force a new policy be close to a previous one \citep{peters2010relative,schulman2015trust,abdolmaleki2018maximum,geist2019theory,vieillard2020leverage}. While such regularization changes the target policy and may confer benefits to algorithms \citep{schulman2015trust,vieillard2020leverage}, there still remains a question of how to project the target policy, which is the main focus of the present paper.} Other options are possible, such as total variation or Wasserstein distance. We focus on the KL because it underlies many existing methods---as has been previously shown \cite{vieillard2020leverage} and as we more comprehensively summarize in Section \ref{sec:approx-greed-kl}.
Further, the KL divergence is a convenient choice because stochastic estimation of this objective only requires the ability to sample from the distributions and evaluate them at single points. 

Even though the KL has been used, it is as yet unclear whether to use the reverse or the forward KL divergence, here also called RKL and FKL respectively. That is, should the first argument of the KL divergence be policy $\pi$, or should it be the Boltzmann distribution over the action values? \citet{neumann2011variational} argues in favour of the reverse KL divergence to obtain a most cost-average policy, whereas \citet{norouzi2016reward} uses the forward KL divergence to induce a more exploratory policy (i.e., more diverse state visitation distribution). 

The typical default is the reverse KL. The reverse KL without entropy regularization corresponds to a standard Actor-Critic update and is easy to compute, as we show in \Cref{sec:approx-greed-kl}. More recently, it was shown that the reverse KL guarantees policy improvement when the KL can be minimized separately for each state \citep[p.~4]{haarnoja2018soft}; this finding motivated the development of Soft-Actor Critic. Regret analyses involving Bregman divergences, like for mirror descent \citep{orabona2019modern,shani2019adaptive}, also tend to imply results for the reverse KL, but not for the forward KL.

Some work, though, has used the forward KL \citep{norouzi2016reward,nachum2016improving,agarwal2019learning,vieillard2019deep}, including implicitly some work in classification for RL \citep{lagoudakis2003reinforcement,lazaric2010analysis,farahmand2015classificationbased}. For contextual bandits, \citet{chen2019surrogate} showed improved performance when using a surrogate, forward KL objective for the smoothed risk. Others used the forward KL to prevent mode collapse, given that the forward KL is mode-covering \citep{agarwal2019learning,mei2019principled}. 

Though both have been used and advocated for, there is no comprehensive investigation into their differences for approximate greedification. The closest work is \citet{neumann2011variational}, but they do KL divergence reduction in the context of EM-based policy search using the variational inference framework, whereas we frame the problem as approximate policy iteration, which leads to different optimization processes and cost functions. Their reverse KL target, for example, is a reward weighted trajectory distribution, which is different from the Boltzmann distribution we use here and they minimize the KL divergence with respect to the variational distribution, while we minimize it directly with respect to the policy. Their work does not provide any theoretical results and their experimental settings are limited to single step decision making, whereas we experiment on sequential decision making. 

The goal of this work is to investigate the differences between using a forward or reverse KL divergence, under entropy regularization, for approximate greedification. We ask, given that we optimize a policy to reduce either the forward or the reverse KL divergence to a Boltzmann distribution over the action values, what is the quality of the resulting policy?
We provide some clarity on this question with the following contributions. 
\begin{enumerate}[leftmargin=12pt]
\setlength{\itemsep}{0pt}
    \item We highlight four choices for greedification: forward KL (FKL) or reverse KL (RKL) to a Boltzmann distribution on the action-values, with or without entropy regularization. We show that many existing methods can be categorized into one of these four quadrants, and particularly show that the standard Actor-Critic update corresponds to using the RKL.
    \item We extend the policy improvement result for the RKL \citep{haarnoja2018soft} in two ways. (a) Instead of reducing the RKL for all states, we only need to reduce it on average under a certain state-weighting. (b) We characterize improvement under approximate action-values, rather than only exact action-values. 
    \item We further extend our theoretical results under a condition where the policy does not change too much; these results provide an extension of the seminal improvement results for conservative policy iteration \citep[Theorem 4.1]{kakade2002approximately}, to parameterized policies using gradient descent on the RKL for greedification.   
    \item We show via a counterexample that merely reducing the FKL is not sufficient to guarantee improvement and discuss additional sufficient conditions to guarantee improvement.
    \item We investigate optimization differences in small MDPs, and find that, particularly under continuous actions, (a) the RKL can converge faster, but sometimes to suboptimal local minima solutions, but (b) the optimal solution of the FKL can be worse than the corresponding RKL, particularly under higher entropy regularization.
    \item In a maze environment, with neural network function approximation, we show that the FKL promotes more exploration under continuous actions, by maintaining a higher variance in the learned policy, but for discrete actions, exploration is very similar for both. 
\end{enumerate}

In addition to these carefully controlled experiments, we tested the approaches in benchmark environments. We found that performance between the two was similar.
We hypothesize that the reason for this outcome is that the
action-values and the corresponding Boltzmann policy are largely unimodal for the benchmark problems; bigger differences should arise for the multi-modal setting. 
We conclude the work with a discussion about open questions and key next steps, including how to leverage these insights about FKL and RKL to potentially obtain improved policy optimization algorithms and theory.

\section{Problem Formulation}

We formalize the reinforcement learning problem \citep{sutton2018reinforcement} as a {Markov Decision Process} (MDP): a tuple $(\statespace, \actionspace, \gamma, r, p)$ where $\statespace$ is the state space; $\actionspace$ is the action space; $\gamma \in [0,1]$ is the discount factor; $r : \statespace \times \actionspace \to \R$ is the reward function; and, for every $(s, a) \in \statespace \times \actionspace$, $p(\cdot \mid s, a)$ gives the conditional transition probabilities over $\statespace$. A \textit{policy} is a mapping $\pi : \statespace \to \Delta_\actionspace$, where $\Delta_\actionspace$ is the space of probability distributions over $\actionspace$. 
At every discrete time step $t$, an agent observes a state $S_t$, from which it draws an action from its policy: $A_t \sim \pi(\cdot \mid S_t)$. The agent sends the action $A_t$ to the environment, from which it receives the reward signal $r(S_t, A_t)$ and the next state $S_{t + 1}$. 

In this work we focus on the episodic problem setting, where the goal of the RL agent is to maximize the \textit{expected return}---the expectation of a discounted sum of rewards---from the set of start states. 
To formalize this goal, we define the \textit{value function} $\vpi$ for policy $\pi$ as
\begin{align*}
    \vpi(s) := \Ex_{\pi}\left[\sum_{k = 0}^\infty \gamma^k r(S_k, A_k) \mid S_0 = s \right].
\end{align*}
The expectation above is over the trajectory $(S_0, A_0, S_1, A_1, \cdots)$ induced by $\pi$ and the transition kernel $p$. For simplicity, we omit $p$ in the subscript, because the expectation is always according to $p$. We similarly define the \textit{action-value function}:
\begin{align*}
    \qpi(s, a) &:= \Ex_{\pi}\left[\sum_{k = 0}^\infty \gamma^k r(S_k, A_k) \mid S_0 = s, A_0 = a \right] 
    = r(s, a) + \gamma \,\Ex[V^\pi(S') | S = a, A = a] \label{eq:short-q-def}
    .
\end{align*}

A common objective in policy optimization is the value of
the policy $\pi_\policyparams$ averaged over the start state distribution $\rho_0$: 
\begin{equation}\label{eq:policy-obj}
    \eta(\pi_\policyparams) \defeq \int_{\statespace} \rho_0(s) \int_{\actionspace} \pi_\policyparams(a | s) Q^{\pi_\policyparams}(s, a) \, \diff a \, \diff s.
\end{equation}
The {policy gradient theorem} gives us the gradient of $J(\policyparams)$ \citep{sutton2000policy},
\begin{align}\label{eq:policy-gradient-thm}
    \nabla_\policyparams \eta(\pi_\policyparams) &= \frac{1}{1-\gamma}\int_{\statespace} d^{\pi_\policyparams}(s) \int_{\actionspace} Q^{\pi_\policyparams}(s, a) \nabla_\policyparams \pi_\policyparams(a \mid s)\, \diff a\, \diff s,
\end{align}
where $d^{\pi_\policyparams}(s) \defeq (1 - \gamma) \sum_{t = 0}^\infty \gamma^t p(S_t = s \mid S_0 \sim \rho_0)$ is the {normalized discounted state visitation distribution}. 

Because we do not have access to $Q^{\pi_\policyparams}$, we instead approximate. For example, in REINFORCE \citep{williams1992simple}, a sampled return from $(s,a)$ is used as an unbiased estimate of $Q^{\pi_\policyparams}(s,a)$. This method, however, assumes on-policy returns and tends to be sample inefficient.
Commonly, a biased but lower-variance choice is to use a learned estimate $\Qhat$ of $Q^{\pi_\policyparams}$, obtained through policy evaluation algorithms like SARSA \citep{sutton2018reinforcement}. In these Actor-Critic algorithms, the actor---the policy---updates with a (biased) estimate of the above gradient, given by this $\Qhat$---the critic. 

\begin{figure}[t]
    \centering
    \includegraphics[width=1.0\linewidth]{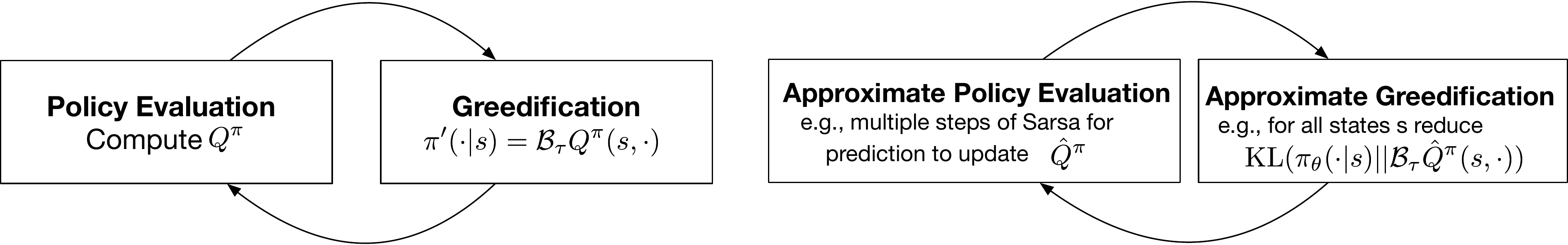}
    \vspace{-0.4cm}
    \caption{\textbf{Contrasting Policy Iteration (PI) and Approximation PI (API)}, left and right respectively. \textbf{[Left-hand Side] In PI}, the policy evaluation step and the greedification step are done exactly. In the policy evaluation step, the action values for the current policy $\pi$ are computed. In the greedification step---also called the policy improvement step---the policy $\pi'$ is set to the (soft) greedy policy. Then this $\pi'$ is handed to the policy evaluation step, as the new $\pi$, the new action values are computed, and this process repeats. In this work, the soft greedy policy that we consider is the Boltzmann policy, $\BQ^\pi$, defined in Equation \eqref{eq:boltzmann-q}. We discuss why this choice for greedification provides policy improvement in \Cref{sec:targetpolicy}. \\\textbf{[Right-hand Side] In API}, we use the same target policy for greedification---namely the Boltzmann policy---but we only approximate this target policy. We investigate minimizing---or at least reducing---a KL divergence between a parameterized $\pi_\theta$ to the target policy for this approximate greedification. 
    This learned $\pi_\theta$ corresponds to the new policy $\pi'$ that is handed back to the approximate policy evaluation step. The approximate greedification step can fully minimize the KL or only reduce it. There is a similar choice for the approximate policy evaluation step: we can either obtain the best approximate action-values with a batch algorithm like least-squares TD, or simply improve the estimate from the existing estimate using multiple stochastic updates to the action-values under the new policy.}
    \label{fig_api}
        \vspace{-0.4cm}
\end{figure}

This Actor-Critic procedure with learned $\Qhat$ can be interpreted as Approximate Policy Iteration (API). API methods alternate between approximate policy evaluation to obtain a new $\Qhat$ and approximate greedification to get a policy $\pi$ that is more greedy with respect to $\Qhat$. We depict this approach in Figure \ref{fig_api}, and contrast it to PI. As we show in the next section, the gradient in Equation \eqref{eq:policy-gradient-thm} can be recast as the gradient of a KL divergence to a policy peaked at maximal actions under $\Qhat$; reducing this KL divergence updates the policy to increase the probability of these maximal actions, and so become more greedy with respect to $\Qhat$. Under this API view, we obtain a clear separation between estimating $\Qhat$ and greedifying $\pi$. We can be agnostic to the strategy for updating $\Qhat$---we can even use soft action values \citep{ziebart2010modeling} or Q-learning \citep{watkins1992q}---and focus on answering: for a given $\Qhat$, how can we perform an approximate greedification step? 

This work focuses on understanding the differences between using the forward and reverse KL divergences towards the Boltzmann policy for approximate greedification. We explicitly define each of these approaches in the next section. Other choices for approximate greedification are possible---other target policies and other divergences or metrics---but we constrain our investigation to a feasible scope. In the next section, we further motivate why we investigate these approaches for approximate greedification, and later summarize how these variants underlie a variety of policy optimization methods.

\section{Approximate Greedification}\label{sec:greedification}


In this section, we formalize how to do approximate greedification. First, we discuss an appropriate choice for the (soft) greedy policy $\pi_{target}$. Given access to such a policy $\pi_{\text{target}}$, we can update our existing policy $\pi$ to be closer to $\pi_{\text{target}}$, using the KL divergence. The KL divergence, however, is not symmetric and has an entropy parameter $\tau$, resulting in four variants. We present these four variants that we use throughout the paper, and derive the updates under each choice. Finally, we discuss the importance of the state weighting in the final objective, which weights the divergence to the target policy in each state.

\subsection{Defining a Target Policy}\label{sec:targetpolicy}

A reasonable choice for the target policy is the Boltzmann distribution, as we motivate in this section. The Boltzmann distribution we use here is also common in pseudo-likelihood methods \citep{kober2008policy,neumann2011variational,levine2018reinforcement}, ensuring that one has a target distribution based on the action-values. 

Let $\Qhat$ be an action-value function estimate. For a given $\tau > 0$, the Boltzmann distribution $\boltzmannQ(s,\cdot)$ for a state $s$ is defined as
\begin{equation}\label{eq:boltzmann-q}
    \boltzmannQ(s, a) := \frac{\exp(\Qhat(s, a)\tau^{-1})}{Z(s)} \quad\quad \text{for } Z(s) := \int_\actionspace \exp(\Qhat(s, b)\tau^{-1}) \, \diff b
\end{equation}
where $Z(s)$ is known as the partition function.
The definition in \Cref{eq:boltzmann-q} does not depend upon a particular policy: we can input any $\Qhat$ that is a function of states and actions. For larger $\tau$, the Boltzmann distribution is more stochastic: it has higher entropy.

This distribution can be derived by solving for the entropy-regularized greedy policy on $\Qhat$. To see why, recall the definition of the {entropy} of a distribution, which captures how spread out the distribution is:
\begin{equation*}
    \entropy(\pi(\cdot \mid s)) \defeq - \int_\actionspace \pi(a \mid s) \log \pi(a \mid s)\, \intd a.
\end{equation*}
The higher the entropy, the less the probability mass of $\pi(\cdot \mid s)$ is concentrated in any particular area. 
Let $\mathcal{F}$ be the set of all nonnegative functions on $\actionspace$ that integrate to 1. At a given state, the entropy-regularized greedy policy is given by 
\begin{align}\label{eq:target policy definition}
    \pi_{\text{target}}(\cdot \mid s) &:= \argmax_{p \in \mathcal{F}} \int p(a) (Q(s, a) - \tau \log p(a))\, \diff a.
\end{align}
The integrand can be rewritten as follows:
\begin{align*}
    p(a) \left( Q(s, a) - \tau \log p(a) \right) &= \tau p(a) \log \frac{\exp (Q(s, a) / \tau) }{p(a)} \\
        &= \tau p(a) \log \frac{\boltzmannQ (s, a)}{p(a)} + \tau p(a) \log \int \exp \left( \frac{Q(s, a)}{\tau} \right)\, da\,.
\end{align*}
The right summand becomes a constant when integrated, so $\pi_{\text{target}}(\cdot \mid s)$ can be rewritten as
\begin{align*}
    \pi_{\text{target}}(\cdot \mid s)  &= \argmin_{p \in \mathcal{F}} \int p(a) \log \frac{p(a)}{\boltzmannQ (s, a)}\, \diff a 
        = \boltzmannQ (s, \cdot).
\end{align*}
where the first term is actually the KL divergence between $p$ and $\boltzmannQ (s, a)$, which is minimized by setting $p$ to $\boltzmannQ (s, a)$.  

The use of entropy-regularization avoids obtaining deterministic, greedy policies that can be problematic in policy gradient methods. Instead, this approach allows for soft greedification, giving the most greedy policy under the constraint that the entropy of the policy remains non-negligible. This policy can be shown to provide guaranteed policy improvement, but under a different criteria: according to soft value functions \citep{ziebart2010modeling}.

Soft value functions are value functions where an entropy term is added to the reward. 
\begin{equation*}
    V^{\pi}_\tau(s) \defeq  \Ex_\pi\left[ \sum_{k = 0}^\infty \gamma^k \left[ r(S_k, A_k) + \tau \entropy(\pi(\cdot|S_k))\right] \mid S_0 = s \right] 
\end{equation*}
We can define the soft action-value function in terms of the soft value function. 
\begin{align*}
    Q^{\pi}_\tau(s,a) \defeq  r(s, a) + \gamma \Ex[V^\pi_\tau(S') | S = s, A = a]
\end{align*}
We can also write the state-value function in terms of the action-value function.
\begin{align*}
    V^\pi_\tau(s) = \Ex_\pi[Q^\pi_\tau(s, A) - \tau \log \pi(A \mid s)].
\end{align*}
These soft value functions corresponds to a slightly different RL problem described as entropy-regularized MDPs \citep{geist2019theory}.

For these soft value functions, we can guarantee policy improvement under greedification with the Boltzmann distribution. 
If we set $\pi'(\cdot | s) = \boltzmannQ^{\pi}(s, \cdot)$ for all $s \in \statespace$, then $Q^{\pi'}_\tau(s,a) \ge Q^{\pi}_\tau(s,a)$ for all $(s,a)$ \citep[Theorem 4]{haarnoja2017reinforcement}. This parallels the classical policy improvement result in policy iteration \citep{sutton2018reinforcement}. This guaranteed policy improvement is a motivation for using $\boltzmannQ$ as a target policy for greedification. We extend this policy improvement result to hold under weaker conditions in Section \ref{sec:theory}.

\subsection{Approximate Greedification with the KL}\label{sec:approx-greed-kl}

In this section we discuss how to use a KL divergence to bring $\pi_\policyparams$ closer to $\boltzmannQ$. One might wonder why we do not just set $\pi(\cdot | s) = \boltzmannQ(s, \cdot)$. Indeed, for discrete action spaces, we can draw actions from $\boltzmannQ(s, \cdot)$ easily at each time step. However, for continuous actions, even calculating $\boltzmannQ(s, \cdot)$ requires approximating a generally intractable integral. Furthermore, even in the discrete-action regime, using $\boltzmannQ$ might not be desirable as $Q$ is usually just an action-value \textit{estimate}. Greedifying with respect to an action-value estimate does not guarantee greedification with respect to the action-value.  


In this work we focus on the KL divergence to measure the difference between $\pi$ and $\pi_{target}$. 
Given two probability distributions $p, q$ on $\actionspace$, the KL divergence between $p$ and $q$ is 
\begin{equation}\label{eq:kl-div}
  \KL{p}{q} \defeq \int_\actionspace p(a) \log\frac{p(a)}{q(a)}\, da  ,
\end{equation}
where $p$ is assumed to be {absolutely continuous} \citep{billingsley2008probability} with respect to $q$ (i.e. $p$ is never nonzero where $q$ is zero), to ensure that the KL divergence exists. The KL divergence is zero iff $p = q$ almost everywhere, and is always non-negative.
Stochastic estimation of the KL divergence has the advantage of requiring just the ability to sample from $p$ and to calculate $p$ and $q$. This feature is in contrast to the Wasserstein metric for example, which generally requires solving an optimization problem just to compute it. 

The KL divergence is not symmetric. For example, $\KL{p}{q}$ may be defined while $\KL{q}{p}$ may not even exist if $q$ is not absolutely continuous with respect to $p$. This asymmetry leads to the two possible choices for measuring differences between distributions: the reverse KL and the forward KL. Assume that $q$ is a true distribution that we would like to match with our learned distribution $p_\theta$, where  $p_\policyparams$ is smooth with respect to $\theta \in \R^k$. The \textit{forward} KL divergence is $\KL{q} {p_\theta}$ and the \textit{reverse} KL divergence is $\KL{p_\theta}{q}$. 

We define the \textbf{Reverse KL} (RKL) for greedification on $\Qhat$ at a given state $s$ as
\begin{align*}
 \RKL{\pi_\policyparams}{\boltzmannQ}(s) &\defeq\KL{\pi_\policyparams}{\boltzmannQ}(s),
\end{align*}
where we additionally define for any two policies $\pi_1$, $\pi_2$,
\begin{align*}
    \KL{\pi_1}{\pi_2}(s) := \int_\cA \pi_1(a \mid s) \log \frac{\pi_1(a \mid s)}{\pi_2(a \mid s)}\,da.
\end{align*}
%
This $\Qhat$ is any action-value on which we perform approximate greedification; it can be a soft action value or not. 
We can rewrite the RKL as follows:
\begin{align*}
\begin{split}
\RKL{\pi_\policyparams}{\boltzmannQ}(s) &= \int_\actionspace \pi_\policyparams( a \mid s) \log \frac{\pi_\policyparams( a \mid s)}{\boltzmannQ(s, a)} \, \diff a \\
  &= \int_\actionspace \pi_\policyparams( a \mid s) \left( \log \pi_\policyparams( a \mid s) -  \frac{Q(s, a)}{\tau} + \log Z(s) \right) \, \diff a \\
  &= -\entropy(\pi_\policyparams( \cdot \mid s)) - \int_\actionspace \pi_\policyparams( a \mid s)  \frac{Q(s, a)}{\tau} \, \diff a + \log Z(s),
  \end{split}
\end{align*}
with gradient
\begin{align*}
 \nabla_{\policyparams} \RKL{\pi_\policyparams}{\boltzmannQ}(s)  = -\nabla_{\policyparams}\entropy(\pi_\policyparams( \cdot \mid s)) - \int_\actionspace \nabla_{\policyparams}\pi_\policyparams( a \mid s)  \frac{Q(s, a)}{\tau} \, \diff a
 .
\end{align*}
If we scale by $\tau$ to get $\tau \text{RKL}(\policyparams; s, \Qhat)$, 
we can see that $\tau$ plays the role of an entropy regularization parameter:\footnote{For a fixed $\boltzmannQ$, the policy that minimizes the RKL is the same regardless of the scaling by a constant in front, so we use the more standard unscaled KL to define the RKL.} a larger $\tau$ results in more entropy regularization on $\pi_\policyparams(\cdot \mid s)$. 

For a finite action space, we can take the limit to get the \textbf{Hard Reverse KL}.\footnote{When investigating $\tau \to 0$, it is not straightforward to extend the calculations we do here to integrals, so we derive them for the discrete case and extrapolate to the continuous case.} 
\begin{align*}
\begin{split}
    \lim_{\tau \to 0} \tau \RKL{\pi_\policyparams}{\boltzmannQ}(s) &= \lim_{\tau \to 0} \tau \sum_a \pi_\policyparams(a \mid s) \log \pi_\policyparams(a \mid s) - \tau \sum_\actionspace \pi_\policyparams(a \mid s) \log \boltzmannQ(s, a)\\
    &= 0-\lim_{\tau \to 0} \sum_a \pi_\policyparams(a \mid s) \left( Q(s, a) - \tau \log \sum_b \exp(Q(s, b)\tau^{-1}) \right)\\
    &= \left( -\sum_a \pi_\theta(a \mid s) Q(s, a) \right) + \lim_{\tau \to 0} \tau \log \sum_b \exp( Q(s, b)\tau^{-1})\\
    &= \left( -\sum_a \pi_\theta(a \mid s) Q(s, a) \right) + (n_s) \max_a Q(s, a),
\end{split}    
\end{align*}
where $n_s$ is the number of maximizing actions in $s$. Since the last term of the RHS does not depend on $\pi_\theta$, we are motivated to define the Hard Reverse KL as follows, for both finite and infinite action spaces.
\begin{align*}
    \text{Hard }\RKL{\pi_\policyparams}{\boltzmannQ}(s) &\defeq -\int_\actionspace \pi_\theta(a \mid s) Q(s, a) \diff a,
\end{align*}
with gradient
\begin{align*}
 \nabla_\policyparams \text{Hard }\RKL{\pi_\policyparams}{\boltzmannQ}(s) &= -\int_\actionspace \nabla_\policyparams \pi_\theta(a \mid s) Q(s, a)\, \diff a
 .
\end{align*}
%
If $Q$ is equal to $Q^{\pi_\theta}$, then this gradient is exactly the negative of the inner term of the policy gradient in \Cref{eq:policy-gradient-thm}.\footnote{
We are unaware of a previous statement of this result in the literature, though similar results have been reported. For example, \citet{kober2008policy} derive the policy gradient update from a pseudo-likelihood method. \citet{belousov2019entropic} also derive it as a special case of f-divergence constrained relative entropy policy search \citep{peters2010relative}. 
Some references to a connection between value-based methods with entropy regularization and policy gradient can be found in \citep{nachum2017bridging}.
} 
This similarity in form means that the typical policy gradient update in actor-critic can be thought of as a greedification step with the Hard RKL. 

Similarly, we can define the \textbf{Forward KL} (FKL) for greedification: 
\begin{align*}
\FKL{\pi_\policyparams}{\boltzmannQ}(s) &\defeq 
    \KL{\boltzmannQ}{\pi_\policyparams}(s).
\end{align*}
We can rewrite the FKL as
\begin{align*}
\begin{split}
 \FKL{\pi_\policyparams}{\boltzmannQ}(s) &= \int_\actionspace \boltzmannQ(s, a) \log \frac{\boltzmannQ(s, a)}{\pi_\theta(a \mid s)} \, \diff a \\
  &= \int_\actionspace \boltzmannQ(s, a) \log \boltzmannQ(s, a) \, \diff a - \int_\actionspace \boltzmannQ(s, a) \log \pi_\theta(a \mid s) \, \diff a\\
  &= -\entropy(\boltzmannQ(s, \cdot)) - \int_\actionspace \boltzmannQ(s, a) \log \pi_\theta(a \mid s) \, \diff a
\end{split}
\end{align*}
with gradient
\begin{align*}
 \nabla_{\policyparams} \FKL{\pi_\policyparams}{\boltzmannQ}(s) = - \int_\actionspace \boltzmannQ(s, a) \nabla_\policyparams \log \pi_\theta(a \mid s) \, \diff a
 .
\end{align*}

\begin{table}[t]
\centering
\begin{center}
\begin{tabular}{ | m{2.5em} | m{10.5em}| m{14em} | m{8em} | } 
\hline
\bf{KL} & \bf{Formula} & \bf{Gradient} & \bf{Comment} \\ 
\hline
RKL & $\KL{\pi_\policyparams}{\boltzmannQ}(s)$ & $- \int_\actionspace {\nabla_{\policyparams}\pi_\policyparams( a \mid s)} \tau^{-1} Q(s, a) \, \diff a - \nabla_{\policyparams}{\entropy(\pi_\policyparams( \cdot \mid s))}$ & A likelihood-based Soft Actor-Critic.\footnote{Soft Actor-Critic actually uses the reparametrization trick for its gradients, which we describe in further detail in \Cref{sec:api-alg} and also use in our experiments.} \\ 
\hline
Hard RKL & $-\int_\actionspace {\pi_\theta(a \mid s)} Q(s, a)\, \diff a$ & $-\int_\actionspace \nabla_\policyparams \pi_\theta(a \mid s) Q(s, a)\, \diff a$   & Equivalent to vanilla actor-critic if action value is unregularized.\\ 
\hline
FKL & $\KL{\boltzmannQ}{\pi_\policyparams}(s)$ & $- \int_\actionspace \boltzmannQ(s, a) {\nabla_\policyparams \log \pi_\theta(a \mid s) \, \diff a}$ & Like classification with cross-entropy loss and $\boltzmannQ$ the distribution over the correct label.\\
\hline
Hard FKL & $-\frac{1}{A^*}  \sum_{i = 1}^{A^*} \log \pi_\theta(a^*_i \mid s)$ & $-\frac{1}{A^*}  \sum_{i = 1}^{A^*}  \nabla_\policyparams  \log \pi_\theta(a^*_i \mid s)$ & Like classification with cross-entropy loss and $a^*_i$ the correct labels.\\
\hline
\end{tabular}
\end{center}
	\vspace{-0.6cm}
\caption{This table summarizes the four KL variants for greedification, including the objective, gradient, and a descriptive comment. 
We fix a state $s$ and define $a^*_i \in \argmax_a Q(s, a)$ with $A^*$ maximal actions. To reduce clutter, we have assumed that there can only be finitely many maximizing actions for the Hard FKL.}
\label{tab:kl_summary}
\end{table}

We can again consider the limit $\tau \to 0$ in the case of a finite action space (in this case there is no need to multiply the KL divergence by $\tau$). Assume that there are $A^*$ maximizing actions of $Q(s, \cdot)$, indexed by $a^*_i$.
\begin{align*}
   \lim_{\tau \to 0} \FKL{\pi_\policyparams}{\boltzmannQ}(s) 
   &= -\lim_{\tau \to 0} \entropy(\boltzmannQ(s, \cdot))  -\lim_{\tau \to 0}\sum_a \frac{\exp(Q(s, a)\tau^{-1})}{\sum_b \exp(Q(s, b)\tau^{-1}) } \log \pi_\policyparams(a \mid s)  \nonumber\\
    &= -\lim_{\tau \to 0} \entropy(\boltzmannQ(s, \cdot)) -\sum_a \lim_{\tau \to 0} \frac{\exp(Q(s, a)\tau^{-1})}{\sum_b \exp(Q(s, b)\tau^{-1})} \log \pi_\policyparams(a \mid s) \\
    &= -\lim_{\tau \to 0} \entropy(\boltzmannQ(s, \cdot))-\frac{1}{A^*}  \sum_{i = 1}^{A^*} \log \pi_\theta(a^*_i \mid s)
    .
\end{align*}
As the first term does not depend upon the policy parameters, we ignore it to define the \textbf{Hard Forward KL} as 
\begin{align*}
  \text{Hard }\FKL{\pi_\policyparams}{\boltzmannQ}(s) &\defeq -\frac{1}{A^*}  \sum_{i = 1}^{A^*} \log \pi_\theta(a^*_i \mid s)
\end{align*}
For continuous actions, if we have a compact action space, then the $\argmax$ exists over $a$ and provides valid actions. We can use a similar definition to the discrete case, assuming there are a finite number of maximal actions. If there are flat regions, with intervals of maximal actions, then the sum is replaced with an integral; for simplicity we assume a finite set and provide the definition under this assumption.
The gradient for the Hard FKL is
\begin{equation*}
 \nabla_\policyparams \text{Hard }\FKL{\pi_\policyparams}{\boltzmannQ}(s) = -  \frac{1}{A^*}  \sum_{i = 1}^{A^*} \nabla_\policyparams \log \pi_\theta(a^*_i \mid s)
\end{equation*}
The Hard FKL expression looks quite similar to the cross-entropy loss in supervised classification, if one views the maximum action of $Q(s, \cdot)$ as the correct class of state $s$. 
We are unaware of any literature that analyzes the Hard FKL for approximate greedification. 

We summarize the main expressions, gradients and results in Table \ref{tab:kl_summary}. Interestingly, many existing policy gradient methods have policy updates that fit into one of these four quadrants. We also summarize this categorization in Table \ref{tab:categorization}, with more justification for the categorization in Appendix \ref{sec:api}. For example, TRPO uses the Hard RKL, but with an additional constraint that the policy should not change too much after an update. This is encoded as a KL divergence to the previous policy, which is a difference use of the KL than described above. The algorithms in this summary table have many properties beyond the underlying update to the Boltzmann target policy; we do not intend to imply they are solely defined by the use of the RKL or FKL, with or without entropy regularization. Nonetheless, it provides another lens on understanding similarities and differences between these algorithms, due to the choice of this underlying update.

\begin{table}[t]
\centering
\begin{center}
\begin{tabular}{ | m{3.5em} | m{21em}| m{12em} | } 
\hline
& \bf{Reverse KL} & \bf{Forward KL} \\ 
\hline
\textbf{Hard} \newline No \newline entropy & Actor-Critic \citep{sutton2018reinforcement}; Trust Region Policy Optimization (TRPO) \citep{schulman2015trust} and variants including MPO \citep{abdolmaleki2018maximum}, Mirror descent policy iteration \citep{geist2019theory}, REPS \citep{peters2010relative};  PPO \citep{schulman2017proximal}; {Deep Deterministic Policy Gradient (DDPG)} \citep{silver2014deterministic,lillicrap2015continuous}  & {Deep Conservative Policy Iteration (DCPI)} \citep{vieillard2019deep} ; {Policy Greedification as Classification} \citep{lagoudakis2003reinforcement,lazaric2010analysis,farahmand2015classificationbased}\\ 
\hline
\textbf{Soft} \newline With entropy & {Soft Q-learning (SQL)} \citep{haarnoja2017reinforcement}; {Soft Actor-Critic (SAC)} \citep{haarnoja2018soft}; {A3C} \citep{mnih2016asynchronous}; Conservative value iteration \citep{kozuno2019theoretical} & {UREX} \citep{nachum2016improving,agarwal2019learning}; {Exploratory Conservative Policy Optimization (ECPO)}
     \citep{mei2019principled}\\ 
\hline
\end{tabular}
\end{center}
\vspace{-0.5cm}
\caption{A categorization of existing policy optimization methods, in terms of which underlying approach is used to update towards the Boltzmann policy. This list is not exhaustive, because (a) we have likely missed a few methods, given so many variants are based on an Actor-Critic update and (b) some policy optimization methods cannot be seen as updating with a KL divergence to the Boltzmann policy. See Appendix \ref{sec:api} for justification for this categorization.}
\label{tab:categorization}
\end{table}

\subsection{Differences between the FKL and RKL} \label{sec:mean_vs_mode}

Although switching $\pi_\policyparams$ and $\boltzmannQ$ might seem like a small change, there are several differences in terms of the optimization and the solution itself. In terms of the optimization, it can be simpler to sample the gradient of the RKL, because actions do not have to be samplied according to $\mathcal{B}_\tau Q$. The FKL, on the other hand, requires actions to be sampled from the $\boltzmannQ$, which can be expensive. But, favourably for the FKL, if $\pi_\theta(a \mid s) \propto \exp(\theta_a)$, then the FKL is convex with respect to $\theta$ because $\log \sum_i \exp(\theta_i)$ is convex.\footnote{This fact can be seen by showing that the Hessian is positive semi-definite.} The RKL, on the other hand, is generally not convex with respect to $\theta$, even if $\pi_\theta$ is parameterized with a Boltzmann distribution. 

The other critical difference is in terms of the solution itself. If $\boltzmannQ$ is representable by $\pi_\policyparams$ for some $\policyparams$, then the FKL and RKL both have the same solution: $\boltzmannQ$. Otherwise, they make different trade-offs. Of particular note is the well-known fact that the forward KL causes mean-seeking behavior and the reverse KL causes mode-seeking behavior \citep{bishop_mlbook}. To understand the reason, we look at the expressions for each divergence. For a target distribution $q$, the forward KL is $\int_x q(x) \log\left(\frac{q(x)}{p(x)}\right) \diff x$. 
If there is an $x$ where $q(x) >> 0$ (i.e. $q(x)$ is significantly greater than zero) and $p(x)$ is very close to zero, then $q(x)\log\left(\frac{q(x)}{p(x)}\right)$ is large. In fact, as $p(x) \rightarrow 0$, we have that $q(x)\log\left(\frac{q(x)}{p(x)}\right) \to \infty$. Therefore, to keep the forward KL small, whenever $q(x) >> 0$ then we also need $p(x) >> 0$.

This can result in $p$ that is quite different from $q$, if the parameterized $p$ cannot represent $q$. Particularly, if $q$ is multimodal and $p$ is unimodal, then the $p^*$ that minimizes the forward KL will try to cover all of the modes simultaneously, even if the cost for that is placing high mass in regions where $q(x) \approx 0$. For us, this corresponds to regions where the action-values are low. This forward KL solution, which is also called the M-projection (the M stands for moment), is known to be moment-matching. In the case that the family of distributions parameterizing $p$ is exponential and has some element whose moments match those of $q$, then the moments of $p^*$ and $q$ will match \citep{koller2009probabilistic}.

The reverse KL, on the other hand, has the expression $\int_x p(x) \log\left(\frac{p(x)}{q(x)}\right) \diff x$.
Even if $q(x) >> 0$, we can choose $p(x) = 0$ without causing $p(x)\log\left(\frac{p(x)}{q(x)}\right)$ to get big. This means that $p$ can select one mode, if $q$ is multi-modal. This can be desirable in RL, because it can concentrate the action probability on the region with highest value, even if there is another region with somewhat high action values. This distinction provides another helpful mnemonic: the \textbf{forward} KL cares about the \textbf{full} support of the target distribution $q$ and the \textbf{reverse} KL can \textbf{restrict} the support.

\begin{figure}[tb!]
	\centering
	\begin{tabular}{c c c}
		\begin{subfigure}[b]{0.33\linewidth}
		\includegraphics[width=\columnwidth]{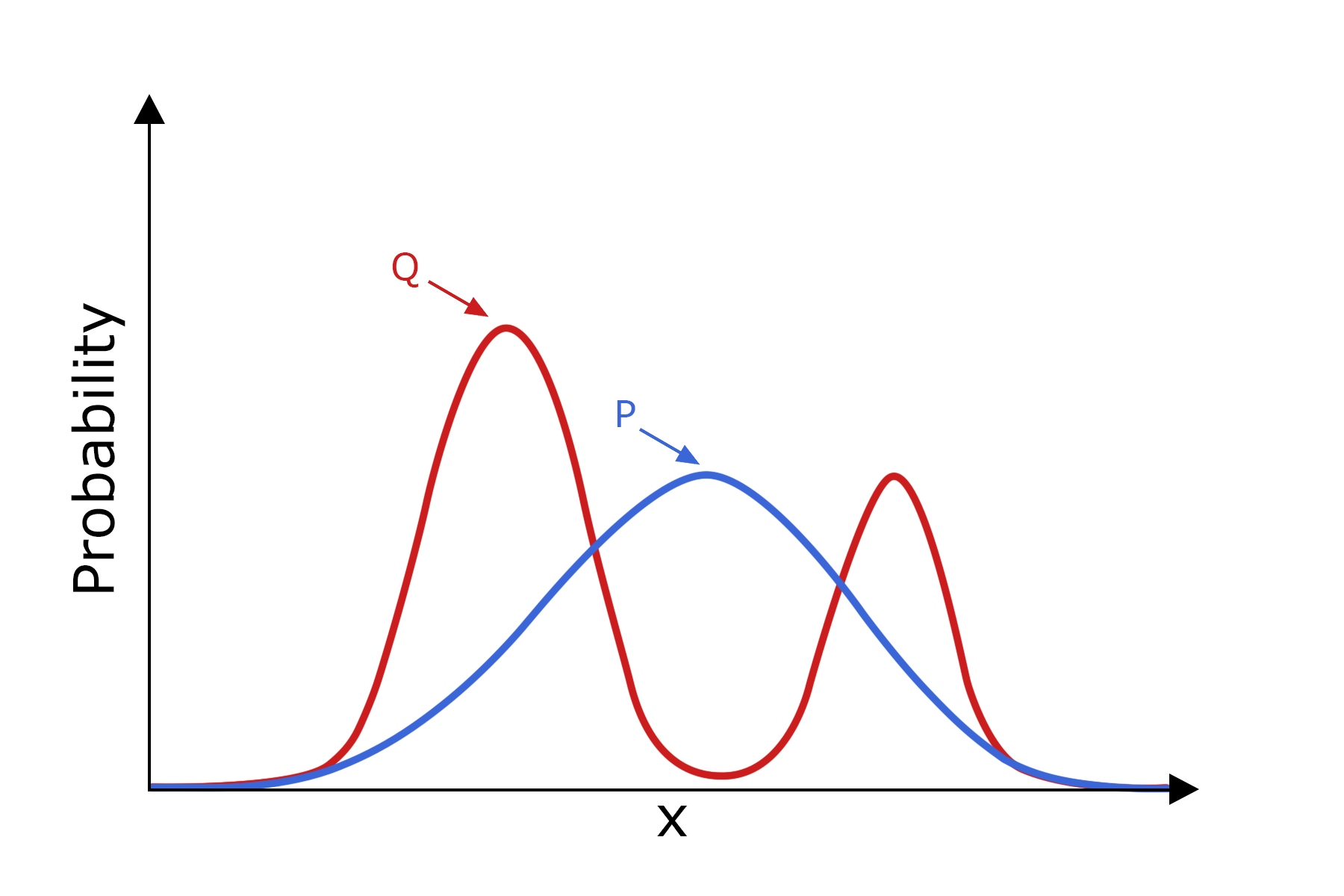} 
		\caption{FKL}
		\end{subfigure} &				
		\begin{subfigure}[b]{0.3\linewidth}
		\includegraphics[width=\columnwidth]{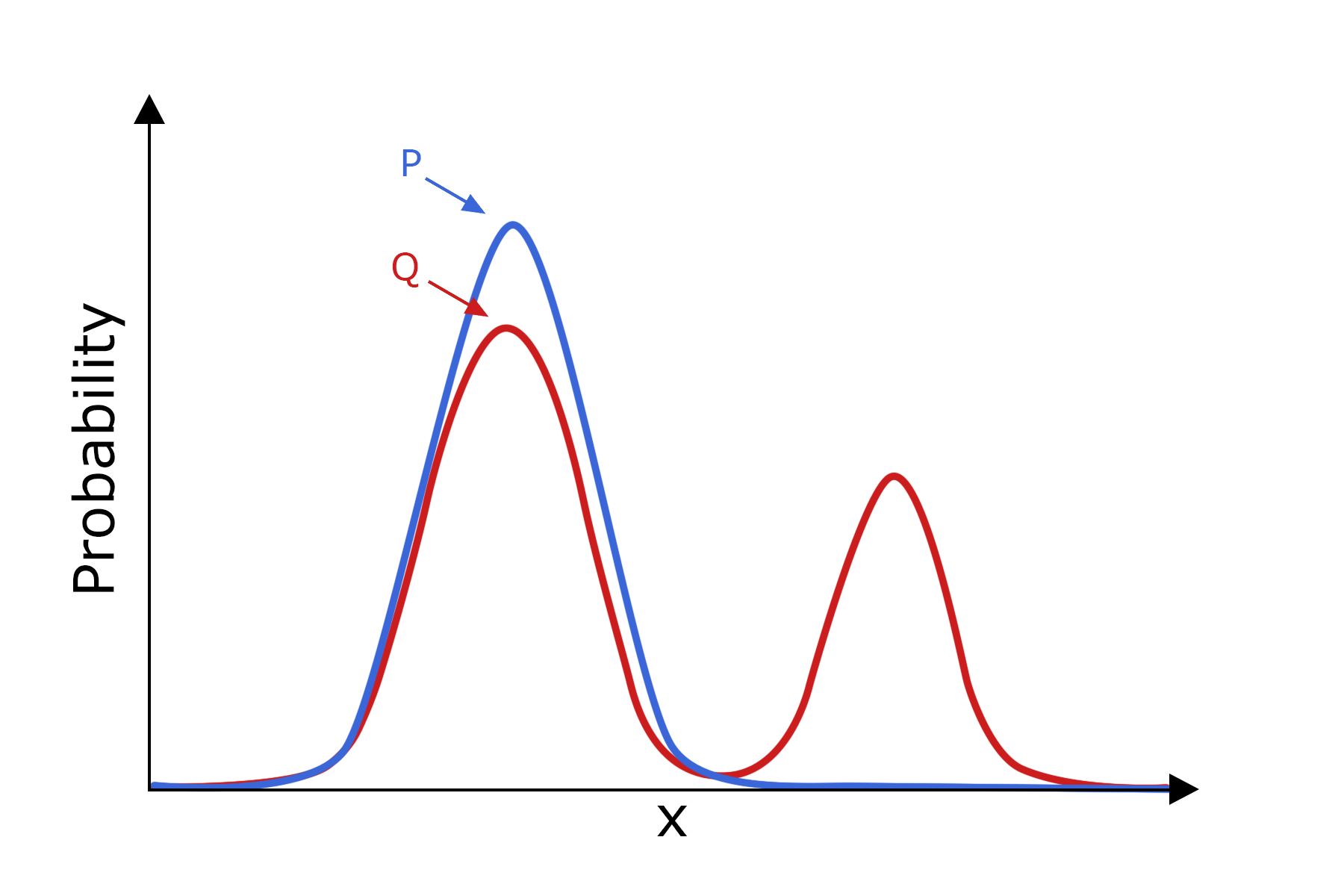} 
		\caption{RKL (global maximum)}
		\end{subfigure} &
		\begin{subfigure}[b]{0.3\linewidth}
		\includegraphics[width=\columnwidth]{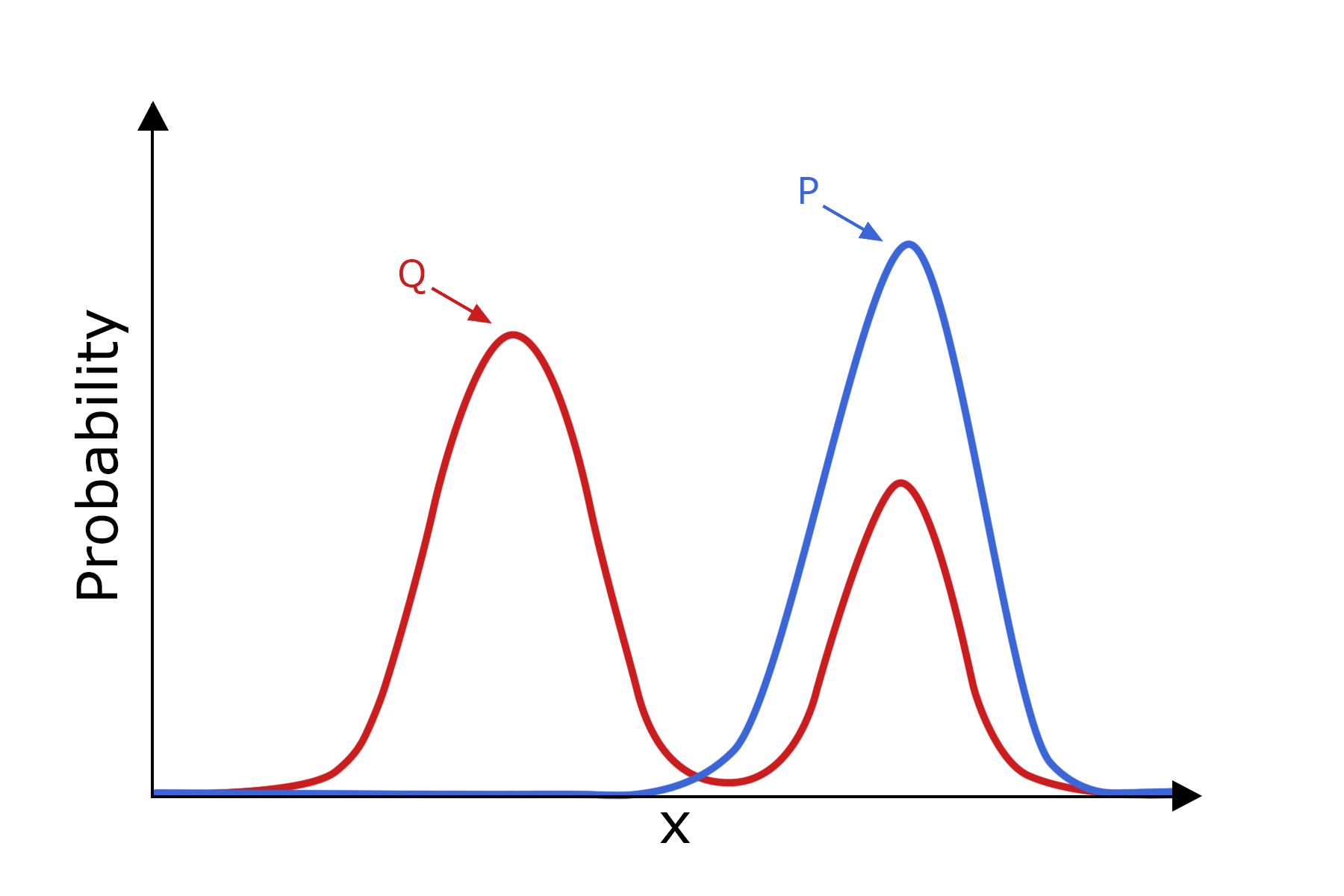} 
		\caption{RKL (subopt. local max)}
		\end{subfigure} 
	\end{tabular}  
	\vspace{-0.6cm}
	\caption{Mean-seeking (FKL) and mode-seeking (RKL) behavior.}\label{fig:mean_mode_seeking}
\end{figure}

However, the reverse KL can get stuck in sub-optimal solutions. 
If $q(x)$ is near zero for some $x$, and we pick $p(x) >> 0$, then $p(x)\log\left(\frac{p(x)}{q(x)}\right)$ can be large. As $q(x)$ gets closer to zero, this number goes to infinity. Therefore, reducing the reverse KL will lead to $p$ such that when $q(x) \approx 0$ then $p(x) \approx 0$. Because of this, the reverse KL is sometimes called zero-forcing or cost-averse. Similarly to the forward KL, this has certain consequences if $p$ is parameterized in a way that cannot represent $q$. In the multimodal target example we used for the forward KL, $p^*$ that minimizes the reverse KL, often called the I-projection (the I stands for information), will try to avoid placing mass in regions where $q(x) \approx 0$. This means that $p$ may end up in a sub-optimal mode when reducing the reverse KL via gradient descent. Both behaviors are illustrated in \Cref{fig:mean_mode_seeking}.

When approximating some distributions, reducing the FKL can cause overestimation of the tail of the target because of the mean-seeking behavior, whereas reducing the RKL underestimates it because of the mode-seeking behavior. In variational inference, the posterior distribution is approximated using a variational distribution. Importance sampling (IS) can be used to debias estimates obtained from some Bayesian inference procedures. The IS proposal distribution is commonly obtained by minimizing the RKL. However, because of this underestimation of the tail of the target, the quality of the IS estimates sometimes suffer, and a proposal distribution that overestimates the tail of the target is more desirable. Recent work \citep{jerfel2021variational} has shown that the FKL can be superior in this case. Better understanding the implications of reducing the FKL in the context of approximate policy improvement will facilitate incorporate such advantages into future RL algorithms.

\subsection{The Weighting over States}

The above greedification objectives, and corresponding gradients, are defined per state. To specify the full greedification objective across states, we need a weighting $d: \statespace \rightarrow \mathbb{R}^{+}$. Under function approximation, the agent requires this distribution to trade-off accuracy of greedification across states. The full objective for the RKL is
\begin{equation*}
    \int_\statespace d(s) \RKL{\pi_\policyparams}{\boltzmannQ}(s) \diff s.
\end{equation*}
The other objectives are specified similarly.  

The state weighting specifies how function approximation resources should be allocated for greedification. If there are no trade-offs, such as if the Boltzmann policy can be perfectly represented in each state, then the state weighting plays almost no role. It simply needs to be positive in a state to ensure the KL is minimized for that state. Otherwise, it may be that to make the policy closer to the Boltzmann policy in one state, it has to make it further in another state. 
The state weighting specifies which states to prioritize in this trade-off. 

Algorithms in practice use a replay buffer, where, without reweighting, $d$ implicitly corresponds to the state frequency in the replay buffer. 
We might expect early on that the implicit weighting is similar to the state visitation distribution under a random policy, and later more similar to the state visitation under a near-optimal policy---if learning is effective. The ramifications of allowing $d$ to be chosen implicitly by the replay buffer are as yet not well understood. In practice, algorithms seem to perform reasonably well, even without carefully controlling this weighting, possibly in part due to the fact that large neural networks are used to parameterize the policy that are capable of representing the target policy. 

There is, however, some evidence that the weighting can matter, particularly from theoretical work on policy gradient methods. 
This role of the weighting might seem quite different from the typical role in the policy gradient, but there are some clear connections. When averaging the gradient of the Hard RKL with weighting $d$, we have 
$- \int_\statespace d(s)\int_\actionspace Q(s,a) \nabla \pi_\theta(a \mid s) \, \diff a\, \diff s.$
If $d = d^{\pi_\policyparams}$ and $Q = Q^{\pi_\policyparams}$, then we have the true policy gradient; otherwise, for different weightings $d$, it may not correspond to the gradient of any function \citep{nota2019policy}. A similar issue has been highlighted for the off-policy policy gradient setting \citep{imani2018off}, where using the wrong weighting results in convergence to a poor stationary point. These counterexamples have implications for API, as they suggest that with accurate policy evaluation steps, the iteration between evaluation and greedification---the policy gradient step---may converge to poor solutions, without carefully selecting the state weighting. 

At the same time, this does not mean that the weighting must correspond to the policy state visitation distribution. Outside these counterexamples, many other choices could result in good policies. In fact, the work on CPI indicates that the weighting with $d^\pi$ can require a large number of samples to get accurate gradient estimates, and moving to a more uniform weighting over states can be significantly better \citep{kakade2002approximately}. The choice of weighting remains an important open question. For this work, we do not investigate this question, and simply opt for the typical choice in practice---using replay. 

\section{An API Algorithm with FKL or RKL}\label{sec:api-alg}

In this section, we provide a concrete algorithm that uses either the FKL or RKL for greedification within an API framework. The algorithm resembles Soft Actor-Critic (SAC) \citep{haarnoja2018soft}, which was originally described as a policy iteration algorithm. The key choices in the algorithm include (1) how to learn the (soft) action-values, (2) how to obtain an estimate of the RKL or FKL for a given state, and (3) how to sample states. 

To sample states, we use the standard strategy of maintaining a buffer of the most recent experience. We sample states uniformly from this buffer.
To obtain an estimate of the FKL or RKL for a given state, we need to estimate the gradient that has a sum or integral over actions. For the discrete action setting, we can simply sum over all actions. The \textbf{All-Actions} updates from a state $s$ correspond to
\begin{align}
    &\textbf{RKL}: \nabla_\theta \sum_{a \in \actionspace} \pi_\theta(a \mid s) \log \pi_\theta(a \mid s) - \sum_{a \in \actionspace}\nabla_\theta \pi_\theta(a \mid s) \log \boltzmannQ(a \mid s) \nonumber\\
    &\quad\quad\quad\quad= -\sum_{a \in \actionspace} \nabla_\theta\pi_\theta(a \mid s)\left[ \frac{Q(s,a)}{\tau} - \log \pi_\theta(a \mid s) \right] \label{eq:rkl-gradient-dis}\\
    &\textbf{FKL}: - \sum_{a \in \actionspace}  \boltzmannQ(a \mid s) \nabla_\theta \log \pi_\theta(a \mid s)  \label{eq:fkl-gradient-dis} \\
    &\textbf{Hard RKL}: - \sum_{a \in \actionspace}\nabla_\theta \pi_\theta(a \mid s) Q(s, a) \label{eq:hrkl-gradient-dis}\\
    &\textbf{Hard FKL}: -\nabla_\theta \log \pi_\theta \left(\argmax_b Q(s, b) \mid s\right) \label{eq:hfkl-gradient-dis}.
\end{align}
For the Hard FKL, when there is more than one maximal action, we assume that ties are broken randomly. For the continuous action setting, we can try to estimate the All-Actions update with numerical integration. 
More practically, we can simply sample actions.

Sampling actions is more straightforward for the RKL than the FKL. For the RKL and Hard RKL, we simply need to sample actions from the policy. In this case, we assume we sample $n$ actions $a_1, \ldots, a_n$ from $\pi_\theta( \cdot \mid s)$ and, using $\nabla_\theta \pi_\theta (a | s) =  \pi_\theta (a | s) \nabla_\theta \log \pi_\theta (a | s)$ and $\sum_{a \in \actionspace}\pi_\theta (a | s) \nabla_\theta \log \pi_\theta (a | s) \approx \sum_{i=1}^n \frac{ \nabla_\theta \log \pi_\theta (a_i | s) }{n}$, compute a \textbf{Sampled-Action} update, where we also change $Q$ to $Q-V$
\begin{align}
    &\textbf{RKL}: - \frac{1}{n} \sum_{i=1}^n  \nabla_\theta \log \pi_\theta(a_i \mid s) \left( \frac{Q(s,a_i) - V(s)}{\tau} - \log \pi_\theta(a_i \mid s) \right) \label{eq:rkl-gradient}\\
    &\textbf{Hard RKL}: - \frac{1}{n} \sum_{i=1}^n \nabla_\theta \log \pi_\theta(a_i \mid s) (Q(s, a_i) - V(s)) \label{eq:hrkl-gradient}
\end{align}
The inclusion of a baseline $V$ reduces variance due to sampling actions and does not introduce bias. 
Alternatively, for certain distributions, we can use the reparametrization trick and compute alternative sampled-action updates. For the case where the policy is parametrized as a multivariate normal, with $\pi_\theta(\cdot | s) \sim \mathcal{N}(\mu_\theta(s),\Diag(\sigma_\theta(s)))$, where $\Diag(\cdot)$ converts a vector to a diagonal matrix, an action sampled $a \sim \pi_\theta(\cdot | s)$ can be written as $a = a_\theta(s,a') = \mu_\theta(s) + \Diag(\sigma_\theta(s)) a'$ for $a' \sim \mathcal{N}(0,I)$. 
This reparameterization allows gradients to flow through sampled actions by using the chain rule. We can write:
\begin{align*}
\nabla_\theta \Ex_{\pi_\theta}[ f(\pi_\theta(\cdot|s)) ] &= \nabla_\theta \int_a \pi_\theta(a|s) f(\pi_\theta(a|s)) \diff{a} 
= \nabla_\theta \int_{a'} p(a') f(\pi_\theta(a_\theta(s,a')|s)) \diff{a'}\\
&= \int_{a'} p(a') \nabla_\theta f(\pi_\theta(a_\theta(s,a')|s)) \diff{a'} 
\approx \frac{1}{N}\sum_{i=1}^N  \nabla_\theta f(\pi_\theta(a_\theta(s,a_i')|s))
\end{align*}
Applying this to the formulas in \Cref{tab:kl_summary}, the updates are
\begin{align}
    &\textbf{RKL}: \frac{1}{n} \sum_{i=1}^n \nabla_\theta \log(\pi_\theta (a_\theta(s,a_i')|s)) - \Big( \frac{\nabla_\theta Q(s,a_\theta(s,a_i'))}{\tau} \Big) \label{eq:rkl-gradient-reparam}\\
    &\textbf{Hard RKL}: - \frac{1}{n} \sum_{i=1}^n \nabla_\theta Q(s,a_\theta(s,a_i')) \label{eq:hrkl-gradient-reparam} 
\end{align}

For the FKL, we need to sample according to $\boltzmannQ(\cdot \mid s)$, which can be expensive. Instead, we will use weighted importance sampling, similarly to a previous method that minimizes FKL \citep{nachum2016improving}. We can sample actions from $\pi_\theta$, and compute importance sampling rations ${\rho}_i \defeq \frac{\boltzmannQ(a_i \mid s)}{\pi_\theta( a_i \mid s)} \propto \frac{\exp(Q(s, a_i) \tau^{-1})}{\pi(a_i \mid s)}$. To reduce variance, we use weighted importance sampling with $ \tilde{\rho}_i \defeq \frac{{\rho}_i}{\sum_{j = 1}^n {\rho}_j}$, to get the update
\begin{align}
  &\textbf{FKL}: - \sum_{i = 1}^n \tilde{\rho}_i \nabla_\theta \log \pi_\theta(a_i \mid s) \label{eq:fkl-gradient}
\end{align}

The Hard FKL update is the same as in the discrete action setting, with the additional complication that computing the argmax action is more difficult for continuous actions. The simplest strategy is to do gradient ascent on $Q(s, \cdot)$, to find a maximal action. There are, however, smarter strategies that have been explored for continuous action Q-learning \citep{amos2016icnn,gu2016naf,kalashnikov2018qt,ryu2019caql,gu2019sample}. For example, input convex neural networks (ICNNs) ensure the learned neural network is convex in the input action, so that a gradient descent search to find the minimal action of the negative of the action-values is guaranteed to find a maximal action. We could use these approaches to find maximal actions, and then also learn an explicit policy with the hard FKL update that increases the likelihood of these maximizing actions.

Finally, we use a standard bootstrapping approach to learn the soft action-values. 
We perform bootstrapping as per the recommendations in \citet{pardo2017time}. For a non-terminal transition $(s, a, r, s')$, the action values $Q(s,a)$ are updated with the bootstrap target $r + \gamma V(s')$. For a terminal transition, the target is simply $r$. Note that an episode cut-off--- where the agent is teleported to a start state if it reaches a maximum number of steps in the episode---is not a terminal transition and is updated with the usual $r + \gamma V(s')$. 

To compute this bootstrap target, we learn a separate $V$. It is possible to instead simply use $Q(s',a') - \tau \log(\pi(a'|s'))$ for the bootstrap target, but this has higher variance. Instead, a lower variance approach is to use the idea behind Expected Sarsa, which is to compute the expected value for the given policy in the next state. $V(s')$ is a direct estimate of this expected value, rather than computing it from $Q(s', \cdot)$. 
To update $V$, we can use the same bootstrap target for $Q$, but need to incorporate an importance sampling ratio to correct the distribution over actions. To avoid using importance sampling, another option is to use the approach in SAC, where the target for $V(s)$ is $Q(s, a) - \tau \log(\pi(a \mid s))$. The complete algorithm, putting this all together, is in Algorithm \ref{alg:kl-agent}.



\begin{algorithm}
\caption{Approximate Policy Iteration (API) with KL Greedification}
\label{alg:kl-agent}
\begin{algorithmic}
    \State \textbf{Input}: choice of KL divergence; temperature $\tau \geq 0$; learning rates $\alpha_\theta$, $\alpha_v$, $\alpha_w$
    \State Initialize: policy $\pi_\theta$ (parameters $\theta$); action-value estimate $Q_\beta$ (parameters $\beta$); state-value estimate $V_w$ (parameters $w$); experience replay buffer $\mathcal{B}$  \\ 
    Get initial state $s_0$
    \For{$t = 0, \ldots$}
        \State Apply action $a_t \sim \pi_\theta(\cdot \mid s_t)$ and observe $r$, $s_{t+1}$,$\done$
        \State $\mathcal{B} = \mathcal{B} \cup (s_t, a_t, s_{t + 1}, r, \done)$
        \If{ $|\mathcal{B}| \geq $ batch size}
            \State Draw minibatch $\mathcal{D} \sim \mathcal{B}$ 
            
            \State Calculate $g_\theta \approx \Ex_{\mathcal{D}} [\nabla_\theta KL]$ using one of Equations \ref{eq:rkl-gradient-dis} - \ref{eq:fkl-gradient}
            \State Calculate $g_w, g_v$ using \Cref{alg:diff-value-update} 
   
            \State $\theta = \theta - \alpha_\theta g_\theta$
            \State $w = w - \alpha_w g_w$
            \State $v = v - \alpha_v g_v$
        \EndIf
    \EndFor
\end{algorithmic}
\end{algorithm}

\begin{algorithm}
\caption{GetValueUpdates}
\label{alg:diff-value-update}
\begin{algorithmic}
    \State Given: policy $\pi_\theta$; $Q_\beta$; $V_w$; $\tau \geq 0$; batch of data $\mathcal{D}$
    \State $g_w \gets 0, g_v \gets 0$
     \For {$(s, a, r, s')$ in $\mathcal{D}$}  
	\State Draw $\tilde{a} \sim \pi_\theta(\cdot \mid s)$.     
    \State $g_w \gets g_w - (Q_\beta(s, \tilde{a}) - \tau \log \pi_\theta(\tilde{a} \mid s) - V_w(s)) \nabla_w V_w(s)$
    \State $g_v \gets g_v - (r + \gamma \cdot (1 - \done)\cdot V_w(s') - Q_\beta(s, a)) \nabla_v Q_\beta(s, a)$
    \EndFor
    \State $g_w \gets \frac{g_w}{|\mathcal{D}|}, g_v \gets \frac{g_v}{|\mathcal{D}|}$ 
\end{algorithmic}
\end{algorithm}

\section{Theoretical Results on Policy Improvement Guarantees}\label{sec:theory}

We study the theoretical policy improvement guarantees under the RKL and FKL. We start with definitions and by motivating the choice of the entropy regularized setting. We then consider the guarantees, or lack thereof, for the RKL and FKL. 
We first provide an extension of Lemma 2 in \citep{haarnoja2018soft} to rely only upon RKL minimization \textit{on average} across states, and then further extend the result to approximate action values. We then show we can obtain a more practical result, under an additional condition that the policy update does not take too big of a step away from the current policy. 

\citet{geist2019theory} performed error propagation analysis of entropy-regularized approximate dynamic programming algorithms, but they did not provide a monotonic policy improvement guarantee similar to ours. In contrast, \citet{shani2019adaptive} analyzed TRPO and provide monotonic policy improvement guarantee (Lemma 15). However, their result relies on the tabular representation of the policy, and it does not necessarily apply to RKL reduction for general, non-tabular policies. Lemma 2 of \citet{lan2021policyMirror} is the same result as our \Cref{lemma:soft-performance-difference}. \citet{lingwei2020ensuring} provided monotonic policy improvement guarantee similar to \Cref{lemma:soft-performance-difference}. Compared to their result, we take a further step to show that RKL reduction on average across states suffices for monotonic policy improvement. Moreover, we provide additional results that do not depend on having the true action values or improbable state distributions. 

Then we investigate the FKL, for which there are currently no existing policy improvements results. We show that the FKL does not have as strong of policy improvement guarantees as the RKL. We provide a counterexample where optimizing the FKL \textit{does not} induce policy improvement. But, this counterexample does not imply that FKL reduction {cannot} provide policy improvement. We discuss further assumptions that can be made to ensure that FKL does induce policy improvement.
All proofs are contained in \Cref{app:proofs}. 

\subsection{Definitions and Assumptions}


We characterize performance of the policy in the entropy regularized setting. First, it will be useful to introduce some concepts for unregularized MDPs and then present their counterparts for entropy regularized MDPs. Throughout, we assume that the class of policies $\Pi$ consists of policies whose entropies are finite. This assumption is not restrictive for finite action-spaces, as entropy is always finite in that setting; for continuous action-spaces, for most distributions used in practice like Gaussians with finite variance, the entropy will also be finite. The assumption of finite entropies is necessary to ensure that the soft value functions are well-defined.

For some of the theoretical results for the FKL, we will restrict our attention further to finite action-spaces, to ensure we have non-negative entropies and to use the total variation distance for discrete sets. We use sums instead of integrals throughout our proofs to enhance clarity; unless we explicitly assume a finite action space, all of our results hold as well for general action spaces given standard measure-theoretic assumptions.

\begin{assumption}
Every $\pi \in \Pi$ has finite entropy: $\entropy(\pi(\cdot | s)) < \infty$ for all $s \in \statespace$.
\end{assumption}

\begin{definition}[Unregularized Performance Criterion]\label{def:perf-criterion}
	For a start state distribution $\rho_0$, the performance criterion is defined as
	\begin{align*}
		\eta(\pi) := \Ex_{\rho_0}[V^\pi(S)].
	\end{align*}
\end{definition}

\begin{definition}[Unregularized Advantage]\label{def:advantage}
	For any policy $\pi$, the \textit{advantage} is 
	\begin{align*}
		A^\pi(s, a) := Q^\pi(s, a) - V^\pi(s).
	\end{align*}
\end{definition}
The advantage asks: what is the average benefit if I take action $a$ in state $s$, as opposed to drawing an action from $\pi$? The soft extensions of these quantities are as follows.

\begin{definition}[Soft Performance Criterion]\label{def:soft-performance}
For a start state distribution $\rho_0$ and temperature $\tau > 0$,the soft performance criterion is defined as
	\begin{align*}
		\eta_\tau(\pi) := \Ex_{\rho_0}[V^\pi_\tau(S)].
	\end{align*}
\end{definition}
It will also be helpful to have a soft version of the advantage. An intuition for the advantage in the non-soft setting is that it should be zero when averaged over $\pi$. To enforce this requirement in the soft setting, we require a small modification.
\begin{definition}[Soft Advantage]\label{def:soft-advantage}
	For a policy $\pi$ and temperature $\tau > 0$, the soft advantage is
	\begin{align*}
		A^\pi_\tau(s, a) := Q^\pi_\tau(s, a) - \tau \log \pi(a \mid s) - V^\pi_\tau(s).
	\end{align*}
\end{definition}
\noindent If $\tau = 0$, we recover the usual definition of the advantage function. Like unregularized advantage functions, this definition also ensures $\Ex_\pi[A^\pi_\tau(s, A)]= 0$.

\subsection{Why Use the Entropy Regularized Framework?}

Since the actual goal of RL is to optimize the unregularized objective, it might sound unnatural to instead study guarantees in its regularized counterpart. We can view the entropy regularized setting as a surrogate for the unregularized setting, or simply of alternative interest. In the first case, it may be too difficult to optimize $\eta(\pi)$; entropy regularization can improve the optimization landscape and potentially promote exploration. Optimizing $\eta_\tau(\pi)$ is more feasible and can still get us close enough to a good solution of $\eta(\pi)$. In the second case, we may in fact want to reason about optimal stochastic policies, obtained through entropy regularization. In either setting, it is sensible to understand if we can obtain policy improvement guarantees under entropy regularization. 

There have been several recent papers highlighting that entropy regularization can improve the optimization behavior of policy gradient algorithms. \citet{mei_global2020} studied how entropy regularization affects convergence rates in the tabular case, considering policies parametrized by a softmax. By using a proof technique based on Łojasiewicz inequalities, they were able to show that policy gradients without entropy regularization converge to the optimal policy at a $O(1/t)$ rate. Furthermore, they also showed a $\Omega(1/t)$ bound for this same method, concluding that the bound is unimprovable for vanilla policy gradients. By adding entropy regularization, the convergence rate can be improved to $O(\mathrm{e}^{-t})$.

\citet{ahmed2018understanding} empirically studied how adding entropy regularization changes the optimization landscape for policy gradient methods. By sampling multiple directions in parameter space for some suboptimal policy and visualizing scatter plots of curvature and gradient values around that policy, combined with visualization techniques that linearly interpolate policies, they concluded that adding entropy regularization likely connects local optima. The optimization landscape can be made smoother, while also allowing the use of higher learning rates.

Finally, \citet{ghosh2020operator} provided theoretical justification that (nearly) deterministic policies can stall learning progress. They first provide an operator view of policy gradient methods, particularly showing that REINFORCE can be seen as a repeated application of an improvement operator and a projection operator \citep[Proposition 1]{ghosh2020operator}. They then showed \citep[Proposition 5]{ghosh2020operator} that the performance of the (non-projected) improved policy $\pi'$, $\eta(\pi')$, is equal to $\eta(\piold)$ times a term including the variance: $\eta(\pi') = \eta(\piold) (1+\frac{\text{Variance of Return under $\piold$}}{\text{Expected Return under $\piold$}}) \ge \eta(\piold)$. This means that if the variance under $\piold$ is near zero, then $\eta(\pi') \approx \eta(\piold)$. In that sense, having higher variance can help the algorithm make consistent progress. A common way of achieving higher variance is by adding entropy regularization.

Finally, there is some theoretical work relating the solutions under the unregularized and regularized objectives. From \citep[Proposition 3]{geist2019theory}, if the entropy is bounded for all policies with constants $L_\tau, U_\tau$ giving $L_\tau \leq -\tau \entropy(\pi)  \leq U_\tau$, then we know that 
		$V^\pi(s) - \frac{U_\tau}{1 - \gamma} \leq V^\pi_\tau(s) \leq V^\pi(s) - \frac{L_\tau}{1 - \gamma}.$
%
Using this result, 
we can take expectations across the state space with respect to the starting state distribution, to get
\begin{align*}
	\eta(\pi) - \frac{U_\tau}{1 - \gamma} \leq \eta_\tau(\pi) \leq \eta(\pi) - \frac{L_\tau}{1 - \gamma}.
\end{align*}
Hence, if the upper bound is tight, increasing $\eta_\tau(\pi)$ will increase $\eta(\pi)$. A similar result exists for single-step decision making with discrete actions \citep[Proposition 2]{chen2019surrogate}.

\subsection{Policy Improvement with the RKL}\label{sec_rkl_improvement}

First, we note a strengthening of the original result for policy improvement under RKL reduction \citep{haarnoja2018soft}. 
Particularly, they take $\pi_{new}$ to be the policy that \emph{minimizes} the RKL to $\boltzmannQ^\piold_\tau(s, \cdot)$ {at every state}. Examining their proof reveals that their new policy $\pi_{\mathrm{new}}$ does not have to be the minimizer; rather, it suffices that $\pi_{\mathrm{new}}$ is smaller in RKL than $\pi_{\mathrm{old}}$ at every state $s$. We therefore restate their lemma with this slight modification.
\begin{lemma}[Restatement of Lemma 2 \citep{haarnoja2018soft}]\label{lem:stronger-sac}
\!\!For $\piold, \!\pinew \in \!\Pi$, if for all $s$
\begin{align*}
    \RKL{\pinew}{\boltzmannQ^\piold_\tau}(s) \le \RKL{\piold}{\boltzmannQ^\piold_\tau}(s),
\end{align*}
then $Q^\pinew_\tau(s, a) \geq Q^\piold_\tau(s, a)$ for all $(s, a)$ and $\tau > 0$.
\end{lemma}
\begin{proof}
Same proof as in \citet{haarnoja2018soft}.
\end{proof}
We extend this result by considering an RKL reduction in average across states, rather than requiring RKL reduction in every state. 
To prove our result, it will be useful to prove a soft counterpart to the classical performance difference lemma \citep{kakade2002approximately}.
\begin{restatable}{lemma}{spd}[Soft Performance Difference]\label{lemma:soft-performance-difference}
For any policies $\piold, \pinew$, any $\tau \geq 0$, we have
\begin{equation*}
   \eta_\tau(\pinew) - \eta_\tau(\piold)= 
    \frac{1}{1 - \gamma}\Ex_{d^\pinew} \bigg[\Ex_{\pinew}[A_\tau^{\piold}(S, A)] -\tau\KL{\pinew}{\piold}(S)\bigg].
\end{equation*}
\end{restatable}

If we set $\tau = 0$, we recover the classical performance difference lemma. Now, we can show that reducing the RKL on average is sufficient and necessary for policy improvement.

\begin{restatable}{proposition}{rklimprove}[Improvement Under Average RKL Reduction]\label{prop:avg-reverse-kl}
For $\piold, \pinew \in \Pi$, define 
\begin{align*}
\Delta \RKL{\piold}{\pinew}(S) 
\!\defeq \!\RKL{\piold}{\BQt^{\piold}}(S) - \RKL{\pinew}{\BQt^{\piold}}(S).
\end{align*}
\begin{equation}
	\text{For } \tau > 0: \quad \eta_\tau(\pinew) - \eta_\tau(\piold) = \frac{\tau}{1 - \gamma} \Ed{\Delta \RKL{\piold}{\pinew}(S)}. \label{eq:rkl_reduction}
\end{equation}
Furthermore, $\eta_\tau(\pinew) \geq \eta_\tau(\piold)$ if and only if $\Ex_{d^{\pinew}}[\Delta \RKL{\piold}{\pinew}(S)] \geq 0$.
\end{restatable}
This result shows that reducing the RKL on average, under weighting $d^\pinew$, guarantees improvement. Notice that the more stringent condition of RKL reduction in every state from \Cref{lem:stronger-sac} ensures reduction under the weighting $d^\pinew$; this new result is therefore more general. Ensuring reduction under $d^\pinew$, however, may be difficult in practice, as we do not have access to data under $\pinew$; rather, we have data from $\piold$. We extend this result in Section \ref{sec_extension_dpiold}, to a weighting under $\piold$, by adding a condition on how far $\pinew$ moves from $\piold$. This more practical result relies on the above result, and so we present the above result first as a standalone to highlight the key reason for the policy improvement.

This result provides some theoretical support for using stochastic gradient descent for the RKL, as is done in practice. It is unlikely that we will completely minimize the RKL on every step, nor reduce it in every state. 
With sufficient reduction of the average RKL on each step, this iterative procedure between approximate greedification and exact policy evaluation should converge to an optimal policy. In fact, by inspecting Equation \eqref{eq:rkl_reduction}, we can see that any optimal policy satisfies, for any fixed $\pi_0$,
\begin{align*}
	\pi^* \in \mathrm{arg} \max_{\pi \in \Pi} \eta_\tau(\pi)  = \mathrm{arg} \max_{\pi \in \Pi} \eta_\tau(\pi) - \eta_\tau(\pi_0) = \mathrm{arg} \max_{\pi \in \Pi} \frac{\tau}{1 - \gamma} \Ex_{d^\pi}[\Delta \RKL{\pi_0}{\pi}(S)]
\end{align*}
Note, however, that we cannot generally guarantee that this procedure will converge to the optimal policy. This is because the average RKL reduction may decrease to zero prematurely, in the sense that $\lim_{i \to \infty} \eta_\tau(\pi_i)$ may be less than $\sup_{\pi \in \Pi} \eta_\tau(\pi)$. 

\subsubsection{Extension to Action-value Estimates}\label{sec_rkl_imp_action}

In the previous section we focused on approximate greedification with exact action-values. The theoretical results allowed for improvements on average across the state space, better reflecting what is done in practice. However, algorithms in practice are also not likely to have exact action-values. In this section, we further extend the theoretical results to allow for both approximate greedification and approximation policy evaluation. 

First, we prove an analogue of the soft performance difference lemma for approximate action-values. 
\begin{restatable}{lemma}{approxspd}[Approximate Soft Performance Difference]\label{lemma:approximate-soft-performance-difference}
Let $\piold, \pinew$ be any policies and let $\tau \geq 0$. Let $\hat Q : \statespace \times \actionspace \to \R$ represent an action-value estimate and let $\epsilon(s, a) := \hat Q(s, a) - Q^\piold_\tau(s, a)$ be the per-state approximation error. As well, define
\begin{align*}
    \hat \Delta\RKL{\piold}{\pinew}(S) &:= \RKL{\piold}{\mathcal{B}_\tau \hat Q}(S) - \RKL{\pinew}{\mathcal{B}_\tau \hat Q}(S),\\
    \bar \epsilon &:= \Ex_{d^\pinew}[\Ex_{\pinew}[\epsilon(S, A)] - \Ex_{\piold}[\epsilon(S, A)]].
\end{align*}
%
\begin{equation}\label{eq:approx-spd}
   \text{We have: } \quad \eta_\tau(\pinew) - \eta_\tau(\piold) + \bar \epsilon = \frac{\tau}{1 - \gamma}\Ex_{d^\pinew} [ \hat \Delta\RKL{\piold}{\pinew}(S) ].
\end{equation}
\end{restatable}
As a corollary, we have a corresponding policy improvement result.
\begin{restatable}{corollary}{approxrkl}[Approximate RKL Reduction]\label{cor:approximate-rkl-reduction}
    Under the assumptions of \Cref{lemma:approximate-soft-performance-difference}, $\eta_\tau(\pinew) \geq \eta_\tau(\piold)$ iff $ \frac{\tau}{1 - \gamma} \Ex_{d^\pinew} \Delta\RKL{\piold}{\pinew}(S) \geq \bar \epsilon.$
\end{restatable}
If we only have access to an estimate $\hat Q$ of $Q^\piold_\tau$, reducing the RKL is not enough to guarantee policy improvement. When one reduces the RKL, one must also take care that $\bar \epsilon$ not be larger than the RKL amount reduced. The quantity $\bar \epsilon$ represents the difference in average approximation error over the action distributions of $\pinew$ and $\piold$. For $\bar \epsilon$ to be small, the approximation error averaged over the $\pinew$ distribution should not be much larger than the approximation error averaged over the $\piold$ distribution. For example, $\bar \epsilon$ could be small if $\pinew$ is similar to $\piold$, or if $\hat Q$ already approximates $Q^\piold_\tau$ well under the state-action distribution induced by $\pinew$.

\subsubsection{Extensions to Weighting Under $d^\piold$}\label{sec_extension_dpiold}
CPI \citep{kakade2002approximately} derives a lower bound of the performance difference under which one is able to guarantee policy improvement using $d^\piold$, assuming $\pinew$ is sufficiently close to $\piold$. Similar considerations allow us to derive a corresponding bound in our setting.

\begin{restatable}{proposition}{dpioldpdl}\label{prop:dpiold-extension}
If $\KL{\pinew}{\piold} (s) \leq \alpha$ for all $s \in \cS$,
\begin{align}
    \eta_\tau(\pinew) - \eta_\tau(\piold)
    &\geq
    \frac{1}{1 - \gamma} \Ex_{d^\piold} \bigg[
        \sum_a \pinew (a|S) \advold (S, a)
        - \tau \KL{\pinew}{\piold} (S)
        - 4 \sqrt{2} V_{\tau, max} \sqrt{\alpha}
    \bigg]\nonumber\\
    &= \frac{1}{1 - \gamma} \Ex_{d^\piold} \bigg[
        \Delta \RKL{\piold}{\pinew}
        - 4 \sqrt{2} V_{\tau, max} \sqrt{\alpha}
    \bigg]
\end{align}
\end{restatable}
With knowledge of $V_{\tau, max}$ and the exact value functions, one could optimize the lower bound in \Cref{prop:dpiold-extension} as a function of both $\alpha$ and $\pinew$, without knowledge of $d^\piold$. If $V_{\tau ,max}$ is large, then $\alpha$ must be rather small to ensure that the RHS of \Cref{prop:dpiold-extension} is non-negative. Intuitively, the larger the maximum return, the greater the possible error in using $d^\piold$ rather than $d^\pinew$.

We can also combine \Cref{prop:dpiold-extension} and \Cref{lemma:approximate-soft-performance-difference}.
\begin{restatable}{proposition}{approxdpiold}\label{prop:approx-dpiold-extension}
If $\KL{\pinew}{\piold} (s) \leq \alpha$ for all $s \in \cS$,
\begin{align}
    \eta_\tau(\pinew) - \eta_\tau(\piold)
    &\geq
    \frac{1}{1 - \gamma} \Ex_{d^\piold} \bigg[
        \tau  \hat \Delta\RKL{\piold}{\pinew}(S) \nonumber\\
        &\quad\quad\quad+ \sum_a \epsilon(s, a)(\piold(a) - \pinew(a))
        - 4 V_{\tau, max} \sqrt{2\alpha}\bigg].
\end{align}
\end{restatable}

\subsection{Policy Improvement with the FKL}\label{sec_fkl_improvement}

In this section, we study the policy improvement properties of reducing the FKL. First, in \Cref{sec:fkl-counterexample} we provide a counterexample showing that reducing the FKL leads to a strictly worse policy. Second, in \Cref{sec:direct-fkl-reduction} we provide a sufficient condition on the FKL reduction to ensure policy improvement. The plots in that section show that this bound is non-trivial, but that unfortunately the required reduction is close to the maximum possible reduction. Third, in \Cref{sec:fkl-rkl-connection}, we discuss when reducing the FKL may be used as a surrogate for reducing the RKL, in particular by providing an upper bound for the RKL in terms of the FKL. It will turn out that reducing the FKL alone is insufficient for reducing the RKL because this bound involves not only the FKL, but another term that depends upon $\pinew$. We conclude with a discussion about the implications for the use of FKL for approximate greedification. 

\subsubsection{Counterexample for Policy Improvement under FKL Reduction}\label{sec:fkl-counterexample}
Unfortunately, the FKL does not enjoy the same policy improvement guarantees as the RKL. In the next proposition, we provide a counterexample where reducing the FKL makes the policy worse. 
The intuition behind this example is that $\piold$ almost always chooses the good action, but is made close to deterministic and thus arbitrarily large in FKL to $\BQold$, while $\pinew$, by being less deterministic, reduces the FKL to $\BQold$ but it almost always chooses the bad action, thus being worse in the soft-objective. 

Again we use the notation
	\begin{align*}
	\Delta \FKL{\piold}{\pinew} (s) \defeq \FKL{\piold}{\BQt^{\piold}} (s)  - \FKL{\pinew}{\BQt^{\piold} }(s)
	\end{align*}
where $\Delta \FKL{\piold}{\pinew }(s) > 0$ means we obtained FKL reduction: the new policy has lower FKL than the old policy.	
\begin{restatable}{proposition}{fklcounter}[FKL Counterexample]\label{lem:forward-kl-counterexample} There exists an MDP such that, for any $\tau \geq 0$, there exists a pair of policies $(\piold, \pinew)$ where $\Delta \FKL{\piold}{\pinew }(s)  > 0$ at every state but the new policy has lower value: $\Qpitau{\pinew} (s, a) < \Qpitau{\piold} (s, a)$ at every state-action pair $(s, a) \in \cSA$, $\Vpitau{\pinew} (s) < \Vpitau{\piold} (s)$ at every state $s \in \statespace$, and $\eta_\tau (\pinew) < \eta_\tau (\piold)$.
\end{restatable}

\subsubsection{Policy Improvement under Sufficient FKL Reduction}\label{sec:direct-fkl-reduction}

We know that completely reducing the FKL, under the assumption that we can represent all policies, will guarantee improvement, since then we would have $\pinew = \BQt^\piold$. 
One might hope that additional conditions, that ensure \emph{sufficient} reduction in the FKL, might imply policy improvement. Because policy improvement is obtained if and only if the RKL is reduced, as per \Cref{prop:avg-reverse-kl}, we can equivalently ask what conditions on the FKL ensure we obtain RKL reduction. In this section, we provide a lower bound on the FKL reduction, that guarantees the RKL is reduced and so the policy is improved. We numerically investigate the magnitude of required reduction under this condition, to see how much lower it is than completely reducing the FKL. 

As before, we could first prove this result under reduction per-state; this is a more restricted setting that implies reduction in average across the states. We therefore provide only the more general result, which averages across states, as the connection between per-state and across state has already been clearly shown above and is not useful to repeat. 
\begin{restatable}[Improvement Under Average Sufficient FKL Reduction]{proposition}{fklredavg}
	\label{prop:fklredavg} Assume the action set is finite. If
	\begin{equation}\label{eq:rkl-cross-ent-positive}
	    \Ed{\RKL{\piold}{\BQold}(S) +   \Ex_{\BQold}[\log \piold(\cdot|S)]} \geq 0,
	\end{equation}
 \begin{align}
 	&\quad \Ex_{d^{\pinew}}[\Delta \FKL{\piold}{\pinew}(S) ]	 \geq \Ex_{d^{\pinew}}[\FKL{\piold}{\BQold}(S)] - \frac{1}{2}\Bigg(\! \tfrac{\tau}{\norm{Q^{\piold}_\tau\!}_\infty} \Big( \Ed{\RKL{\piold}{\BQold}(S)} \nonumber \\
 	& \quad\quad\quad\quad\quad\quad\quad\quad\quad\quad\quad\quad\quad\quad\quad+  \Ex_{d^{\pinew}} \left[ \Ex_{\BQold}[\log(\piold(\cdot|S))] \right] \!\Big)\! \Bigg)^2, \label{eq_required_fkl}\\
 	&\quad\quad\quad\quad\quad\quad\quad\quad\quad\quad\text{and} \quad\forall s \in \cS, \,\Delta \FKL{\piold}{\pinew}(S) \geq 0,
 	\end{align}
	then $\eta_\tau(\pinew) \geq \eta_\tau(\piold)$.
\end{restatable}
\Cref{eq:rkl-cross-ent-positive} essentially says that $\Ed{\RKL{\piold}{\BQold}(S)}$ is greater than or equal to $\Ed{\FKL{\piold}{\BQold}(S) + \entropy(\BQold(S, \cdot))}$. When might the RKL be larger than the FKL and the entropy? If $\BQold$ has low entropy across states, then $\BQold$ will have low probability mass placed on certain actions. If $\piold$ places probability mass on these actions, the RKL will likely be high because the RKL incentivizes mode-matching. Unfortunately, this result also assumes that FKL reduction is non-negative in all states. 

Just as with the RKL, we can extend these results to use an action-value estimate $\hat Q$. We provide these extensions and their proofs in Appendix \ref{app_direct-fkl-reduction}.

We know that fully reducing the FKL, so that $\pinew$ equals $\BQold$, guarantees improvement, which is a strong requirement; we can ask how much less strict the above condition is in comparison.  
We can check this numerically, in a simple bandit setting with $|\cA| = 5$.
We test different $\piold$ calculated as:
$$\piold = (1 - \lambda) \pi_{\text{rand}} + (\lambda) \BQ$$
for $\lambda \in \{ 0, 1/3, 2/3, 0.99 \}$, $ \pi_{\text{rand}} $ the random policy and each of these policies corresponding to probability vectors over the 5 actions. Varying $\lambda$ allows us to see the impact on the bound for policies far from the target $\BQ$ ($\lambda = 0$) to very close to the target $(\lambda = 0.99)$. Additionally, the temperature plays an important role in the bound. We therefore measure the bound for a variety of $\tau$, for each $\lambda$. 
We include the results for 30 seeds in Figure \ref{fig:num_bound-single}.

\begin{figure}[tb!]
		\begin{tabular}{c c c c}
			\includegraphics[width=0.22\columnwidth]{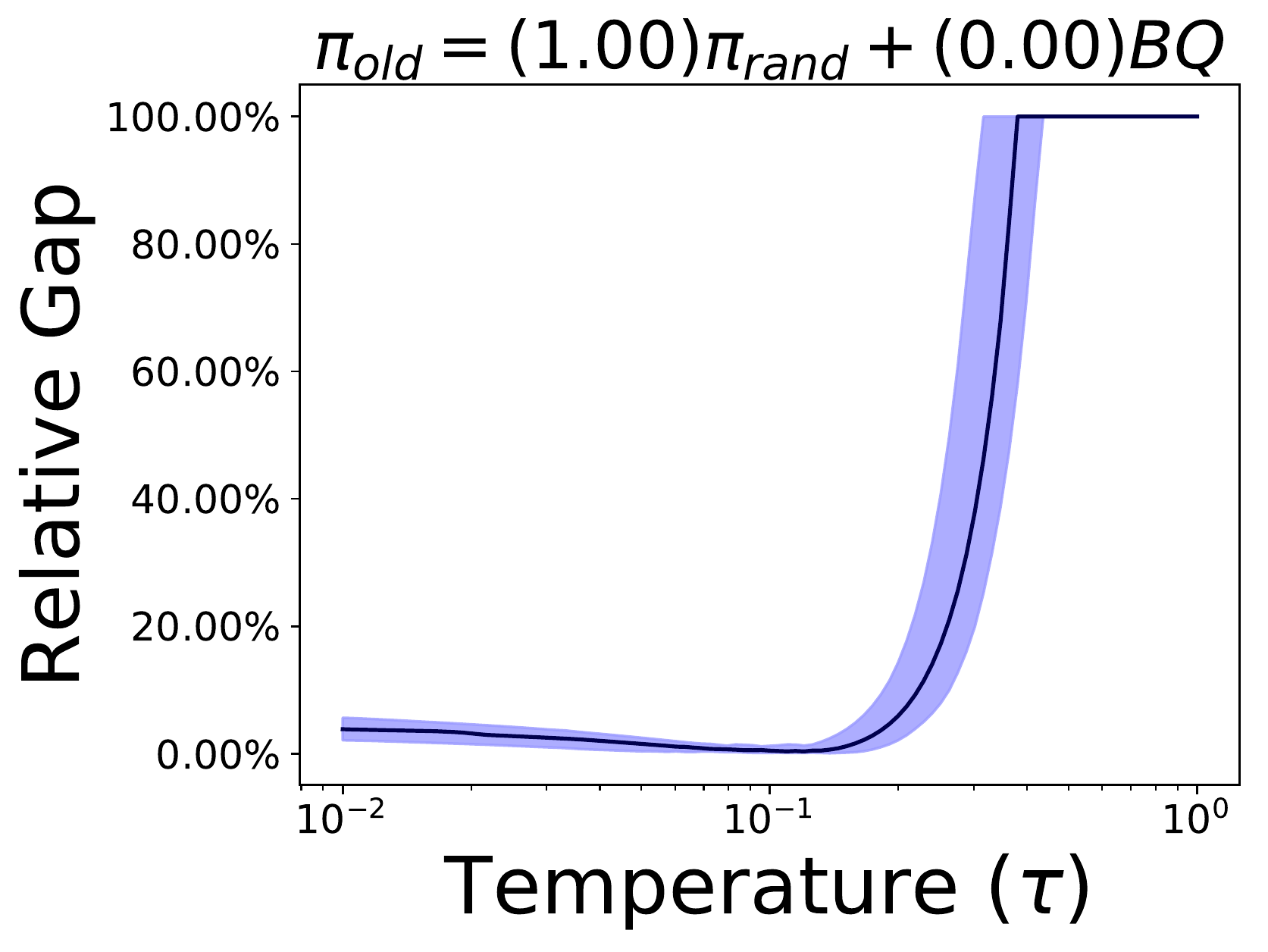} 
			&    
			\includegraphics[width=0.22\columnwidth]{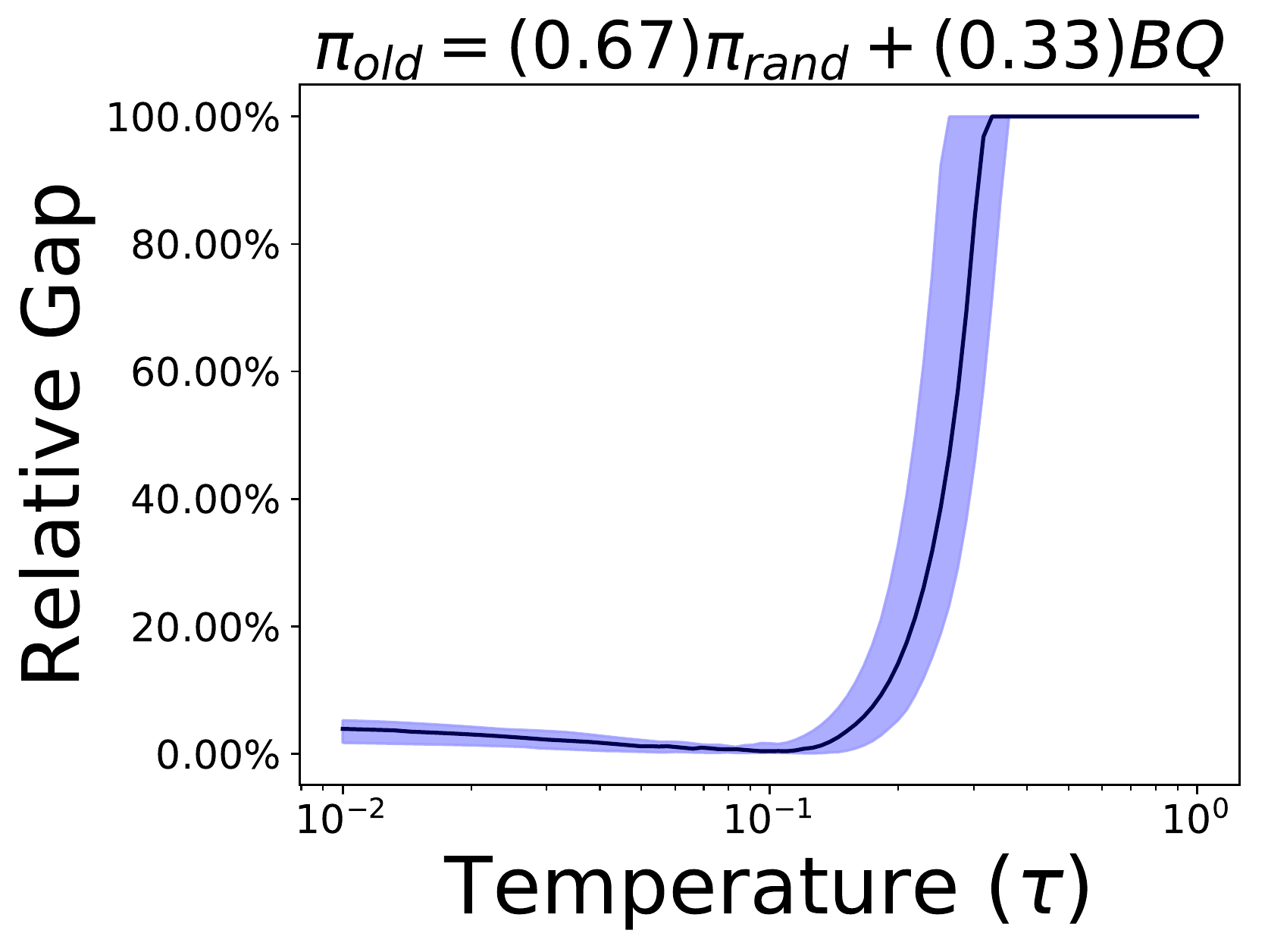}
			&
			\includegraphics[width=0.22\columnwidth]{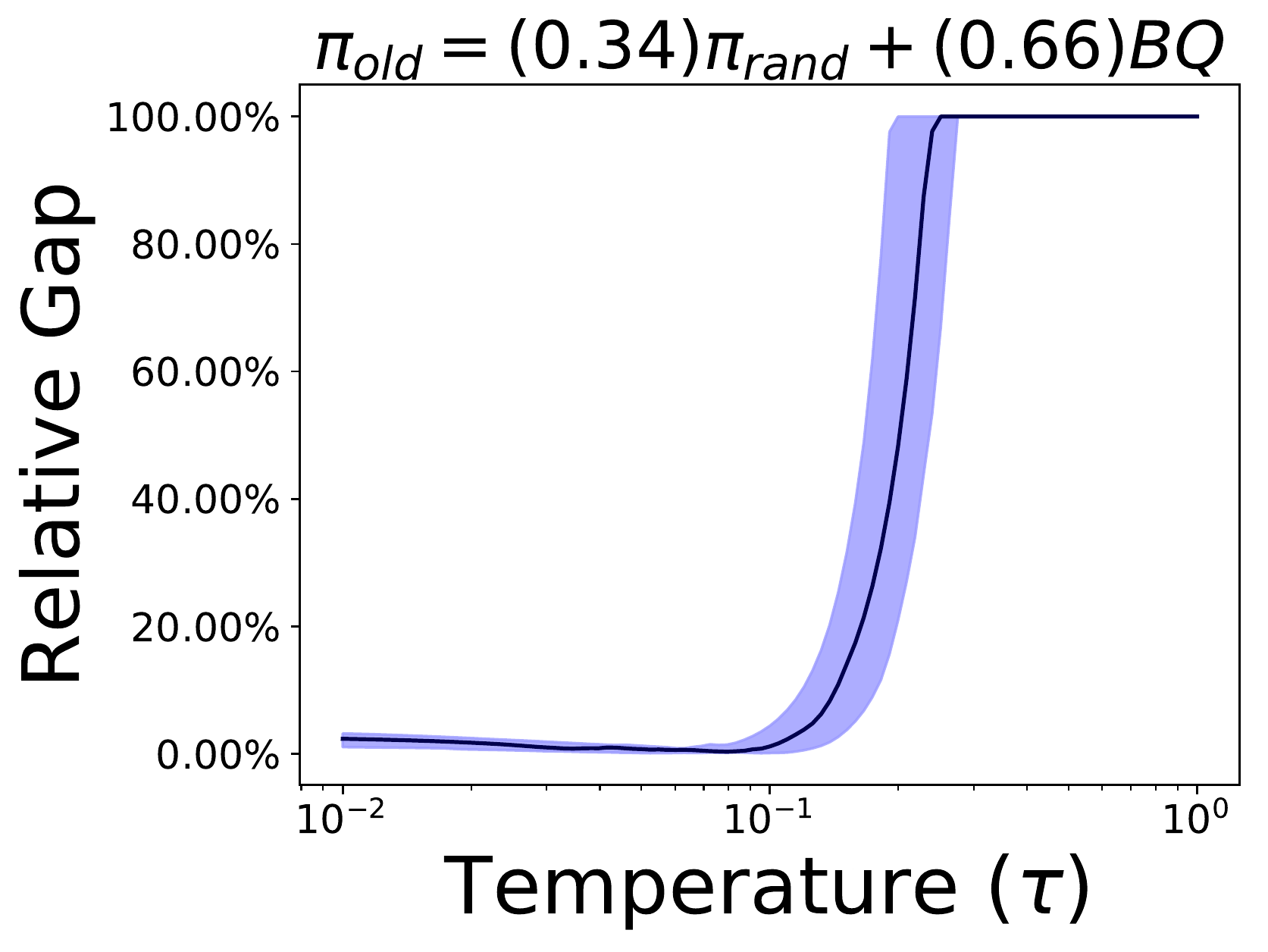}
			&
			\includegraphics[width=0.22\columnwidth]{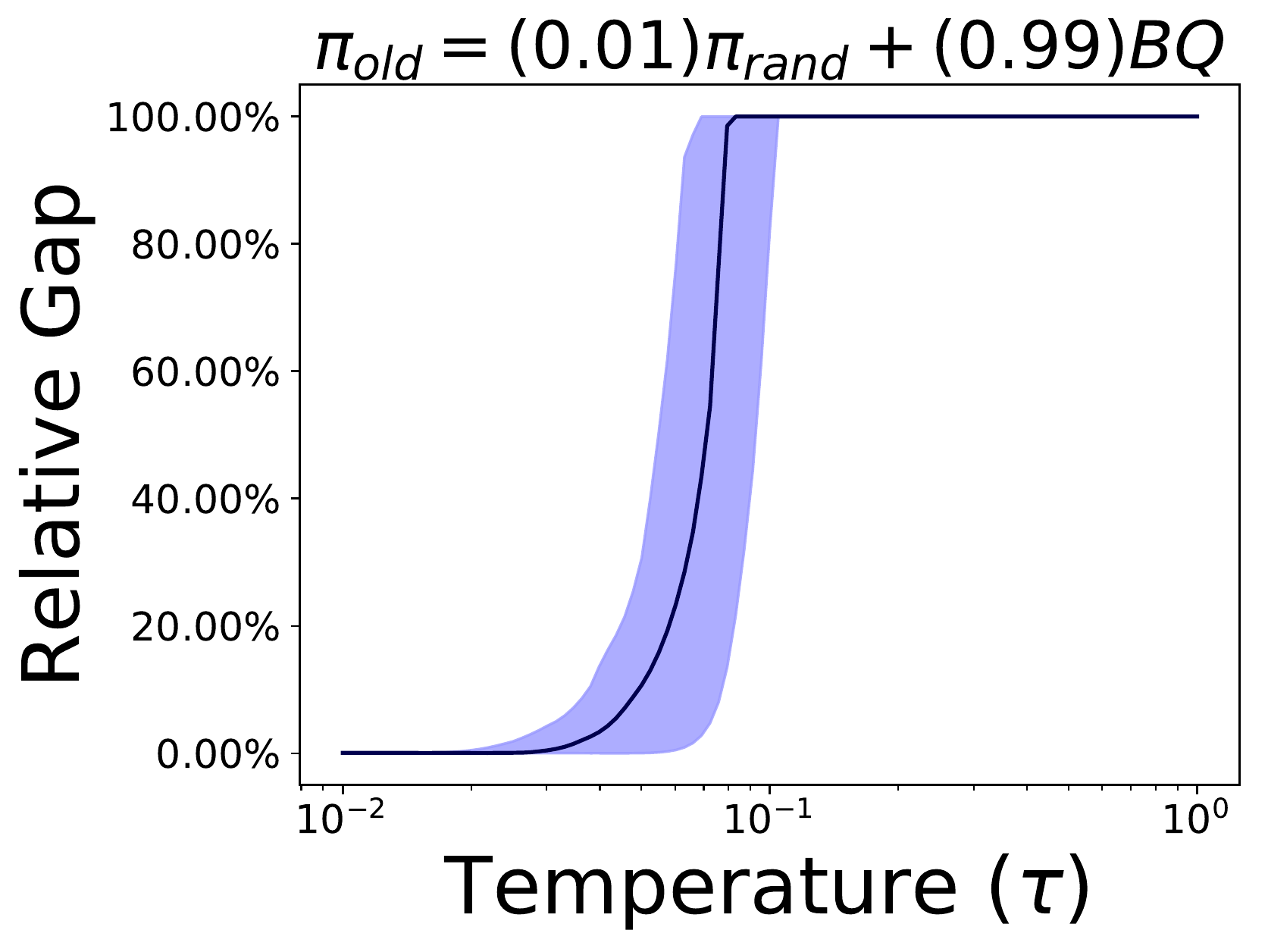}			
		\end{tabular} 
	\caption{The Relative Gap between the maximal $\Delta \FKL{}{}$ and our bound. We measure the ratio between our bound on the required level of reduction in Equation \eqref{eq_required_fkl} and the maximum possible FKL reduction (i.e. $\Delta \FKL{}{} = \FKL{\piold}{\BQ}$). The Relative Gap is 1-ratio, where values close to zero indicate near-maximal reduction is necessary for improvement in the soft objective. We show the result for policies far from the target (leftmost) to policies very close to the target (rightmost). The solid line is the median over 30 runs, with the shaded region showing the 25\% percentile and the 75\% percentile. } \label{fig:num_bound-single}
\end{figure}

We can see that the bound is very conservative and a near-maximum FKL reduction is necessary for many temperatures. These plots suggest that additional conditions are needed for FKL reduction to guarantee improvement, as we discuss further at the end of this section.

\subsubsection{Upper Bounding the RKL in Terms of the FKL}\label{sec:fkl-rkl-connection}

In this section we show that the FKL times a term that depends on the new and old policies gives an upper bound on the RKL. We discuss how this connection provides insight into why FKL reduction may not result in improvement. We omit dependence on the state in the following result, but it holds per state. This result is a straightforward application of a result from \citet{sason2016f}. The result uses the Rényi divergence of order $\infty$:
\begin{equation*}
	D_\infty (P \parallel Q) := \log \left( \max_i \frac{p_i}{q_i} \right)
\end{equation*}
where $P$ and $Q$ are two discrete probability distributions, with elements $p_i$ and $q_i$ respectively, that are absolutely continuous with respect to each other (i.e. one is never nonzero where the other one is zero).
\begin{restatable}[An Upper Bound on RKL in Terms of the FKL]{lemma}{fkl-rkl-bound}
	\label{prop:fkl-rkl-bound} Assume the action set is finite. For $\kappa(t) := \frac{t \log t + 1 - t}{t - 1 - \log t}$ where $\kappa$ is defined on $(0, 1) \cup (1, \infty)$ and all $s \in \cS$, 
	\begin{equation}
	\RKL{\pinew}{\BQold}(s) \leq \kappa\left(\exp(D_\infty(\pinew \parallel \BQold)) \right) \FKL{\pinew}{\BQold}(s) 
\end{equation}
\end{restatable}
\begin{proof}
To obtain this result, we bound the difference between the two choices of KL divergence. 
Define 
\begin{align*}
	\beta_1 &:= \exp (-D_\infty( P \parallel Q)) && \text{ and } \quad
	\beta_2 := \exp(-D_\infty(Q \parallel P)),
\end{align*}
Then, from Equation (161) in \citet{sason2016f}, as long as $P \neq Q$, we have
\begin{align*}
	\kappa(\beta_2) \leq \frac{\KL{P}{Q}}{\KL{Q}{P}} \leq \kappa(\beta_1^{-1}).
\end{align*}
Setting $P = \pinew$ and $Q = \BQold$ at a particular state $s$, where we omit the dependence on $s$, we have
\begin{align}
	\RKL{\pinew}{\BQold}(s) &\leq \kappa(\beta_1^{-1}) \FKL{\pinew}{\BQold}(s) \label{eq:kappa_ineq}\\
	&= \kappa\left(\exp(D_\infty(\pinew \parallel \BQold)) \right) \FKL{\pinew}{\BQold}(s) \nonumber
\end{align}
 \par\vspace{-0.5cm}
\end{proof}
To reduce the RKL as a function of $\pinew$, it thus suffices to reduce the right-hand side of the inequality. There are, however, problems with this approach. First, the bound itself may not be tight; even if we could reduce FKL and the multiplicand $\kappa(\beta_1^{-1})$, we still may not obtain a reduction in RKL. Second, we have only developed a mechanism to reduce the FKL, rather than the FKL and the multiplicand. A simple proxy could be to just focus on reducing the FKL.

The bound given above includes $\kappa(\beta_1^{-1})$, which also depends on $\pinew$. It is possible that in reducing the FKL, we actually also increase $\kappa(\beta_1^{-1})$, possibly offsetting our reduction of the FKL. For example, because of limited function approximation capacity, reducing the FKL might result in $\pinew$ covering a low-probability region of $\BQold$ in order to place some mass at multiple high-probability regions of $\BQold$. While such a $\pinew$ might have a moderate value of FKL, the resulting $D_\infty(\pinew \parallel \BQold)$ would be large, making $\beta_1^{-1}$ large. Correspondingly, because $\kappa$ is a monotone increasing function \citep{sason2016f}, $\kappa(\beta_1^{-1})$ would also be large. Consequently, $\kappa(\beta_1^{-1}) \FKL{\pinew}{\BQold}$ may not be be small enough to enforce a reduction in RKL.

On a more positive note, however, we know that the $\kappa$ term in \Cref{eq:kappa_ineq},
$$\kappa(\beta_1^{-1}) = \kappa\left(\exp(D_\infty(\pinew \parallel \BQold)) \right) = \kappa \left( \max_{i} \frac{(\pinew)_i}{(\BQold)_i}\right),$$
only grows logarithmically with $\max_i \frac{(\pinew)_i}{(\BQold)_i}$. Particularly,
\begin{equation*}
	\lim_{x \to \infty} \frac{\kappa\left(x\right)}{\log(x)} = 1 \quad \text{and so } \quad
	\kappa\left(x \right) = \Theta( \log(x) ).
\end{equation*}
Therefore, $\beta_1^{-1}$ has to increase by orders of magnitude to significantly increase $\kappa(\beta_1^{-1})$. 

A modification to the FKL reduction strategy could be to use $\kappa(\beta_1^{-1}) \FKL{\pinew}{\BQold}$ as an objective. The main difficulty with this approach is that $\beta_1$ is not differentiable because of the $\max$ operation in the calculation of $D_\infty(\pinew \parallel \BQ)$. It might be possible to approximate this maximum with smooth operations like $\mathrm{LogSumExp}$, but we leave exploration of this avenue for future work. 

\subsection{Summary and Discussion}

There are two key takeaways from the above results. First, the RKL has a stronger policy improvement result than the FKL as it requires only that the RKL of $\pinew$ be no greater than the RKL of $\piold$. In fact, RKL reduction under a certain state-distribution is a necessary and sufficient condition for improvement to occur.
Second, the FKL can fail to induce policy improvement, but sufficient reduction guarantees such improvement. The current bounds, although sufficient, are not a necessary condition for improvement to occur.

The theoretical results suggest that the FKL is inferior to the RKL for improving the policy, and that the FKL requires additional conditions. 
We hypothesize that the nature of these conditions has to do with the mean-seeking and mode-seeking behavior. 
Approximating a target distribution via RKL reduction is very sensitive to placing non-negligible probabilities in regions where the target distribution is close to zero. The FKL, on the other hand, focuses on placing high probabilities in the regions where the target probability is high. It is not hard to see that reducing FKL can increase RKL, if for example we approximate a bimodal distribution with a unimodal one, or if we use a Gaussian parameterization and the target distribution is highly skewed. As we discussed, to obtain improvement, we need a sufficient reduction in the FKL to ensure RKL reduction. Our bound provided one such condition, but as we found with numerical experiments, this bound was relatively loose. It remains an open question to understand the conditions that guarantee improvement, and when FKL reduction does not give RKL reduction. 

The settings where the RKL and FKL are significantly different---meaning that FKL reduction can actually cause the RKL to increase---may not be as prevalent in practice. For example, if the target distributions are unimodal and symmetric, we may find that RKL and FKL have similar empirical performance. We will see in our experiments that the FKL is often able to induce policy improvement in practice, suggesting a gap between the theory developed and the practical performance. 

An important next step is to leverage these policy improvement results to prove convergence to an optimal policy under approximate greedification. When completely reducing the RKL per state, it is known that the iterative procedure between policy evaluation and greedification with RKL minimization converges to the optimal policy in the policy set \citep[Theorem 1]{haarnoja2018soft}. This result should similarly hold, under only RKL reduction, as long as that reduction is sufficient on each step. A next step is to understand the conditions on how much reduction is needed per step, for both the RKL and FKL, to obtain this result.

\section{Empirical Results Comparing the FKL and RKL}

In this section, we complement the theoretical results with an investigation of the other practical properties of the FKL and RKL. The theory focused on their differences in terms of inducing policy improvement. We also care about (1) the optimization behavior of these greedification operators, when using (stochastic) gradient descent and (2) the nature of the policies induced during learning, in particular whether stochasticity collapses quickly and differences in encourage exploratory behavior. We may also want to understand generally how these two approaches perform in the wild, on a suite of problems. We investigate these three questions empirically in this section. 

\subsection{Optimization Behavior in Microworlds} 
The goal in this section is to understand differences between FKL and RKL in terms of (1) the loss surface and (2) the behavior of iterates optimized under the losses. By behavior, we mean whether the iterates reach multiple local optima, how stable iterates under that loss are, and how often iterates reach the global optimum (or optima). Given the fine-grained nature of our questions, we focus upon small-scale environments, which we call \textit{microworlds}. Doing so allows us to avoid any possible confounding factors associated with larger, more complicated environments, and furthermore allows us to more fully separate any issues to do with stochasticity. 

We use two continuous action low-dimensional microworlds to allow us to visualize and thoroughly investigate behavior.
Our first microworld is a \textbf{Bimodal Bandit} in \Cref{fig:bimodal-bandit}. For continuous actions, we designed a continuous bandit with action space $[-1, 1]$ and reward function $Q(a) := \exp( -\tfrac{1}{2} (\tfrac{2 a + 1}{0.2})^2 ) + \tfrac{3}{2} \exp(-\tfrac{1}{2} (\tfrac{2 a - 1}{0.2})^2)$. The two unequal modes at -0.5 and 0.5 enable us to test the mean-seeking and mode-seeking behavior as well as simulate a realistic scenario where the agent's policy parameterization (here, unimodal) cannot represent the true distribution (bimodal). 

Our second microworld is the \textbf{Switch-Stay} domain in \Cref{fig:Switch-Stay}. From $s_0$, action $0$ (stay) gives a reward of 1 and transitions to state $0$. From $s_1$, action 0 gives a reward of 2 and transitions to $s_1$. From $s_0$, action 1 (switch) gives a reward of -1 and transitions to $s_1$, while action 1 from $s_1$ gives a reward of 0 and transitions to $s_0$. 
To adapt this environment to the continuous action setting, we treat actions $> 0$ as switch and actions $\leq 0$ as stay.\footnote{Note that we also compared the RKL and FKL for the discrete action variant of Switch-Stay. Under a softmax parameterization, we found no significant differences between the RKL and FKL.}  We set $\gamma = 0.9$ to ensure that the optimal action from $s_0$ is to switch, which ensures the existence of a short-term/long-term trade-off inherent to realistic RL environments. 

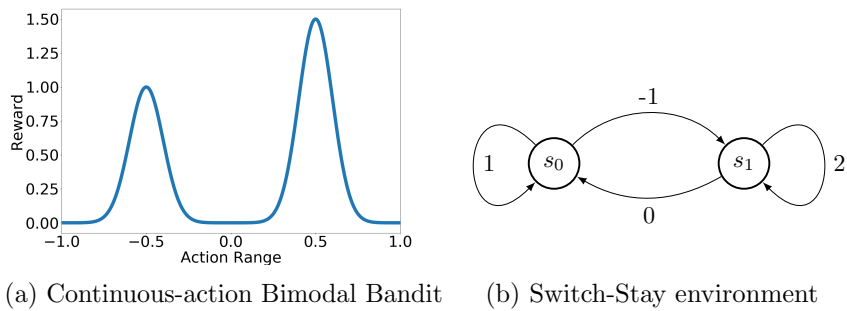
\begin{figure}[!htb]
  \centering
   \begin{subfigure}[b]{0.38\linewidth}
     \centering
    \includegraphics[width=\linewidth]{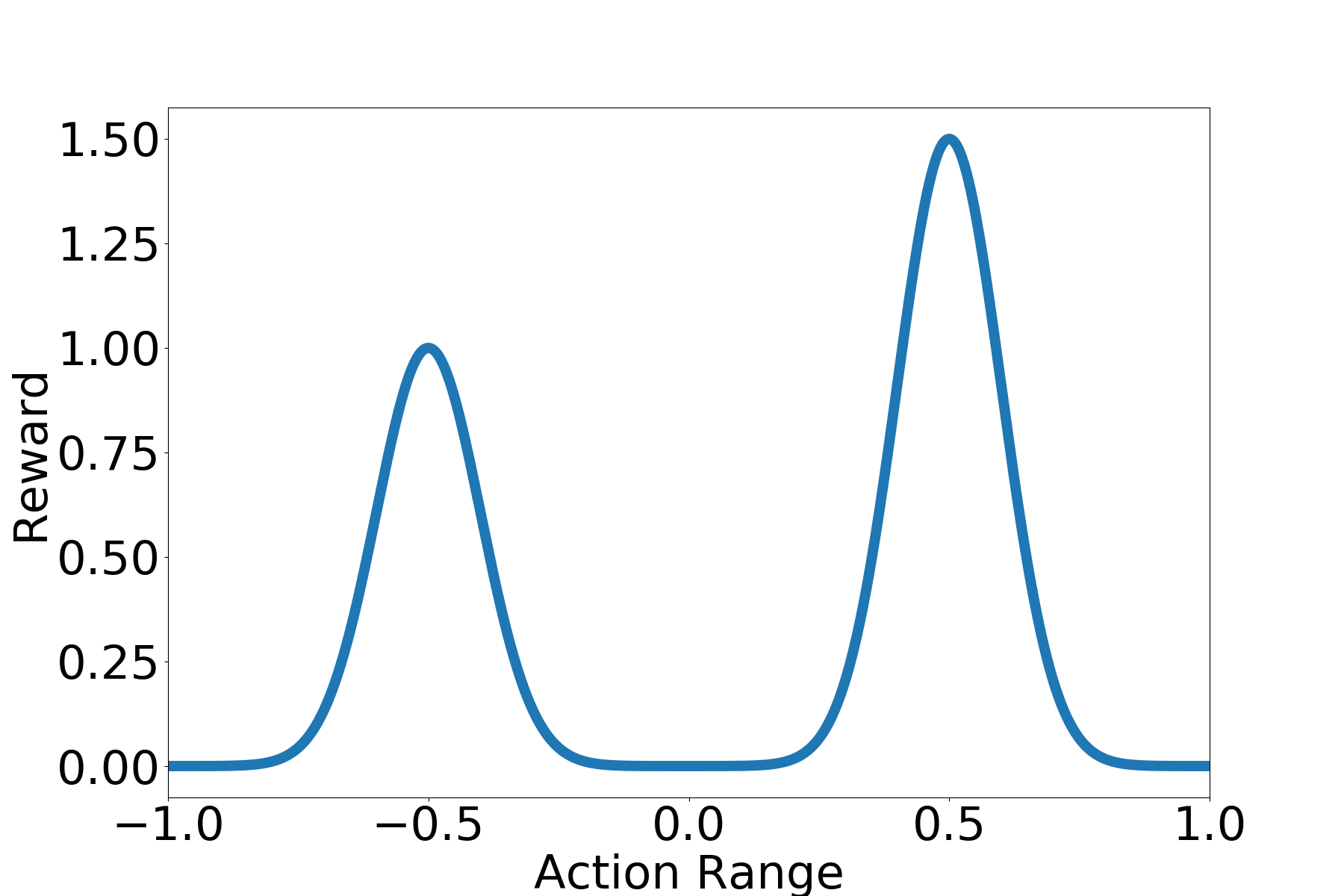}       
     \caption{Continuous-action Bimodal Bandit}
    \label{fig:bimodal-bandit}
   \end{subfigure}
\begin{subfigure}[b]{0.35\linewidth}
 \resizebox{\columnwidth}{!}{%
\begin{tikzpicture}[auto,node distance=10mm,>=latex,font=\small]
    \tikzstyle{round}=[thick,draw=black,circle]

    \node[round] (s0) {$s_0$};
    \node[round,right=20mm of s0] (s1) {$s_1$};

    \draw[->] (s0) to [out=45, in=135] node {-1} (s1);
    \draw [->] (s0) to [out=135,in=225, loop] node {1} (s0);
    \draw [->] (s1) to [out=210,in=330] node{0} (s0);
    \draw [->] (s1) to [out=45,in=-45, loop] node {2} (s1);
\end{tikzpicture}}
    \caption{Switch-Stay environment}
    \label{fig:Switch-Stay}
\end{subfigure}%
\caption{The Microworld environments used to investigate and visualize the optimization behavior of FKL and RKL.}
\end{figure}

\subsubsection{Implementation Details}

All policies are tabular in the state. To calculate the FKL and RKL under continuous actions, we use the Clenshaw-Curtis \citep{clenshaw1960method} numerical integration scheme with 1024 points from the package quadpy,\footnote{\url{https://pypi.org/project/quadpy/}} excluding the first and the last points at -1 and 1 because of numerical stability. We use the true action-values when calculating the KL losses. In the Bimodal Bandit, the action-value is given by the reward function, while in Switch-Stay it is calculated (i.e., not learned). To calculate the Hard FKL, we use the true maximum action as determined by the environment. For Switch-Stay, we calculate and optimize the mean KL across the two states.


For policy parameterizations, in continuous action settings we use a Gaussian policy with mean and standard deviation learned as $(\hat{\mu}, \log(1+\exp(\hat{\sigma}))$ 
The action sampled from the learned Gaussian is passed through $\tanh$ to ensure that the action is in the feasible range $[-1, 1]$ and to avoid the bias induced in the policy gradient when action ranges are not enforced \citep{chou2017improving}. 

Finally, we use the RMSprop optimizer \citep{tieleman2012lecture}. Overall trends for Adam \citep{kingma2014adam} were similar to those for RMSprop, while results for SGD resulted in slower learning for both FKL and RKL and a wider range of limit points, most likely due to oscillation from the constant step-size. We focus on RMSprop here to avoid any confounding factors associated with momentum.

\begin{figure}[htb]
    \centering
    \includegraphics[width=0.99\columnwidth]{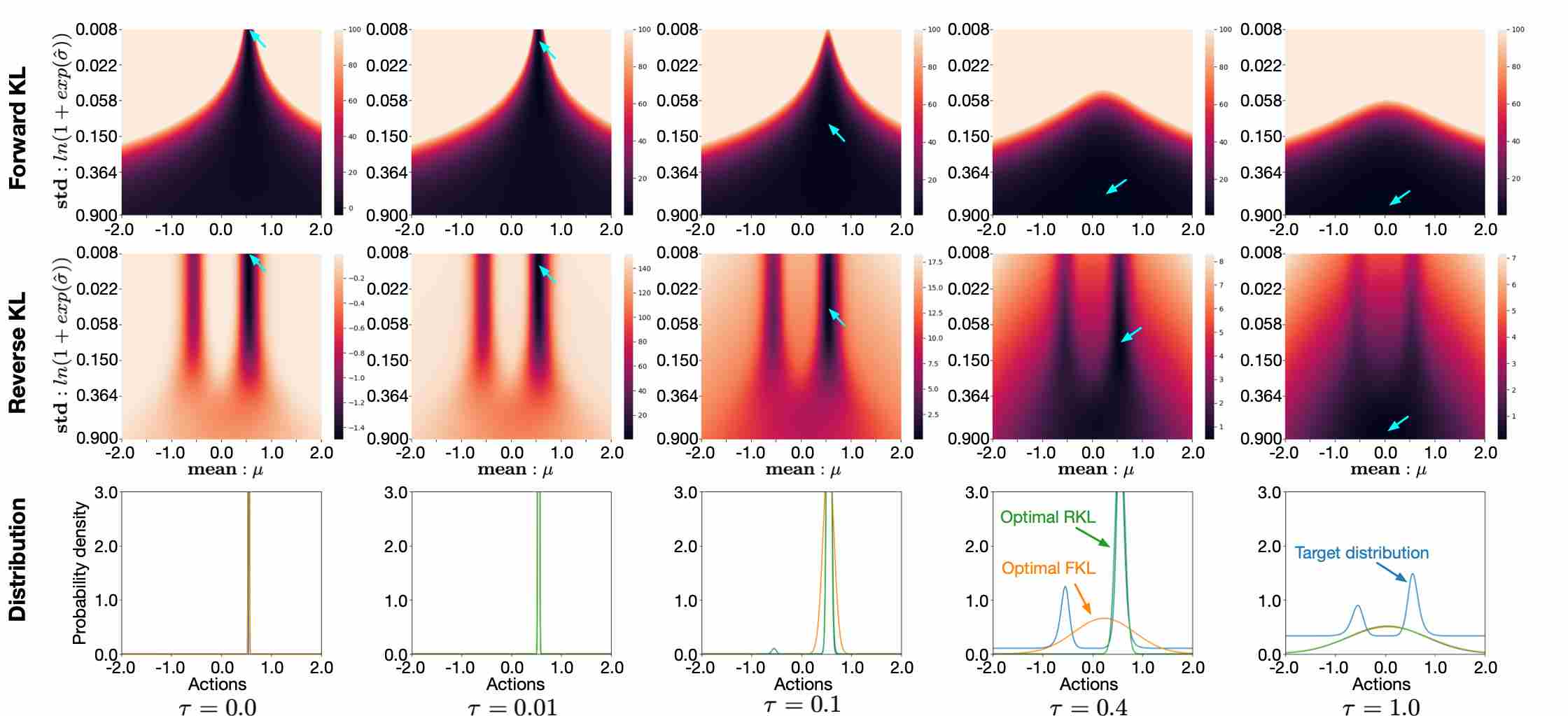}
    \caption{KL loss over mean and standard deviation across temperature. The heatmaps depict the loss for each mean and standard deviation pair. The last row depicts the target distribution over which the KL loss is optimized. Note that the actual action taken applies $\tanh$ to the samples of the resulting distribution (i.e., the optimal mean is at $\tanh^{-1}(0.5) \approx 0.55$). FKL loss has been upper-bounded for better visualization of minima. Arrows indicate the global minimum.}
    \label{fig:bandit-heatmap}
\end{figure}

\subsubsection{Loss Surface in the Bimodal Bandit}
We might expect the FKL to have a smoother loss surface. Given that policies often are part of an exponential family (e.g., softmax policy), having the policy $\pi$ be the second argument of $\KL{p}{q}$ removes the exponential of $\pi$, resulting in an objective that is an affine function of the features. For example, if $\pi(a \mid s) \propto \exp(\phi(s, a))$ for features $(s, a)$, the resulting FKL becomes a sum of a term than is linear in $\phi(s, a)$ and a term involving $\mathrm{LogSumExp}(\phi)$, which is convex.

We visualize the KL loss surfaces in \Cref{fig:bandit-heatmap} with five different temperatures. The surfaces suggest the following. 

\textbf{1)} The FKL surface has a single valley, while the RKL surface has two valleys that are separated from one another. In this sense, the FKL surface seems much smoother than the RKL surface, suggesting that iterates under the FKL will more likely reach the global optimum than iterates under the RKL, which seem likely to fall into either of the valleys. 

\textbf{2)} The smoothness of the RKL landscape increases with temperature as the gap between the peaks becomes less steep. A higher temperature also causes the valley in the FKL map to become less sharply peaked, and for the optimal $\mu$ to move closer to 0. 

\textbf{3)} The optimal $\mu$ for the FKL seems to move more quickly to zero, as $\tau$ increases, than the optimal $\mu$ for the RKL, although both eventually reach 0. It is possible that the FKL may become suboptimal sooner than the RKL as $\tau$ increases, likely because it is mean-seeking. Interestingly, even the RKL appears to be mean-seeking for high $\tau$, because selecting one mode would have lower entropy.  

As a note, it may seem strange that two valleys exist for the RKL at $\tau = 0$ given that the target distribution is unimodal. When $\tau = 0$, however, the loss function is no longer a distributional loss; that is, we are no longer minimizing any pseudo-distance between the policy and a distribution.

\subsubsection{Solution Quality in Switch-Stay}\label{sec:ss_behavior}

In this section, we investigate the properties of the solutions under the FKL and RKL for an environment with more than one state. 
Given our previous results, we might expect the FKL to result in better solutions, because FKL iterates can reach the global optimum more easily. But this depends on the quality of this solution. The global minimum of the FKL objective may not correspond well with the optimal solution of the original, unregularized objective, as we investigate below. 

Th Switch-Stay environment is appropriate to investigate the quality of the stationary points of the RKL and FKL for two reasons. First, it is a simple instantiation of the full RL problem, we are interested in understanding any possible differences between FKL and RKL in the presence of short-term/long-term trade-offs. On Switch-Stay, the agent incur a short-term penalty by switching from state 0 to state 1, but longer term this maximizes return.
Second, the Switch-Stay environment facilitates visualization. Since the MDP has only two states, we can plot any value function as a point on a 2-dimensional plane. In particular, one can view the entire space of value functions, shown recently to be a polytope in the discrete-action setting \citep{dadashi2019value}.

We can similarly visualize the value function polytope for continuous actions in Switch-Stay. Recall that we treat any action $ \leq 0$ as stay, and any action $ > 0$ as switch. To calculate the value function corresponding to a continuous policy $\pi$, we convert $\pi$ to an equivalent discrete policy $\pi_{\mathrm{discrete}}$ of the underlying discrete MDP. The conversion requires the calculation of the probability that $\pi$ outputs an action $\leq 0$ in each state, which we do with numerical integration of the policy PDF. We then calculate the value function of $\pi$ as $(I - \gamma P_{\pi_{\mathrm{discrete}}})^{-1}r_{\pi_{\mathrm{discrete}}}$, where $P_\pi$ and $r_\pi$ are respectively the transition matrix and the reward function induced by $\pi$.

For the hard FKL, we require access to the greedy action of the action-value function. In the continuous-action setting, this greedy action is usually infeasible to obtain. For the purposes of this experiment, if the greedy action is stay, we represent it in $[-1, 1]$ by drawing a uniform random number from $[-1, 0]$. If the greedy action is switch, we represent it as a uniform random number in $[0, 1]$. This design choice is meant to simulate noisy access to the greedy action in practice. 

For all of these experiments, we initialized means in the range $(-0.95, 0.95)$. All experiments are run for 500 gradient steps and each experiment has 1000 iterates. We plot the value function of the final policy for each iterate and experiment in \Cref{fig:cont-ss-poly-0.005} 
by visualizing the value function polytope \citep{dadashi2019value}. That is, for finite state and action spaces, the set of all value functions is a polytope (a union of of convex polytopes). By plotting the value functions of our policies on the value function polytope, we are able to concisely gauge the performance of an algorithm relative to other algorithms. 

\begin{figure}[t]
  \centering
    \includegraphics[width=\columnwidth]{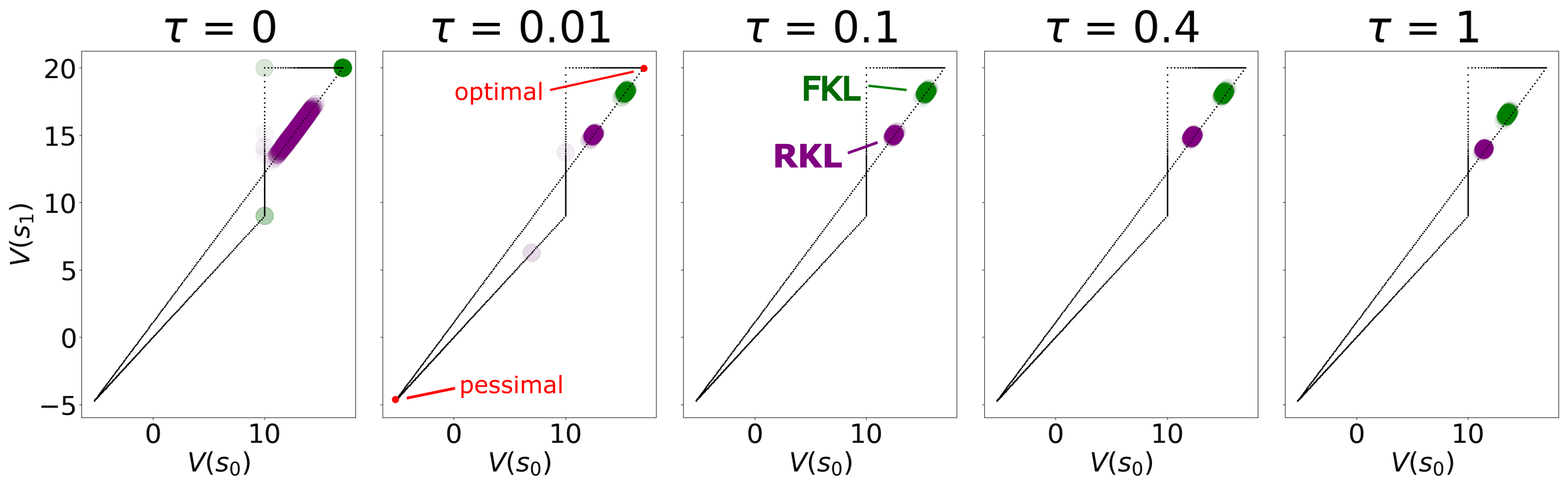}
  \caption{We plot the final value functions on the continuous-action version of Switch-Stay after 500 gradient steps for 1000 iterates, for a variety of temperatures. For clarity, we plot in black the boundary of the value function polytope. All points on the boundary, and all points in the interior of this shape, correspond to value functions of some policy. Each iterate is represented by a translucent dot. RMSprop with an initial learning rate of 0.01 was used. }
  \label{fig:cont-ss-poly-0.005}
\end{figure}

\textbf{1)} FKL with $\tau = 0$ converged noticeably slower than the other temperatures, which seems to be an artifact of our encoding of continuous actions to the underlying discrete dynamics of Switch-Stay, and the fact that we used random tie-breaking when computing the $\argmax$ for hard FKL. 

\textbf{2)} RKL iterates converge slightly faster than FKL iterates across all temperature settings. RKL iterates with $\tau = 0$ sometimes converged to non-optimal value functions on the corners. 

\textbf{3)} The limiting value functions of the FKL iterates seem more suboptimal than the limiting value functions of the RKL iterates. The latter are closer to the optimal value function of the original MDP. This result is consistent with our observations in the continuous bandit. Although the FKL optimum may be more easily reached, that optimal point may be suboptimal with respect to the unregularized objective.

\subsubsection{The Impact of Stochasticity in the Update}\label{sec:stochastic-microworld}
Although with discrete actions it is practical to sum across all actions when calculating the KL losses, difficulty emerges with high-dimensional continuous action spaces. Quadrature methods scale exponentially with the dimension of the action-space, leaving methods like Clenshaw-Curtis impractical. Monte-Carlo integration---in this case sampling actions from the current policy to estimate the update---seems the only feasible answer in this setting. An important distinction between FKL and RKL, therefore, is how they perform when using a noisier estimate of their updates.

We repeated the experiment in Switch-Stay, now
using Monte-Carlo integration instead of Clenshaw-Curtis quadrature to estimate the update for a state, averaged across the sampled actions. As discussed in Section \ref{sec:api-alg}, we can estimate the gradients of the RKL and FKL using sampled actions rather than full integration. 
The hard and soft RKL gradient updates are estimated using sampled actions from the current policy $\pi$, and the soft FKL gradient update is estimated using {weighted importance sampling}. 
Note that since Hard FKL only depends upon the maximum action, we do not modify the algorithm in this experiment. 


\begin{figure}[!htb]
  \centering
  \begin{subfigure}[b]{0.85\linewidth}
    \centering
    \includegraphics[width=\columnwidth]{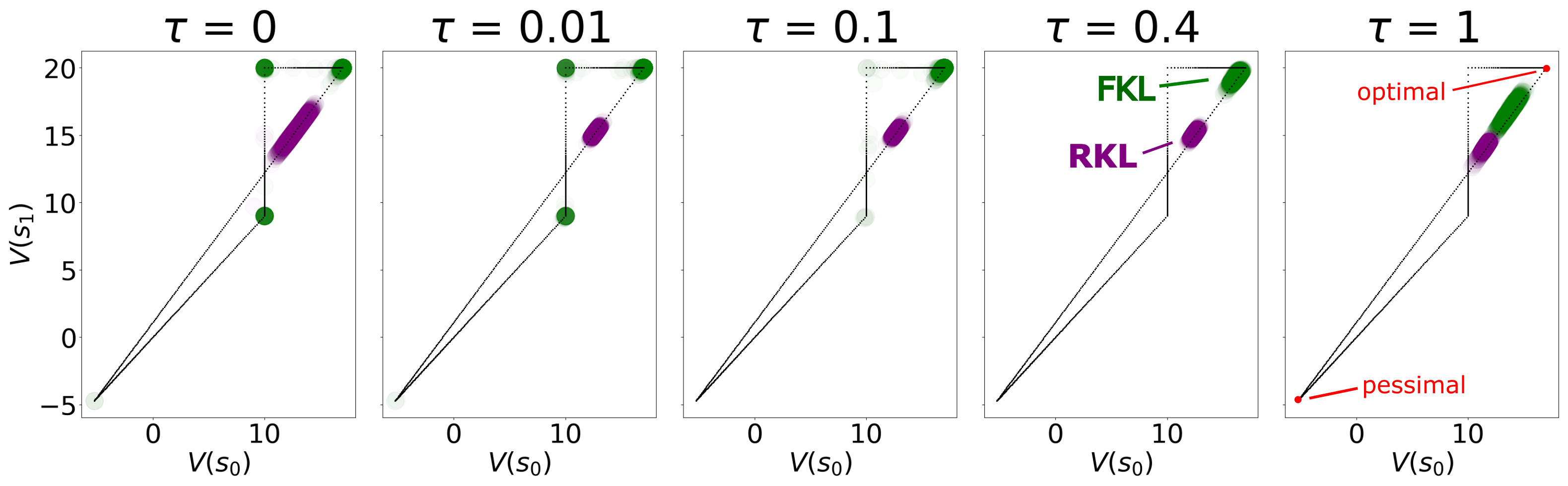}
    \caption{10 sample points.}
    \label{fig:10-sample-Switch-Stay}
  \end{subfigure}
  
  \begin{subfigure}[b]{0.85\linewidth}
        \centering
        \includegraphics[width=\columnwidth]{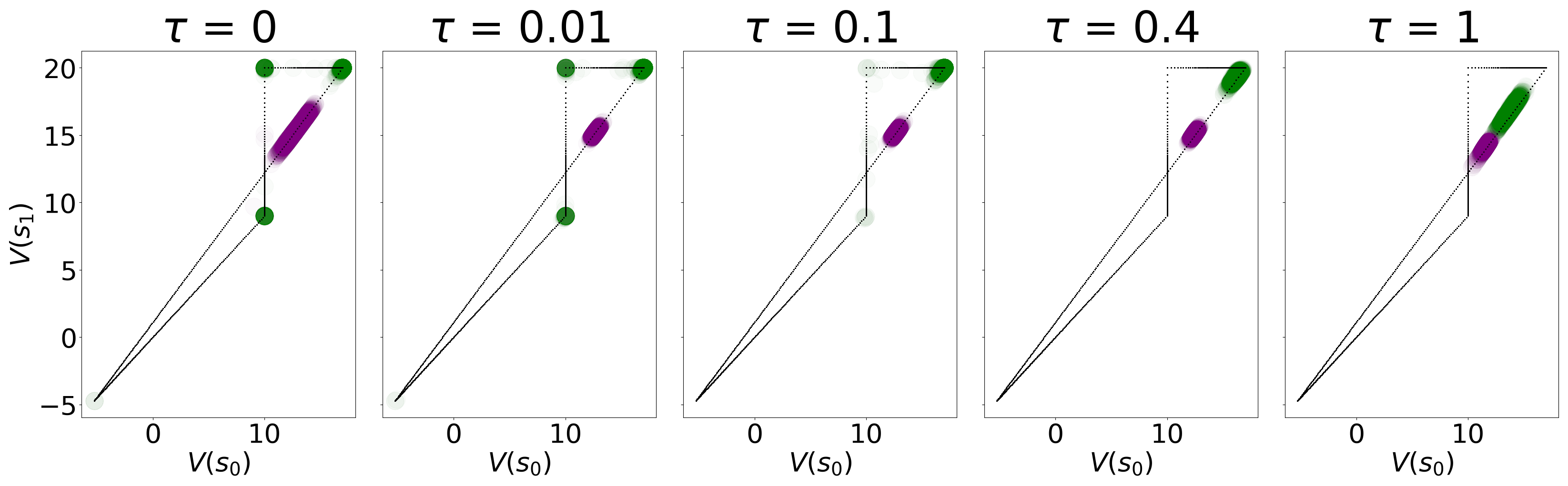}    
        \caption{500 sample points.}
        \label{fig:500-sample-Switch-Stay}
  \end{subfigure}
  \caption{The final value functions in Switch-Stay with stochastic estimation of the gradient using sampled actions, and otherwise the same settings as in Figure \ref{fig:cont-ss-poly-0.005}.}\label{fig_samples_ss}
\end{figure}

We can see in Figure \ref{fig_samples_ss} that the RKL is much more variable than the FKL with a smaller number of sampled actions (10 vs 500). 
RKL iterates converged to minima to which they did not converge in the Clenshaw-Curtis regime, even for 500 sampled actions. In \Cref{fig:500-sample-Switch-Stay}, there is an interesting trend across temperatures. Temperatures below $0.4$ induced many suboptima far from the optimal value function, while temperatures 0.4 and 1 seemed better at clustering RKL iterates near the optimal value function. On the other hand, FKL seemed relatively insensitive both to the temperature and the number of sample points. This relative insensitivity could be due to having a smoother loss landscape to begin with, which tends to direct iterates to a single global optimum. As noted before, though, this global optimum is quite suboptimal with respect to the unregularized MDP; nonetheless, the FKL can reach its optimum more robustly under noise.

\subsection{Exploration Differences between the FKL and RKL} 

The focus of this section is to study whether there are any significant differences in exploration when using the FKL and RKL. To obtain sufficient exploration, the approach should induce a state visitation distribution whose support is larger, namely that covers more of the state space. Accumulating more transitions from more diverse parts of the state space presumably allows for more accurate estimates of the action value function, and hence more reliable policy improvement. Entropy-regularized RL, as it is currently formulated, only benefits exploration by proxy, through penalizing the negative entropy of the policy. In the context of reward maximization, entropy is only a means to an end; at times, the means may conflict with the end. A policy with higher entropy may have a more diverse state visitation distribution, but it may be prevented from exploiting that information to the fullest capacity because of the penalty to negative entropy.

There has been some work discussing the potential differences between the FKL and RKL for exploration. 
\citet{neumann2011variational} argues in favour of the reverse KL divergence as such a resulting policy would be cost-averse, but also mentions that the forward KL averages over all modes of the target distribution, which may cause it to include regions of low reward in its policy. While in principle it may seem like a bad idea to include those, we note that the value function estimates can be highly inaccurate \citep{ilyas2018deep}, causing this inclusion to possibly be beneficial for exploration. Indeed, \citet{norouzi2016reward} use the forward KL divergence to induce a policy that is more exploratory.

We hypothesize that the FKL benefits exploration by causing the agent's policy to commit more slowly to actions that apparently have high value under the current value function estimate. This could benefit exploration both because it causes the agent to explore more and avoids incorrectly committing too quickly to value function estimates that are inaccurate. This non-committal behavior may help the policy avoid converging quickly to a suboptimal policy. Conversely, we hypothesize that the RKL will more quickly reduce the probability of actions that seem low-valued under our current (potentially inaccurate) value estimates. We investigate the differences first in continuous-action Switch-Stay, and then in a Maze environment with a misleading, suboptimal goal. Once again, we observed few differences for the discrete action experiments, even in the Maze environment; we include these results in \Cref{app:exploration}.

\subsubsection{Exploration under Continuous Actions in Switch-Stay}

We first revisit the Switch-Stay environment, and examine if the FKL and RKL exhibited differences in the variance of their policies. Recall that we examined the value functions for the final policies under the FKL and RKL, in \Cref{fig:cont-ss-poly-0.005}. We noted that the FKL converged to more suboptimal policies than the RKL with the same $\tau$, when evaluated under the unregularized objective. A natural hypothesis is that the FKL policy is a more stochastic policy, which is further from the optimal deterministic policy. 

To see if this is the case, we plot the final standard deviations of the learned policies for the learning rate of 0.01. In \Cref{fig:final-ss-probs-0}, we see that the final FKL iterates have higher standard deviation for each $\tau$, meaning that the final policies are further from the optimal deterministic policy of the unregularized MDP. Put informally, the FKL tends to commit less than the RKL.  This means that even when using a target Boltzmann policy with the same level of entropy-regularization, that level of entropy induces a more stochastic policy under the FKL.

\begin{figure}[!htb]
  \centering

    \begin{subfigure}[b]{0.4\linewidth}
    \centering
    \includegraphics[width=\columnwidth]{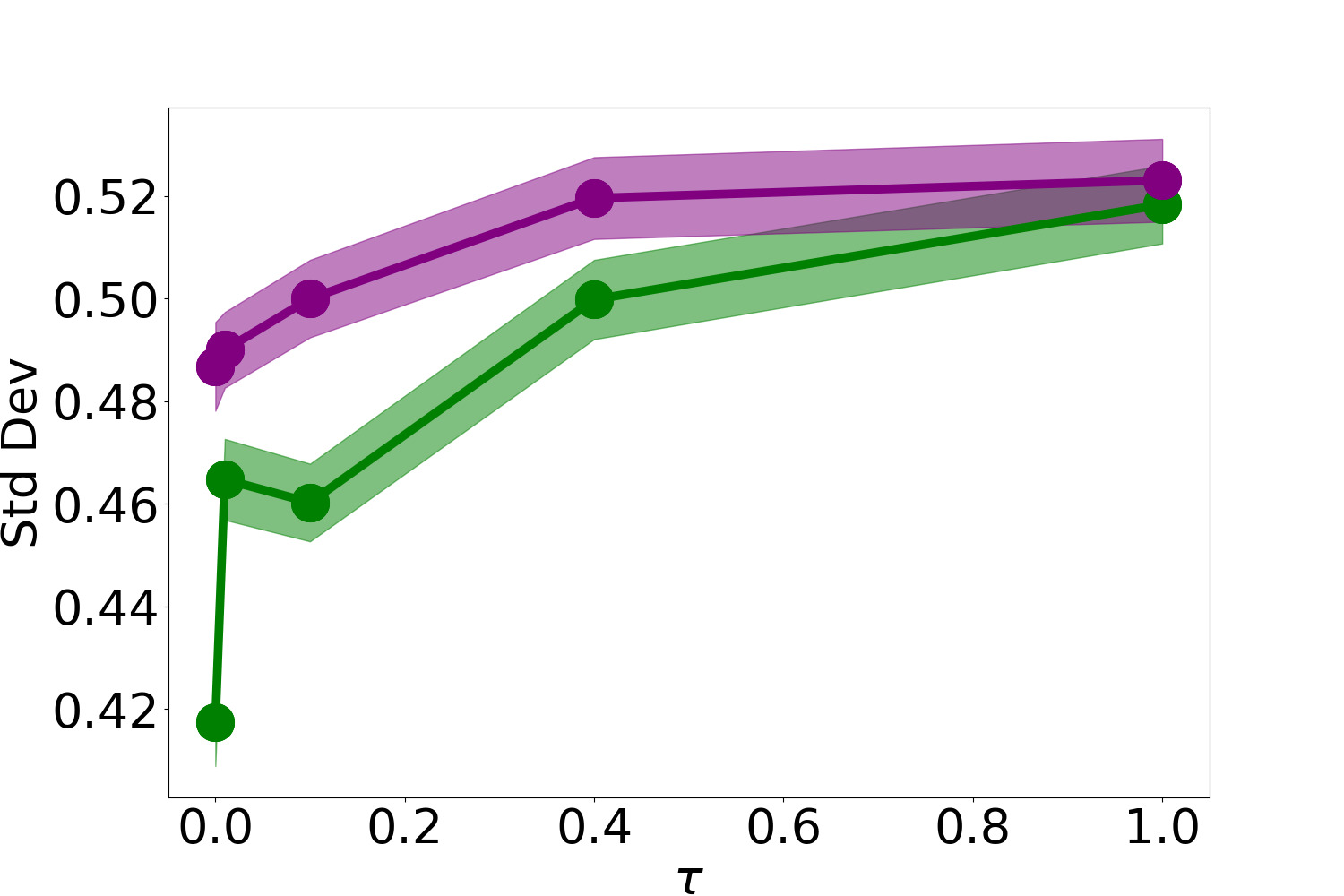} 
    \caption{After 10 updates.}
  \end{subfigure}%
  \begin{subfigure}[b]{0.4\linewidth}
    \centering
    \includegraphics[width=\columnwidth]{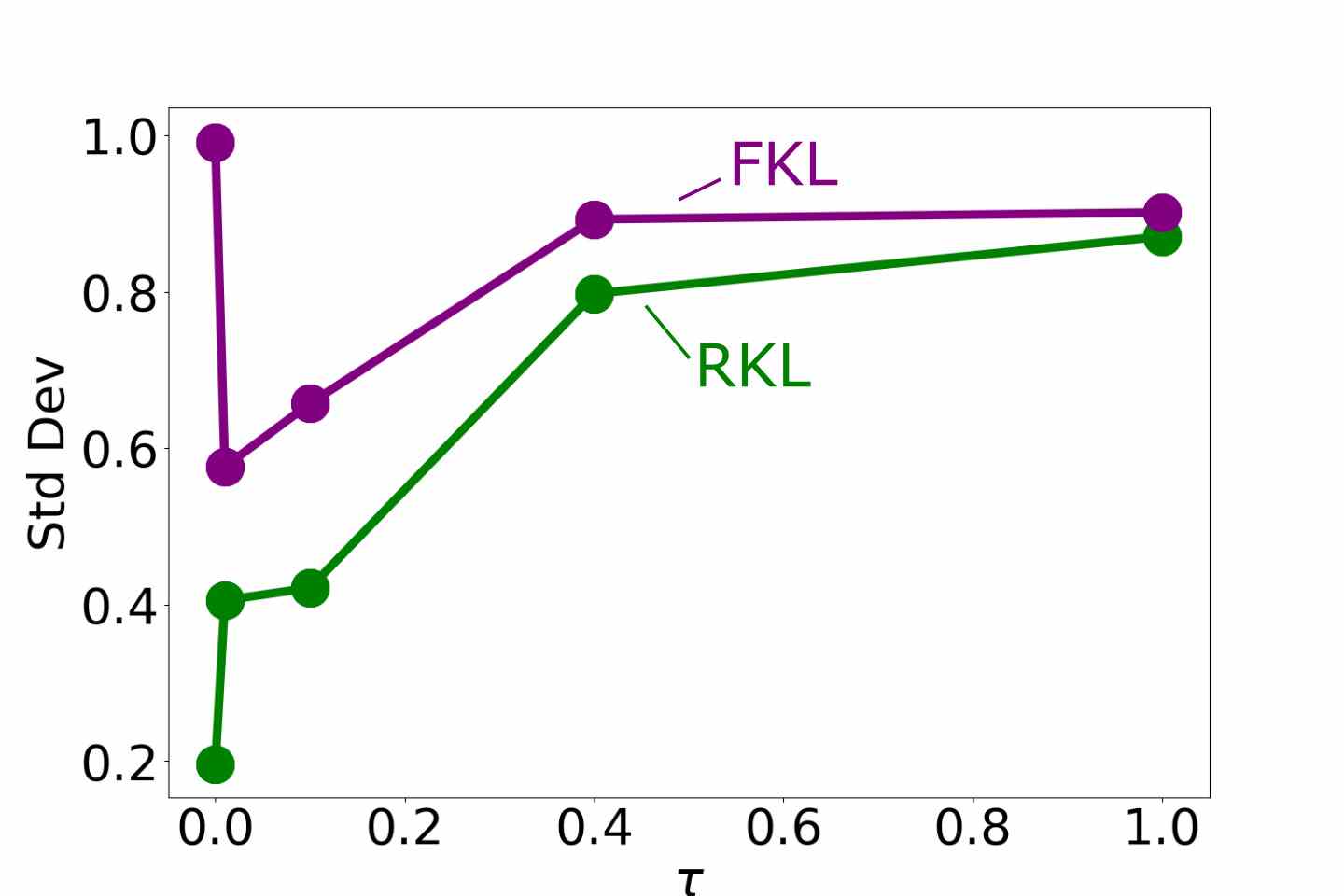} 
    \caption{After 500 updates.}
  \end{subfigure} 
  
  \caption{We plot the standard deviation (y-axis) on the continuous-action version of Switch-Stay, with the temperature varied on the $x$-axis. RMSprop is used with learning rate $0.01$. Each dot is the mean of 1000 iterates, with shaded areas corresponding to the standard error. This plot is for state 0, and recall that every action $\leq 0$ is treated as ``stay'' and every action $> 0$ is treated as ``switch''.
  The large final standard deviation for $\tau = 0$ with the FKL is an artifact of our encoding of the maximum action as a uniform random point in either [0, 1] or [-1, 0], depending on if the maximum action is respectively stay or switch.}
  \label{fig:final-ss-probs-0}

\end{figure}

\subsubsection{Exploration under Continuous Actions in Misleading Maze} \label{sec:expl_maze_cont}

We next investigate exploration behavior in a more difficult exploration problem: a maze with a misleading goal. We also include an experiment in a maze with discrete actions, but find behavior between the FKL and RKL is very similar (see \Cref{app:exploration}).

\Cref{fig:cont_maze_example} illustrates the maze we use in this section. The agent starts in the center of the green block and has to get close to the center of the blue block. The red blocks are the obstacles and the yellow block corresponds to the misleading exit, which terminates the episode but gives a reward much lower than the real exit. Since the misleading exit is closer to the starting point than the actual exit, the agent can only find out about the higher reward after exploring the maze. The coordinates the agent sees are normalized to the range $[-1,1]$, but its actions are given as a tuple $(\diff x,\diff y)$, with $\diff x, \diff y \in [-1,1]$ corresponding to the direction it will try to move (for the actions, $1$ corresponds to the length of one block, as opposed to one unit in state-space). The reward is $-1$ if the agent lands in a normal block, $-10$ if it hits an obstacle or a wall, $1000$ if it lands close to the center of the misleading exit and $100,000$ if it lands close to the center of the actual exit. In case the agent hits a wall or obstacle, any attempt to move in a direction that does not point to the opposite direction will result in no movement and a reward of $-10$. Additionally, there is a timeout of $10,000$ timesteps, after which the agent has its position reset to the starting position without episode termination. To implement this Misleading Maze, we adapt code from GridMap.\footnote{https://github.com/huyaoyu/GridMap}

The FKL and RKL agents are the same as those used in the following benchmark problems in \Cref{sec:main_benchmark}, with pseudocode in \Cref{sec:api-alg}. The FKL is used with weighted importance sampling and RKL is used with the reparametrization trick. The actor and critic are both parametrized as two-layer neural networks of size 128 with ReLU activations, with the actor corresponding to an unimodal Gaussian. The critic learning rate is set to $1e-4$ and the actor learning rate is set to $1e-5$, the optimizer is RMSProp, the batch size is $32$ and we sample $128$ actions to estimate the gradients. Experiments were done using $30$ seeds.

In \Cref{fig:cont_maze_all_temp} we show the mean and standard error of the cumulative number of times the correct exit is reached throughout training for multiple temperatures, where FKL and RKL are plotted separately. \Cref{app:exp_cont} gives more detailed plots of these experiments, where we also show the cumulative number of times the misleading exit is reached and plots for 2M timesteps, instead of the 500k steps we show here. 

We conclude that (1) the FKL seems to be more exploratory than the RKL in this environment and (2) the performance when using FKL seems to be a little more robust to the temperature hyperparameter. We also note that this is the only experiment in the paper where the higher temperatures such as $10$ and $100$ performed better than the lower ones. \Cref{fig:cont_maze_example} shows a sampled trajectory from one of the resulting FKL policies with temperature $100$, trained for 2M timesteps. Although these policies often reach the goal, they are still highly stochastic. A natural next step is to make them more deterministic over time, by using temperature annealing. 

\begin{figure}[tb!]
	\centering
		\begin{tabular}{c c}
			\hspace{0.37\columnwidth} & \includegraphics[width=0.35\columnwidth]{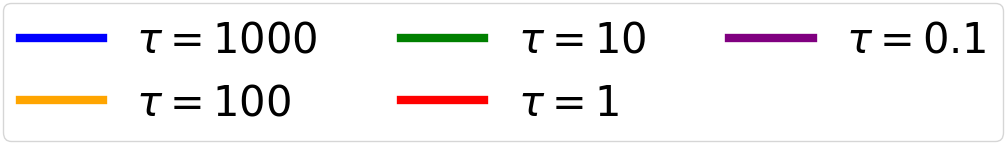}
		\end{tabular}  
		\centering
		\begin{tabular}{c c c}
			\begin{subfigure}{0.3\columnwidth}
			\includegraphics[width=1.0\columnwidth]{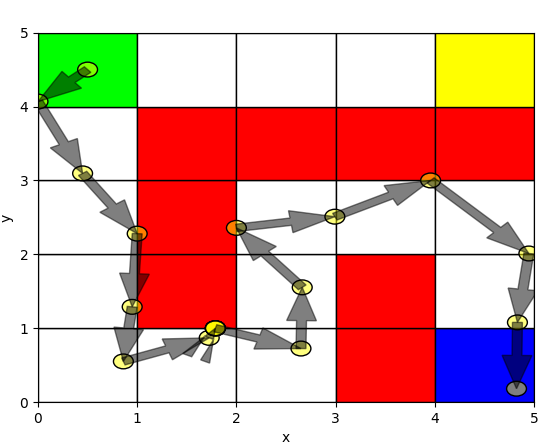}   
			\caption{Misleading Maze}\label{fig:cont_maze_example}
			\end{subfigure}		
			&	
			\begin{subfigure}{0.3\columnwidth}
			\includegraphics[width=1.0\columnwidth]{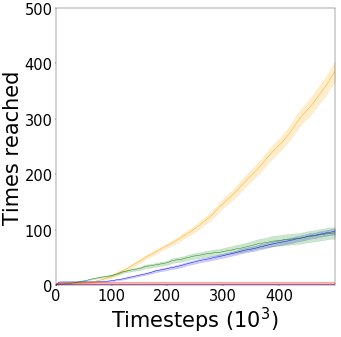}   
			\caption{RKL}
			\end{subfigure}
			&		
			\begin{subfigure}{0.3\columnwidth}
				\centering
				\includegraphics[width=1.0\columnwidth]{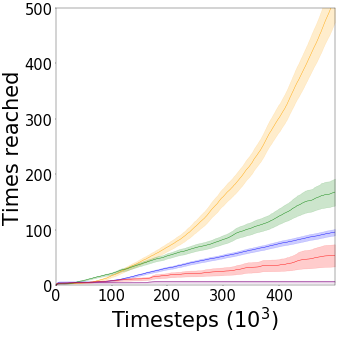}
				\caption{FKL}
			\end{subfigure}  
		\end{tabular}
	\caption{(a) The optimal policy is to go to the blue goal. The yellow goal is misleading, in that it has a positive reward and is easier to get to. Plots in (b) and (c) show the cumulative number of times the correct exit is reached throughout training for multiple temperatures, averaged over 30 runs with the standard errors reported as the shaded regions.}\label{fig:cont_maze_all_temp}
\end{figure}

\subsection{Performance in Benchmark Environments}\label{sec:main_benchmark}

Finally, we compared the FKL and RKL on benchmark continuous and discrete-action environments, using non-linear function approximation. For continuous actions, we experimented on Pendulum, Reacher, Swimmer and HalfCheetah \citep{todorov2012mujoco}. For discrete actions, we experimented on three OpenAI Gym environments \citep{brockman2016openai} and five MinAtar environments \citep{young2019minatar}. We selected policy parameterizations (unimodal Gaussians) and other hyperparameters as is typically done for these problems, as a first investigation into differences under typical conditions.

We found almost no difference between the FKL and RKL for discrete actions and only minimal differences for continuous actions. For completeness and posterity, we still include these systematic results in Appendix \ref{app:main_benchmark}. This contrasts the several differences we found above, in both theory and in more controlled experiments. It is possible that in these environments, the action-values are largely unimodal and so the resulting Boltzmann policy representable by our parameterized policies. The differences between the FKL and RKL are less likely to be pronounced in such settings. A natural next step is to investigate the differences between the two in benchmark problems where it is common to use or need multimodal policies.

\section{Discussion and Conclusion} 

In this work, we investigated the forward KL (FKL) and reverse KL (RKL) divergences to the Boltzmann policy for approximate greedification. To the best of our knowledge, this is the first systematic study into these two natural alternatives, in terms of understanding 1) their theoretical policy improvement properties, 2) their differences in optimization behavior and 3) differences in terms of promoting exploration. Our goal was to highlight that this choice for approximate greedification can have important theoretical and empirical ramifications; we did not advocate for either the FKL or RKL, and found each had useful properties that suit different situations. 

We motivated the importance of understanding this question by explaining that many policy optimization algorithms can actually be well-thought of as approximate policy iteration algorithms, with alternating approximate policy evaluation and greedification steps. We categorized many existing algorithms based on whether they (implicitly) use the FKL and RKL, providing another lens to understand these algorithms. Given this approximate policy iteration perspective, understanding improvements to either step can help us improve our policy optimization algorithms. This work focused on understanding the greedification step, which has been much less investigated than policy evaluation.

Based on our theoretical and empirical results, we can summarize our findings as follows. 

\textbf{Theoretically}, we found reduction in the RKL is both sufficient and necessary to obtain policy improvement (Proposition \ref{prop:avg-reverse-kl}). The FKL, on the other hand, is not guaranteed to induce policy improvement as reliably as the RKL, with a counterexample is a simple one-state MDP (Proposition \ref{sec:fkl-counterexample}). Policy improvement can still occur if a sufficiently high reduction in FKL occurs (Proposition \ref{prop:fklredavg}). We further found that we could upper bound the RKL with the FKL times another term based on the Renyi divergence (Lemma \ref{prop:fkl-rkl-bound}), both suggesting that reducing the FKL on its own may be insufficient without additional conditions that control this term as well. 

\textbf{Empirically, on the microworld and maze experiments}, there were more differences between FKL and RKL in the continuous-action setting, whereas no significant differences were observed in the discrete-action setting. In the continuous-action setting, the FKL tended to have a smoother loss landscape that directed iterates to a global optimum of the entropy-regularized objective. This global optimum, however, was sometimes less optimal with respect to the unregularized objective, especially with higher temperatures, than the optima of the RKL. Moreover, the greater suboptimality of the FKL limit points was correlated with the final FKL policies having higher action variance than the corresponding final policies for the RKL. Further experiments supported the claim that the FKL induces more stochastic policies than the RKL, which is consistent with previously described mean-seeking behaviour of the FKL. The fact that significant differences were only observed in the continuous-action regime suggests the important role of policy parameterization. 

\textbf{Finally, on our benchmark experiments}, there did not seem to be much difference between choosing one divergence or the other for most of the environments, with performance being more heavily dictated by the choice of temperature and learning rate. In fact, plots of both performance and sensitivity to learning rates were near mirror images in most cases. The algorithms had some sensitivity to hyperparameters and the best learning rate-temperature combination was also highly environment dependent.

This work highlights a variety of potential next steps. 
One conclusion from this work is that, though it has been rarely used, the FKL is a promising direction for policy greedification and warrants further investigation. For $\tau \neq 0$, weighted importance sampling allows us to estimate the forward KL objective stochastically, making it practically feasible. Although the FKL performed similarly to the RKL in the benchmark problems, it has properties which can be useful, such as committing less quickly to actions, being more robust to stochastic samples of the gradient of the FKL for a state---which is pertinent as mini-batch estimates will be stochastic---and having a smoother loss surface. 

A natural question from this study is why the differences were the largest for continuous actions in our microworld experiments. One potential reason is the policy parameterization: the Gaussian policy is likely more restrictive than the softmax as it cannot capture multimodal structure. Learning the standard deviation of a Gaussian policy may be another source of instability. In contrast, a softmax policy can represent multiple modes, and does not separate the parameterization of the measure of central tendency (e.g., mean) and the measure of variation (e.g., standard deviation). 

A promising next step is to compare the FKL and RKL with different policy parameterizations for continuous actions.  
Recent work into alternative policy paramaterizations has explored the Beta distribution \citep{chou2017improving}, quantile regression \citep{richter2019learning}, and normalizing flows \citep{ward2019improving}. While the latter two works in particular have focused on the motivation of multimodality for domains that have multiple goals, we believe that the relevance of multimodality for optimization is as important. 


The choice of target distribution in the greedification objective is another interesting question. The Boltzmann distribution over action values is a natural choice for entropy-regularized RL, but one might not want to be tied this framework, especially given sensitivity to the temperature parameter and exploration that is undirected. As yet, there are not too many alternatives. The escort transform has been defined for discrete actions \citep{mei2020escaping}, to avoid the overcommital behavior of the softmax (Boltzmann) policy. For settings where action-values are guaranteed to be positive, the normalized action-values have been used to create a target distribution \citep{ghosh2020operator}. An interesting question is what other alternatives there are to the Boltzmann distribution, especially for continuous actions. 

In addition to the choice of target distribution, there are many other possible choices for a greedification objective. Besides the KL divergences, one may consider the Wasserstein distance, Cramér distance, the JS divergence, and many more. 
Some of our analysis for the KL could actually be used to obtain some of these extensions.
Our derivations for the FKL were based on inequalities connecting f-divergences. 
\citet{sason2016f} compiled a list of inequalities connecting f-divergences, so it is possible to follow similar steps to the ones we followed here to derive bounds for other divergences as well. Since each divergence has situations where it works best, having theoretical guarantees for all of them will make it easier to design algorithms that work in each case.

Though we focused on episodic problems in this work, some of the approaches could be used for the continuing setting with average reward. Recent work has analyzed the regret in continuing RL problems for a policy iteration algorithm, called Politex \citep{abbasi2019politex}. The policy is a Boltzmann distribution over the 
sum of all previous action-value function estimates. For continuous actions, a natural alternative is to consider approximating this distribution with a parameterized policy. Work using differential action-values for the average reward setting could benefit from explicit policy parameterizations that use approximate greedification with respect to those differential action-values. 

Finally, the connection between policy gradients approaches and API could provide new directions to theoretically analyze policy gradient algorithms. A recent global convergence result for policy gradient methods relies on a connection to policy iteration and showing guaranteed policy improvement \citep{bhandari2019global}. These results are similar to earlier results in \citet{scherrer2014local}, and in particular, both works rely on convex policy classes and a closure of the policy class under greedification. 

Many other works have exploited linear action-value parameterizations. \citet{perkins2003convergent} proved that API converges with linear function approximation for the action-values and soft policy improvement. This work, though, requires the best linear action-value approximation in the function class. Other works have assumed instead exact greedification, with approximate action-values, such as Politex \citep{abbasi2019politex} which bounded regret for an API algorithm, with linear action-values learned online. \citet{scherrer2014approximate} provides error propagation analyses of many variants of API; these analyses could be extended to finite-sample guarantees by following \citep{scherrer2015approximateMPI}. Finite-sample analyses for LSPI \citep{lazaric2012finite} and classification-based PI \citep{lazaric2016analysis} also exist, but again where designed for approximate policy evaluation, and exact greedification. 

It is clear there is a rich theoretical literature for API from which to draw, to fill in gaps understanding existing policy optimization algorithms. An important next step is to investigate if this theory can be extended to approximate greedification, and so make steps towards better characterizing many existing policy optimization algorithms. 

\subsection*{Acknowledgements}
We gratefully acknowledge funding from 
NSERC,
the Canada CIFAR AI Chair program, and the Alberta Machine Intelligence Institute (Amii). Special thanks as well to Nicolas Le Roux for comments on an earlier version of this work. 

\bibliography{refs}


\renewcommand{\theHsection}{A\arabic{section}}
\appendix

\section{Categorizing Existing Algorithms by their Greedification Operators}\label{sec:api}

There are many existing policy optimization algorithms. As we motivated in the introduction, many of these can actually be seen as doing API, though they are typically described as policy gradient methods. 
We categorize these methods based on which of these four KLs underlie their policy update.  

It is important to note that these methods are not only characterized by which KL variant they use for updating towards the Boltzmann policy. In fact, in many cases other properties are the critical novelties of the algorithms. For example, as noted below, the Hard RKL underlies TRPO. However, a defining characteristic of TRPO is preventing the policy from changing too much using a KL divergence to the previous policy. Therefore, in addition to using a Hard RKL towards the Boltzmann policy, it also uses a KL divergence to the old policy; these two KL divergences play different roles. 

\subsection{RKL without Entropy Regularization}

Many actor-critic approaches that do not have entropy regularization are implicitly optimizing a hard RKL when performing a policy update.  

        \textbf{Vanilla Actor-Critic} \citep{sutton2018reinforcement} uses the gradient of the Hard RKL for its policy update, in a given state. For the episodic objective $\eta(\pi_\policyparams)$ given in Equation \eqref{eq:policy-obj}, the policy gradient theorem shows that gradient is 
        \begin{align*}
            \nabla_\policyparams \eta(\pi_\policyparams) = \int_{\statespace} d^{\pi_\policyparams}(s) \int_{\actionspace} Q^{\pi_\policyparams}(s, a) \nabla_\policyparams \pi_\policyparams(a \mid s)\, \diff a\, \diff s.    
        \end{align*}
         For each state, the inner update is exactly the Hard RKL. By selecting $d = d^{\pi_\policyparams}$, the Hard RKL averaged across all states exactly equals the policy gradient underlying actor-critic. This weighting is obtained by simply acting on-policy and weighting the update by the discount raised to the power of the step in the episode \citep{thomas2014bias}. In practice, this weighting by the discount is often omitted and the Hard RKL update performed in each state visited under the policy. 
    
        
        \textbf{Trust Region Policy Optimization (TRPO)} \citep{schulman2015trust} has the same hard RKL objective as vanilla actor-critic, but with an \emph{additional constraint} in the projection step that the policy should not change too much after an update. This strategy builds on the earlier Conservative Policy Iteration (CPI) algorithm \citep{kakade2002approximately}, which motivates that the old policy and old action-values can be used in the objective. This objective is sometimes called the linearized objective; with the addition of the constraint, the objective corresponds to 
        \begin{align*}
        \begin{split}
            J(\theta) &= \mathbb{E}_{S \sim d^{\pi_{\theta_{old}}}, A \sim \pi_{\theta}} \left[\Qhat^{\pi_{\theta_{old}}}(S,A)\right] = \mathbb{E}_{S \sim d^{\pi_{\theta_{old}}}, A \sim \pi_{\theta_{old}}}\left[\frac{\pi_\theta(A|S)}{\pi_{\theta_{old}}(A|S)}\Qhat^{\pi_{\theta_{old}}}(S,A)\right] \\
            &\quad\quad\text{subject to } \delta \geq \mathbb{E}_{S \sim d^{\pi_{\theta_{old}}}} [\KL{\pi_{\theta_{old}}}{\pi_\theta}(S)] 
        \end{split}            
        \end{align*}        
        For a given state sampled from $d^{\pi_{\theta_{old}}}$, the inner optimization is precisely a hard RKL to the action-values for the old policy. The objective can be written with actions sampled either according to $\pi_\policyparams$ or according to $\pi_{\theta_{old}}$ with an importance sampling ratio.
        
        Note, though, that the actual TRPO algorithm uses an approximation that results in a natural policy gradient update, rather than a gradient of the above objective. So, though it is motivated by optimizing the above (Equation (14) in their work), it actually solves \citep[Equation 17]{schulman2015trust}. This has further been re-interpreted as a mirror descent update and used to provide convergence of TRPO \citep{neu2017unified,liu2019neural,shani2019adaptive}.
            
            
        \textbf{Proximal Policy Optimization (PPO)} \citep{schulman2017proximal} the baseline they compare against uses the same objective as TRPO, but uses a KL penalty instead of a constraint. The policy objective with the KL penalty is
        \begin{align*}
            J(\theta) = \mathbb{E}_{S \sim d^{\pi_{\theta_{old}}}, A \sim \pi_{\theta_{old}}}\left[\frac{\pi_\theta(A|S)}{\pi_{\theta_{old}}(A|S)}\Qhat^{\pi_{old}}(S,A)\right] - \beta \mathbb{E}_{s \sim d^{\pi_{\theta_{old}}}} \left[ \KL{\pi_{\theta_{old}}}{\pi_\theta}(S) \right]
        \end{align*}
        The greedification component of this objective---the first term---can again been seen as using $d = d^{\pi_{old}}$ with a Hard RKL to action-values $\Qhat^{\pi_{old}}$. The second term simply modifies how the Hard RKL is optimized, because the convergence point is unchanged: when $\theta_{old}$ is optimal, setting $\theta = \theta^*$ maximizes the first term while keeping the second term $0$. 
        
        The actual objective used in PPO clips hyperparameter $\epsilon \geq 0$ and is written as
        \begin{align*}
            J(\theta) = \mathbb{E}_{\pi_{\theta_{old}}}\left[\min\left(\frac{\pi_\theta(a|s)}{\pi_{\theta_{old}}(a|s)}\hat{A}^\pi(s,a), \mathrm{clip}\left(\frac{\pi_\theta(a|s)}{\pi_{\theta_{old}}(a|s)}, 1 - \epsilon, 1 + \epsilon\right)\hat{A}^\pi(s,a)\right)\right]
        \end{align*}
        The clipped objective forces the probability ratio $\frac{\pi_\theta(a|s)}{\pi_{\theta_{old}}(a|s)}$ to stay within $[1-\epsilon, 1+\epsilon]$.

		\textbf{Other trust region policy optimization methods} can be also understood as minimizing the Hard RKL. Maximum a posteriori policy optimization \citep{abdolmaleki2018maximum} resembles TRPO, but uses a target policy obtained by replacing the entropy term in Equation~\eqref{eq:target policy definition} with an FKL between $p$ and an old policy. Mirror descent policy iteration, which is mirror descent modified policy iteration \citep{geist2019theory} with $m=\infty$, similarly uses a target policy obtained by replacing the entropy term with a Bregman divergence. Another policy optimization algorithm based on the trust region method is relative entropy policy search \citep{peters2010relative}. 

        \textbf{Deep Deterministic Policy Gradient (DDPG)} \citep{silver2014deterministic,lillicrap2015continuous} can be viewed as a ``degenerate'' Hard RKL with a deterministic policy, taking care to differentiate through the action. 
        
        
            
    \subsection{RKL with Entropy Regularization}
        The RKL has been used for an actor-critic algorithm and a continuous action Q-learning algorithm.
        
        \textbf{Soft Actor-Critic (SAC)} \citep{haarnoja2018soft} minimizes the RKL, but is written slightly differently because it  subsumes the temperature inside the reward and scales the reward to control the temperature. 
        The policy objective is
        \begin{align*}
        \begin{split}
        J(\theta) &= \mathbb{E}_{S \sim \mathcal{D}} \left[ \mathrm{KL}\left(\pi_\theta(\cdot|S) \parallel \frac{\exp(Q^{\pi_{old}}(S, \cdot))}{Z(S)}\right)\right]
        \end{split}
        \end{align*}
        where the states $s$ are sampled from a dataset given by $\mathcal{D}$.
        We note that {A3C} \citep{mnih2016asynchronous} updates the policy with the same objective, but does not learn the soft value functions.     
        
        
        
        
        
        \textbf{Soft Q-learning (SQL)} \citep{haarnoja2017reinforcement} introduces a soft Bellman optimality update, that iterates towards a soft optimal action-value $Q^*$ using a soft maximization. To avoid sampling the Boltzmann policy, which is expensive, they proposed using Stein variational gradient descent. This approach requires introducing an approximate sampling network, which can be alternatively seen as a policy. The resulting objective corresponds exactly to an RKL, with the Boltzmann policy defined on $\Qhat^*$
        \begin{align*}
            J(\theta) = \mathbb{E}_{S \sim D}\left[ \mathrm{KL}\left( \pi_\theta(\cdot|S) || \frac{\exp(Q^*(S, \cdot)\tau^{-1})}{Z(S)} \right) \right]
        \end{align*}

        Two other approaches are similar to SQL, but for discrete actions. 
    	\citet{asadi2017analternative} propose generalized value iteration with the mellomax operator, where the maximum-entropy mellowmax policy is a Boltzmann policy with state-dependent temperature. Conservative value iteration \citep{kozuno2019theoretical} can be seen as a variant of SQL in which a trust region method is used, formally proven in \citep{vieillard2020leverage}. 

        

\subsection{FKL without Entropy Regularization}
The FKL without entropy regularization and Boltzmann target policy has not previously been explored. However, the FKL to other target distributions has been considered. 

    \textbf{Deep Conservative Policy Iteration (DCPI)} \citep{vieillard2019deep}
     learns both an action value function $Q$ and a policy $\pi_\policyparams$, with target networks $Q^-$ and $\pi^-$ for each. To update the policy, DCPI minimizes an FKL loss between $\pi_\policyparams$ and a regularized greedification of $Q$. In other words, DCPI uses a forward KL, but to different target distribution than the Boltzmann distribution. 
    \begin{align*}
        J(\policyparams) := \Ex_{S \sim \mathcal{D}}[\KL{((1 - \alpha) \pi^- + \alpha \mathcal{G}(Q)}{\pi_\policyparams}(S)]
    \end{align*}
    for $\alpha \in [0, 1]$ and $\mathcal{G}(Q(S,\cdot))$ the greedy policy w.r.t. $Q$ for a given state $S$. 

\textbf{Policy Greedification as Classification} uses multi-class classification, where actions that maximize $Q(s,a)$ at each state are labeled positive, and the policy is updated to predict that greedy action for every state \citep{lagoudakis2003reinforcement,lazaric2010analysis,farahmand2015classificationbased}. This approach is related to using the Hard FKL, because the FKL corresponds to the cross-entropy loss which can be used for classification. Other classification methods, however, like SVMs, are not directly related.  

\subsection{FKL with Entropy Regularization}

    \textbf{Under-appreciated Reward EXploration (UREX)} \citep{nachum2016improving} 
    optimizes a mixture of forward and reverse KLs. Their reverse KL is the usual vanilla actor-critic objective, while the forward KL is given by $\KL{\pi^*}{\pi_\theta}$, where they approximate $\pi^* \propto \exp(\tau^{-1} G - \log \pi_\policyparams)$ with $G$ being the return received at the end of the episode.
    Subsequent work in \citet{agarwal2019learning} also employs both a forward and reverse KL, where the forward KL is initially used to collect diverse trajectories, and the reverse KL is used to learn a robust policy, which performs well in sparse and under-specified reward settings.

    \textbf{Exploratory Conservative Policy Optimization (ECPO)}
     \citep{mei2019principled} splits policy optimization into a \textit{project} and a \textit{lift} step. The project step minimizes the forward KL divergence to a target policy $\overline{\pi}_{\tau, \tau'}$ that is the optimal policy under the entropy-regularized objective, with a KL penalty to the old policy.  
    \begin{align*}
    \begin{split}
        &\mathrm{Project}: \argmin_{\pi_\theta} \KL{\overline{\pi}_{\tau, \tau'} }{\pi_\theta}\\
        &\mathrm{Lift}: \overline{\pi}_{\tau, \tau'} = \argmax_{\pi} \Ex_{S_0}[V^{\pi_{old}}(S_0)] - \tau \KL{\pi}{\pi_{old}}  + \tau' \entropy(\pi) 
    \end{split}        
    \end{align*}
    
    There have also been relevant works in supervised learning using entropy regularization and forward KL. In \textbf{Reward Augmented Maximum Likelihood (RAML)} \citep{norouzi2016reward}, instead of optimizing over the traditional maximum likelihood framework (Hard FKL), a target reward distribution is defined, and the model distribution minimizes the forward KL to that target distribution. 
 
 It is interesting to note that the FKL has been used in other area of reinforcement learning.    
In Inverse Reinforcement Learning and Imitation Learning, \cite{ghasemipour2020divergence} have shown that, under the cost regularized framework proposed by \cite{ho2016generative}, most methods can be categorized as reducing some divergence. Notably among them is behavior cloning, which corresponds to a FKL reduction between policies, whereas AIRL \citep{fu2018learning}, one of the state-of-the-art methods, reduces the RKL between occupancy measures. The paper then proposes a variant of AIRL based on reducing the FKL, obtaining competitive results.    

\section{Proofs}\label{app:proofs}

In this section we provide all the proofs. We give results organized according to the same subsections as in the main body. 

\subsection{Proofs for Section \ref{sec_rkl_improvement}: Policy Improvement with RKL}

\spd*
\begin{proof}
When we write $\Etraj$, we mean the expectation over the trajectory distribution induced by $\rho_0$ and $\pinew$.
\begin{align*}
    \frac{1}{1 - \gamma}&\Ex_{d^{\pinew},\pinew}[A^{\piold}_\tau(S, A)] = \Etraj\left[\sum_{t = 0}^\infty \gamma^t A_{\tau}^{\piold}(S_t, A_t)\right]\\
    &= \Etraj\left[\sum_{t = 0}^\infty \gamma^t (Q^\piold_\tau(S_t, A_t) - \tau \log \piold(\cdot \mid S_t) - V^{\piold}_\tau(S_t))\right]
\end{align*}
where the first equality follows from the definition of the visitation distribution and the second from the definition of the soft advantage. 
We can simplify the term inside the expectation as follows.
\begin{align*}
\sum_{t = 0}^\infty \gamma^t &(Q^\piold_\tau(S_t, A_t) - \tau \log \piold(\cdot \mid S_t) - V^{\piold}_\tau(S_t))\\
        &= \sum_{t = 0}^\infty \gamma^t (r(S_t, A_t) + \gamma V_\tau^{\piold}(S_{t + 1}) - \tau \log \piold(\cdot \mid S_t) - V^{\piold}_\tau(S_t))\\
        &= \Big( \sum_{t = 0}^\infty \gamma^t (r(S_t, A_t) - \tau \log \piold(\cdot \mid S_t)) \Big) - V_\tau^{\piold}(S_0),
\end{align*}
where the second line follows from expanding $Q^\piold_\tau$ and the second from the telescoping series $\gamma V^\piold_\tau(S_{t + 1}) - V^\piold_\tau(S_t)$.  
Plugging this back into the expectation, and using $\Etraj\left[V_\tau^{\piold}(S_0)\right] = \eta_\tau(\piold)$ we get
\begin{align*}
\Etraj&\left[\sum_{t = 0}^\infty \gamma^t A_{\tau}^{\piold}(S_t, A_t)\right] = -\eta_\tau(\piold) + \Etraj\left[\sum_{t = 0}^\infty \gamma^t (r(S_t, A_t) - \tau \log  \piold(\cdot \mid S_t)) \right]\\
        &= -\eta_\tau(\piold) + \eta_\tau(\pinew) + \Etraj\left[\sum_{t = 0}^\infty \gamma^t \tau (\log \pinew(\cdot \mid S_t) -  \log  \piold( \cdot \mid S_t)) \right] \\
        &= -\eta_\tau(\piold) + \eta_\tau(\pinew)  + \frac{\tau}{1 - \gamma} \Ex_{d^{\pinew}}[\KL{\pinew}{\piold}(S)],
\end{align*}
where the second equality is obtained by adding and subtracting $\tau \log \pinew(\cdot \mid S_t)$.
\end{proof}

\rklimprove*
\begin{proof}
We start by writing the RHS of \Cref{lemma:soft-performance-difference}.
\begin{gather*}
\frac{1}{1 - \gamma}\Ex_{d^{\pinew}, \pinew}[A_\tau^{\piold}(S, A)] - \frac{\tau}{1 - \gamma} \Ex_{d^\pinew}[\KL{\pinew}{\piold}(S)] =\\
\begin{split}
\frac{\tau}{1 - \gamma}\EdBiggAct{\frac{Q_\tau^{\piold}(S, A)}{\tau} &-\cancel{\log(\piold(A|S))} - \frac{V_\tau^{\piold}(S)}{\tau}} \\ &- \frac{\tau}{1 - \gamma} \Ex_{d^\pinew,\pinew}[\log(\pinew(A|S)) - \cancel{\log(\piold(A|S))}] =\end{split} \\
\begin{split}
\frac{\tau}{1 - \gamma}\Bigg( \EdBiggAct{ \log(\mathrm{e}^{\frac{Q_\tau^{\piold}(S, A)}{\tau}})} & - \EdBiggActOld{\frac{Q_\tau^{\piold}(S,A)}{\tau} - \log(\piold(A|S))} \Bigg) \\ &- \frac{\tau}{1 - \gamma} \Ex_{d^\pinew,\pinew}[\log(\pinew(A|S)) ] =\\
\end{split} \\
\begin{split}
	\frac{\tau}{1 - \gamma} \EdBiggAct{ \log(\mathrm{e}^{\frac{Q_\tau^{\piold}(S, A)}{\tau}}) &- \log(\pinew(A|S))}  \\&- \frac{\tau}{1 - \gamma} \EdBiggActOld{\log(\mathrm{e}^{\frac{Q_\tau^{\piold}(S, A)}{\tau}}) - \log(\piold(A|S))} =\\
\end{split} \\
\frac{\tau}{1 - \gamma} \Ed{\Delta \RKL{\piold}{\pinew}(S)}.
\end{gather*}
The last equality follows by adding and subtracting $\Ex_{d^{\pinew}} (\log(Z(S)))$ and rearranging. Plugging that in the equation from \Cref{lemma:soft-performance-difference} we get:
\begin{align*}
	\eta_\tau(\pinew) - \eta_\tau(\piold) = \frac{\tau}{1 - \gamma} \Ed{\Delta \RKL{\piold}{\pinew}(S)},
\end{align*}
which can only be nonnegative if $\Ed{\Delta \RKL{\piold}{\pinew}(S)} \geq 0$.
\end{proof}

\subsubsection{Proofs for Section \ref{sec_rkl_imp_action}: Extensions to Approximate Action-Values}

\approxspd*
\begin{proof}
We focus on a particular state $s$ and omit function arguments when they are clear from context.
\begin{align*}
    &\tau(\KL{\piold}{\mathcal{B}_\tau \hat Q} - \KL{\pinew}{\mathcal{B}_\tau \hat Q} )\\
    &\quad= \tau[-H(\piold) + H(\pinew) - \sum_a \piold(a)( \hat Q(s, a) / \tau - \log Z(s)) + \sum_a \pinew(a)( \hat Q(s, a) / \tau - \log Z(s))]\\
    &\quad = -\tau H(\piold) + \tau H(\pinew) - \sum_a \piold(a)\hat Q(s, a) + \sum_a \pinew(a)\hat Q(s, a) \\
    &\eqcomment{policies sum to one and $Z(s)$ does not depend on $a$}\\
    &\quad = -\tau H(\piold) + \tau H(\pinew) - \sum_a \piold(a)(Q^\piold_\tau(s, a) + \epsilon(s, a) ) \\
    &\quad\quad\quad+ \sum_a \pinew(a)(Q^\piold_\tau(s, a) + \epsilon(s, a) + \tau \log \piold(a) - \tau \log \piold(a)) \\
    &\eqcomment{expanding $\hat Q$}\\
    &\quad= -\tau H(\piold) -\tau \KL{\pinew}{\piold} - \sum_a \piold(a)(Q^\piold_\tau(s, a) + \epsilon(s, a) ) \\
    &\quad\quad+ \sum_a \pinew(a)(Q^\piold_\tau(s, a) + \epsilon(s, a) - \tau \log \piold(a)) \\
    &\eqcomment{folding term into KL}\\
    &\quad= -\tau H(\piold) - \tau \KL{\pinew}{\piold} - V_\tau(s) - \sum_a \piold(a)(\tau \log \piold(a) + \epsilon(s, a) ) \\
    &\quad\quad+ \sum_a \pinew(a)(Q^\piold_\tau(s, a) + \epsilon(s, a) - \tau \log \piold(a)) \\
    &\eqcomment{using average of action-value}\\
    &\quad= -\tau \KL{\pinew}{\piold}  - \sum_a \piold(a)\epsilon(s, a)  + \sum_a \pinew(a)(A^\piold_\tau(s, a) + \epsilon(s, a)) \\
    &\eqcomment{cancelling entropy and using definition of soft advantage}\\
\end{align*}
Upon averaging both sides over $d^\pinew$, dividing by $1 - \gamma$, and using the performance difference lemma, we obtain the result.
\begin{align*}
    &\frac{\tau}{1 - \gamma}\Ex_{\pinew}[\KL{\piold}{\mathcal{B}_\tau \hat Q} - \KL{\pinew}{\mathcal{B}_\tau \hat Q} ]\\
    &\quad = \eta(\pinew) - \eta(\piold) + \Ex_{d^\pinew}[\Ex_{\pinew} \epsilon(s, a) - \Ex_{\piold} \epsilon(s, a)].
\end{align*}
\end{proof}

\approxrkl*
\begin{proof}
The proof is simply rearrangement of \Cref{lemma:approximate-soft-performance-difference}. 
\end{proof}

\subsubsection{Proofs for Section \ref{sec_extension_dpiold}: Extensions to Weighting under $d^\piold$}

\dpioldpdl*
\begin{proof}

Reversing the role of $\pinew$ and $\piold$ in \Cref{lemma:soft-performance-difference},
\begin{align*}
    \eta_\tau(\pinew) - \eta_\tau(\piold)
    =
    \frac{1}{1 - \gamma} \Ex_{d^\piold, \piold} \bigg[
        \tau \KL{\piold}{\pinew} (S) - \advnew (S, A)
    \bigg].
\end{align*}
Let $\Delta$ be $\advnew - \advold$. Since $\sum_{a \in \cA} \pi (a|s) A_\tau^\pi (s, a) = 0$, we have that
\begin{align*}
    \sum_{a \in \cA} \piold (a|s) \advnew (s, a)
    &=
    \sum_{a \in \cA} \piold (a|s) \Delta (s, a)
    \\
    &=
    \sum_{a \in \cA} \left( \piold (a|s) - \pinew (a|s) \right) \Delta (s, a)
    + \sum_{a \in \cA} \pinew (a|s) \Delta (s, a)
    \\
    &=
    \sum_{a \in \cA} \left( \piold (a|s) - \pinew (a|s) \right) \Delta (s, a)
    - \sum_{a \in \cA} \pinew (a|s) \advold (s, a).
\end{align*}
Furthermore, note that
\begin{align*}
    \Delta (s, a)
    &=
    \underbrace{
        \qnew (s, a) - \vnew (s) - \qold (s, a) + \vold (s)
    }_{
        := \Delta' (s, a)
    } + \tau \log \frac{\piold (a|s)}{\pinew (a|s)}.
\end{align*}
Therefore, letting $V_{\tau, max} := \frac{r_{max} + \tau \log A}{1-\gamma}$
\begin{align*}
    &\sum_{a \in \cA} \piold (a|s) \advnew (s, a)
    \\
    &=
    \sum_{a \in \cA} \left( \piold (a|s) - \pinew (a|s) \right) \Delta' (s, a)
    + \tau [ \KL{\piold}{\pinew} (s) + \KL{\pinew}{\piold} (s) ]
    - \sum_{a \in \cA} \pinew (a|s) \advold (s, a)
    \\
    &\leq
    4 V_{\tau, max} \sum_{a \in \cA} \left| \piold (a|s) - \pinew (a|s) \right|
    + \tau [ \KL{\piold}{\pinew} (s) + \KL{\pinew}{\piold} (s) ]
    - \sum_{a \in \cA} \pinew (a|s) \advold (s, a)
    \\
    &\leq
    4 V_{\tau, max} \sqrt{2 \KL{\pinew}{\piold} (s)}
    + \tau [ \KL{\piold}{\pinew} (s) + \KL{\pinew}{\piold} (s) ]
    - \sum_{a \in \cA} \pinew (a|s) \advold (s, a).
\end{align*}

Using this result,
\begin{align*}
    &\eta_\tau(\pinew) - \eta_\tau(\piold)\\
    &\quad\quad\geq 
    \frac{1}{1 - \gamma} \Ex_{d^\piold} \bigg[
        \sum_{a \in \cA} \pinew (a|S) \advold (S, a)
        - \tau \KL{\pinew}{\piold} (S)
        - 4 V_{\tau, max} \sqrt{2 \KL{\pinew}{\piold} (S)}
    \bigg]
    \\
    &\quad\quad\geq 
    \frac{1}{1 - \gamma} \Ex_{d^\piold} \bigg[
        \sum_{a \in \cA} \pinew (a|S) \advold (S, a)
        - \tau \KL{\pinew}{\piold} (S)
        - 4 V_{\tau, max} \sqrt{2\alpha}
    \bigg].
\end{align*}
The final result follows by noting that $\Delta \RKL{\piold}{\pinew} = \sum_{a \in \cA} \pinew (a|S) \advold (S, a) - \tau \KL{\pinew}{\piold} (S)$, as in the proof of \Cref{prop:avg-reverse-kl}.
\end{proof}

\approxdpiold*
\begin{proof}
From the proof of \Cref{lemma:approximate-soft-performance-difference}, we have
\begin{align*}
     \sum_{a \in \cA} \pinew (a|S) \advold (S, a)
        - \tau \KL{\pinew}{\piold} (S) = \tau(\KL{\piold}{\mathcal{B}_\tau \hat Q} - \KL{\pinew}{\mathcal{B}_\tau \hat Q} ) + \sum_a \epsilon(s, a)(\piold(a) - \pinew(a)).
\end{align*}
Using this result in \Cref{prop:dpiold-extension}, 
\begin{align*}
    &\eta_\tau(\pinew) - \eta_\tau(\piold)\\
      &\quad\quad\geq 
    \frac{1}{1 - \gamma} \Ex_{d^\piold} \bigg[
        \sum_{a \in \cA} \pinew (a|S) \advold (S, a)
        - \tau \KL{\pinew}{\piold} (S)
        - 4 V_{\tau, max} \sqrt{2\alpha}\bigg]\\
        &\quad\quad = \frac{1}{1 - \gamma} \Ex_{d^\piold} \bigg[
        \tau(\KL{\piold}{\mathcal{B}_\tau \hat Q} - \KL{\pinew}{\mathcal{B}_\tau \hat Q} ) + \sum_a \epsilon(s, a)(\piold(a) - \pinew(a))
        - 4 V_{\tau, max} \sqrt{2\alpha}\bigg]\\
        &\quad\quad = \frac{1}{1 - \gamma} \Ex_{d^\piold} \bigg[
        \tau  \hat \Delta\RKL{\piold}{\pinew}(S) + \sum_a \epsilon(s, a)(\piold(a) - \pinew(a))
        - 4 V_{\tau, max} \sqrt{2\alpha}\bigg].
\end{align*}
\end{proof}

\subsection{Proofs for Section \ref{sec_fkl_improvement}: Policy Improvement with FKL}

\fklcounter*
\begin{proof}
	Consider the environment with a single state $s$ and two actions: $a_1$ and $a_2$. Regardless of the action chosen, the agent always transition to $s$. We will omit dependency on the state in the following notation. The rewards are defined as
	\begin{align*}
		r(a_1) = -1 \quad, \quad r(a_2) = 1.
	\end{align*}
	Take $\piold$ and $\pinew$ as follows.
	\begin{alignat*}{3}
	&\piold(a_1) = \epsilon_1 \quad &&, \quad &&\pinew(a_1) = 1 - \epsilon_2, \\
	&\piold(a_2) = 1-\epsilon_1 \quad &&, \quad &&\pinew(a_2) = \epsilon_2	.
	\end{alignat*}
	We will prove the result by making $\epsilon_1$ arbitrarily small, which we will show forces $\FKL{\piold}{\BQt^{\piold}}$ to $\infty$ while keeping $\epsilon_2$ fixed, causing $\FKL{\pinew}{\BQt^{\piold}}$ to be finite. Note that
	\begin{align*}
		\FKL{\piold}{\BQt^{\piold} } &= \sum_{i=1}^2 \BQt^{\piold}(a_i) \log \left( \frac{\BQt^{\piold}(a_i)}{\piold(a_i)} \right) \\
		 &= -\entropy(\BQt^{\piold}) - \sum_{i=1}^2 \BQt^{\piold}(a_i) \log( \piold(a_i) ), \\
		 \lim_{\epsilon_1 \to 0} \FKL{\piold}{\BQt^{\piold} } & = \underbrace{-\lim_{\epsilon_1 \to 0} \entropy(\BQt^{\piold})}_{\geq -\log 2 \text{ and } \leq 0} - \lim_{\epsilon_1 \to 0} \BQt^{\piold}(a_1) \log(\epsilon_1) + 0.
	\end{align*}
	To calculate the limit of the middle summand, we note that if $ \lim_{\epsilon_1 \to 0} \BQt^{\piold}(a_1) > 0$, the middle summand will go to infinity, since $ \lim_{\epsilon_1 \to 0} -\log(\epsilon_1) = \infty$. We can verify that this is indeed the case.
	\begin{alignat}{3}
		& \lim_{\epsilon_1 \to 0} \Qt^\piold(a_1) = -1 + \frac{\gamma}{1 - \gamma} \qquad &&; \qquad  && \lim_{\epsilon_1 \to 0} \Qt^\piold(a_2) = \frac{1}{1 - \gamma} \nonumber \\		
		& \lim_{\epsilon_1 \to 0} \BQt^\piold(a_1) = \frac{\mathrm{e}^{\left(-1 + \frac{\gamma}{1 - \gamma} \right)\frac{1}{\tau}}}{ Z } \qquad &&; \qquad  && \lim_{\epsilon_1 \to 0} \BQt^\piold(a_2) = \frac{\mathrm{e}^{\left( \frac{1}{1 - \gamma} \right)\frac{1}{\tau}}}{ Z }, \label{eq:counter_ex_bq}			
	\end{alignat}
	where $Z := \mathrm{e}^{\left(-1 + \frac{\gamma}{1 - \gamma} \right)\frac{1}{\tau}} + \mathrm{e}^{\left(\frac{1}{1 - \gamma} \right)\frac{1}{\tau}}$. Since, for fixed $\gamma$ and $\tau$, the quantities in \Cref{eq:counter_ex_bq} are fixed, we have that $ \lim_{\epsilon_1 \to 0} -\BQt^{\piold}(a_1) \log(\epsilon_1) = \infty$, causing $\lim_{\epsilon_1 \to 0} \FKL{\piold}{\BQt^{\piold} } = \infty$. Moreover, $\FKL{\pinew}{\BQt^{\piold}}$ has a similar form of the above FKL, but, since $\epsilon_2$ is assumed to be fixed, this quantity will be finite. The point is that for any $\epsilon_2$, we can find $\epsilon_1$ such that $\FKL{\piold}{\BQt^{\piold}} > \FKL{\pinew}{\BQt^{\piold} }$. 
	
	It remains to be seen that we can have $\epsilon_2$ that guarantees $\Vpitau{\pinew} < \Vpitau{\piold}$. We write
	\begin{align*}
		\Vpitau{\pinew} &= \Ex\left[ \sum_{t=0}^\infty \gamma^t (R + \tau \entropy(\pinew)) \right] 
		= \Ex\left[ \sum_{t=0}^\infty \gamma^t R \right] + \frac{\tau\entropy(\pinew)}{1-\gamma} 
		\\
		&= \left(\sum_{t=0}^\infty \gamma^t \Ex[R]\right) + \frac{\tau\entropy(\pinew)}{1-\gamma} \\
		&= \frac{2 \epsilon_2 - 1}{1-\gamma} + \frac{\tau\entropy(\pinew)}{1-\gamma}. 		
	\end{align*}
	Additionally, we know that $\lim_{\epsilon_1 \to 0} \Vpitau{\piold} = \frac{1}{1-\gamma}$. If we can find $\epsilon_2$ such that $\Vpitau{\pinew} < \frac{1}{1 - \gamma}$, then we can find $\epsilon_1$ such that simultaneously $\Vpitau{\pinew} < \Vpitau{\piold}$ and $\FKL{\piold}{\BQt^{\piold} ) > \FKL (\pinew, \BQt^{\piold} }$. 
	$\Vpitau{\pinew} < \frac{1}{1 - \gamma}$ will hold if
	\begin{gather}
		\frac{1}{1-\gamma} \left(2 \epsilon_2 - 1 + \tau\entropy(\pinew)\right) < \frac{1}{1-\gamma} \nonumber \\
		2 \epsilon_2 - 1 + \tau( -(1-\epsilon_2)\log(1-\epsilon_2) - \epsilon_2 \log(\epsilon_2) ) < 1. \label{eq:fkl_counter_final_ineq}
		\end{gather}
We use $f(\epsilon_2)$ to denote the LHS of \cref{eq:fkl_counter_final_ineq}.
We have
\begin{equation*}
\lim_{\epsilon_2 \to 0} f(\epsilon_2) = -1,
\end{equation*}
which is less than the RHS of \Cref{eq:fkl_counter_final_ineq}. Formally, for any $\epsilon > 0$, we can find $\delta > 0$ such that if $|\epsilon_2| < \delta$, $|f(\epsilon_2) - 1| < \epsilon$. Any $\epsilon < 2$ will suffice: we can conclude that, for any fixed $\tau \geq 0$, there is some $\epsilon_2 \in (0, 1)$ satisfying \cref{eq:fkl_counter_final_ineq}.

It is possible to extend this to an MDP with multiple states. Consider a case where $\cS = \{s_1, s_2, s_3, \ldots, s_N \}$, and at each $s_n \in \cS$ except $s_N$, a state transition to $s_{n+1}$ occurs. At state $s_N$, the transition is back to itself. Setting the reward, $\piold$, and $\pinew$ similarly to the above environment, 
we can show the same result with multiple states.
\end{proof}

\subsubsection{Proofs for Section \ref{sec:direct-fkl-reduction}: Policy Improvement under Sufficient FKL Reduction}\label{app_direct-fkl-reduction}

\begin{lemma}\label{lemma:intermediate_ineq}
Let the action space $\actionspace$ be finite. For any state $s$ and any vector $\hat Q : \statespace \times \actionspace \to \R$, if $\FKL{\bpinew}{\mathcal{B}_\tau \hat Q}(s) \leq \FKL{\bpiold}{\mathcal{B}_\tau \hat Q}(s) $ then 
\begin{align*}
    \RKL{\bpinew}{\mathcal{B}_\tau \hat Q}(s) \leq \frac{\norm{\hat Q}_\infty}{\tau} \sqrt{2 \FKL{\bpinew}{\mathcal{B}_\tau \hat Q}(s)} - \mathcal{B}_\tau \hat Q(s, \cdot)^\top \log(\bpiold(\cdot \mid s))
\end{align*}
\end{lemma}
\begin{proof}
	We will suppress the argument $s$. Start with the following.
	\begin{align}
		\label{eq:intermediate_ineq}
		0 & \leq \FKL{\bpi_{new}}{\BQ} \nonumber \\
		& \leq \BQ^\top \log \left( \frac{\BQ}{\bpinew} \right) \nonumber \\
		&= \BQ^\top \left( \frac{\bv}{\tau} - \log(\bpinew) \right) - \log(Z) \nonumber \\
		&\quad \triangleright \text{Expanding the inner $\BQ$} \nonumber \\
		0 &\leq \BQ^\top \left( \bv - \tau\log(\bpinew) \right) - \tau \log(Z) \nonumber \\
		&\quad \triangleright \text{Multiplying both sides by $\tau$} \nonumber \\
		& = \bpinew^\top (\bv - \tau \log \bpinew) + \bv^\top(\BQ - \bpinew) \, + \nonumber \\ 
		&\quad \tau \log \bpinew^\top (\bpinew - \BQ) - \tau \log(Z) \nonumber \\
		&\quad \triangleright \text{Adding and subtracting $\bpinew^\top\bv$ and $\tau \bpinew^\top \log \bpinew$} \nonumber \\
		&= -\tau \RKL{\bpinew}{\BQ} + \bv^\top(\BQ-\bpinew) \, + \\
		&\quad \tau (\bpinew^\top \log(\bpinew) - \BQ^\top \log (\bpinew)) \nonumber \\
		&\quad \triangleright \text{Absorbing the partition function.} \nonumber
	\end{align}
	We analyze some of the summands from \Cref{eq:intermediate_ineq} in turn. By Hölder's inequality,
	\begin{align*}
		\bv^\top(\BQ-\bpinew) &\leq \norm{\bv}_\infty \sum_i |(\BQ_i - \pinew_i)|.
	\end{align*}
	By Pinsker's inequality \citep{pinsker1964information},
	$$\sum_i |(\BQ_i - \pinew_i)| \leq \sqrt{2 \FKL{\bpinew}{\BQ}}.$$
	Therefore,
	\begin{align*}
		\bv^\top(\BQ-\bpinew) \leq \norm{\bv}_\infty \sqrt{2 \FKL{\bpinew}{\BQ}}.
	\end{align*}
	The other summand from \Cref{eq:intermediate_ineq} can be written as
	\begin{align*}
		\tau (\bpinew^\top \log(\bpinew) - \BQ^\top \log (\bpinew)) \leq \tau (0 - \BQ^\top \log (\bpiold)),
	\end{align*}
	where we used the fact that the negative entropy of $\bpinew$ is less than or equal to zero and that, since the underlying assumption is that we have non-negative FKL reduction, we also have
	$-\BQ^\top \log (\bpinew) \leq -\BQ^\top \log (\bpiold)$ (by the FKL definition). 
	
	We substitute these upper bounds into \Cref{eq:intermediate_ineq}.
	\begin{align*}
		0 &\leq -\tau \RKL{\bpinew}{\BQ} + \bv^\top(\BQ-\bpinew) \, + \\
		&\quad \tau (\bpinew^\top \log(\bpinew) - \BQ^\top \log (\bpinew)) \\
		&\leq -\tau \RKL{\bpinew}{\BQ} + \norm{\bv}_\infty \sqrt{2 \FKL{\bpinew}{\BQ}} \, + \\
		&\quad \tau (- \BQ^\top \log (\bpiold)),  \\	
	\end{align*}
	which then implies
	\begin{align}
		\label{eq:rkl_ub}
		\RKL{\bpinew}{\BQ} \leq \frac{\norm{\bv}_\infty}{\tau} \sqrt{2 \FKL{\bpinew}{\BQ}} - BQ^\top \log(\bpiold).
	\end{align}
\end{proof}

We start by showing this result in bandits, to make it simpler to show across states. 
\begin{definition}[Entropy-Regularized Bandits]
	\label{def:bandits_setting}
    We denote $\bpi$ as a vector in $\R^{|\cA|}$ that satisfies $\bpi \geq \mathbf{0}$ and $\bpi^{\top} \mathbf{1} = 1$, with $\mathbf{1}$ and $\mathbf{0}$ being vectors containing respectively only entries equal to $1$ and $0$. Further, we consider a single state and denote the corresponding action-values as a vector $\mathbf{q} \in \R^{|\cA|}$. The objective is then defined as $\eta_\tau(\bpi) = \bpi^\top (\mathbf{q} - \tau \log(\bpi))$. Moreover, 
        $\BQ = \mathcal{B}_\tau \mathbf{q} =\frac{\exp\left(\frac{\bv}{\tau}\right)}{Z}$
    with $Z = \exp\left(  \frac{\bv}{\tau}\right)^{\top} \mathbf{1}$ being the normalizing constant. We also have $\FKL{\bpi}{\BQ} = \BQ^\top \log\left( \frac{\BQ}{\bpi} \right)$ and $\RKL{\bpi}{\BQ} = \bpi^\top \log\left( \frac{\bpi}{\BQ} \right)$.
\end{definition}

The maximal possible FKL reduction is obtained by moving $\pinew$ all the way to $\BQold$, to give $\Delta \FKL{}{} = \FKL{\piold}{\BQold}$; for this $\pinew$, we can guarantee RKL reduction. The question is if we can still obtain RKL reduction, even without stepping all the way to this maximal possible FKL reduction. We provide a condition on how much FKL reduction is sufficient to ensure that we obtain policy improvement, first in the bandit setting and then generalized to the MDP setting.
\begin{restatable}[Sufficient FKL Reduction in Bandits]{proposition}{fklredbset}
	\label{prop:fklredbset}
	For two policies $\bpiold, \bpinew \in \R^{|\cA|}$ in the bandit setting, if 
\begin{align}
	&\Delta \FKL{\bpiold}{\bpinew} \nonumber\\
	&\quad\quad\geq  \max\{0, \FKL{\bpiold}{\BQ} - \frac{1}{2} \left( \frac{\tau}{\norm{\mathbf{q}}_\infty} \left( \RKL{\bpiold}{\BQ} + \BQ^\top \log(\bpiold) \right) \right)^2\}\label{eq:bound_fkl_general}
\end{align}
and
\begin{align*}
    \RKL{\bpiold}{\BQ} + \BQ^\top \log(\bpiold) \geq 0,
\end{align*}
	then $\Delta \RKL{\bpiold}{\bpinew} \geq 0$ and $\eta_\tau(\bpinew) \geq \eta_\tau(\bpiold)$.
\end{restatable}
\noindent
It is straightforward to extend this result to MDPs when we have reduction in every state.
\begin{corollary}[All-state Sufficient FKL reduction]
	\label{prop:fklredall} Assume the action set is finite. 
	If the assumptions in \Cref{prop:fklredbset} are satisfied for all $s \in \cS$, then $Q^{\pinew}_\tau(s,a) \geq Q^{\piold}_\tau(s,a)$ for all states and actions and $\eta_\tau(\pinew) \geq \eta_\tau(\piold)$.
\end{corollary}
\begin{proof}
We know that reducing the RKL in all states (i.e., the RKL of $\pinew$ to $\BQold$ is smaller than the RKL of $\piold$ to $\BQold$) will lead to $Q^{\pinew}_\tau(s,a) \geq Q^{\piold}_\tau(s,a)$ for all $(s,a)$. From our RKL results, we also know that it implies $\eta_\tau(\pinew) \geq \eta_\tau(\piold)$. Therefore, FKL reduction following \Cref{eq:bound_fkl_general} in all states will also lead to these improvements, since, by the same argument as \Cref{prop:fklredbset}, it leads to RKL reduction in all states.
\end{proof}

\fklredbset*

\begin{proof}
	We start with the result of \Cref{lemma:intermediate_ineq}. If the RHS of \Cref{eq:rkl_ub} is less than or equal to $\RKL{\bpiold}{\BQ}$, we will have $\RKL{\bpinew}{\BQ} \leq \RKL{\bpiold}{\BQ}$, which in turn implies improvement. The assumption that the RHS of \Cref{eq:rkl_ub} is $\leq \RKL{\bpiold}{\BQ}$ can be written as
	\begin{gather*}
		\frac{\norm{\bv}_\infty}{\tau} \sqrt{2 \FKL{\bpinew}{\BQ}} - \BQ^\top \log(\bpiold) \leq \RKL{\bpiold}{\BQ}.
	\end{gather*}
	With some algebraic manipulation, we get that this assumption is equivalent to
	\begin{align*}
	     \frac{\tau}{\norm{\bv}_\infty} (\RKL{\bpiold}{\BQ} + \BQ^\top \log(\bpiold)) \geq   \sqrt{2 \FKL{\bpinew}{\BQ}}.
	\end{align*}
	Assuming that $\RKL{\bpiold}{\BQ} + \BQ^\top \log(\bpiold) \geq 0$, the above is equivalent to
	\begin{align*}
	    \left(\frac{\tau}{\norm{\bv}_\infty} (\RKL{\bpiold}{\BQ} + \BQ^\top \log(\bpiold))\right)^2 \geq   2 \FKL{\bpinew}{\BQ}.
	\end{align*}
	Dividing by 2, adding $\FKL{\bpiold}{\BQ}$, and rearranging yields that the assumption is equivalent to the following. 
	\begin{gather}
		\label{eq:bound1}
		\Delta \FKL{\bpiold}{\bpinew} \geq \FKL{\bpiold}{\BQ} - \frac{1}{2} \left( \frac{\tau}{\norm{\bv}_\infty} \left( \RKL{\bpiold}{\BQ} + \BQ^\top \log(\bpiold) \right) \right)^2.
	\end{gather}	
	The claim follows.
\end{proof}

\fklredavg*
\begin{proof}
The strategy is the same as the one in \Cref{prop:fklredbset}. We know from \Cref{prop:avg-reverse-kl} that if we have
\begin{gather*}
	\Ex_{d^{\pinew}}[\Delta \RKL{\piold}{\pinew}(S) ] \geq 0,
\end{gather*}
then $\eta_\tau(\pinew) \geq \eta_\tau(\piold)$. Applying expectations to both sides of the result of \Cref{lemma:intermediate_ineq}, we obtain 
\begin{align*}
	&\EdBigg{\frac{\norm{Q^{\piold}_\tau(S,\cdot)}_\infty}{\tau}\sqrt{2\FKL{\pinew}{\BQold}(S)} - \Ex_{\BQold}[\log(\piold(\cdot|S))]} \\  &\quad\quad\geq \Ed{\RKL{\pinew}{\BQold}(S)}.
\end{align*}
If we can show that the LHS is smaller than 
$$\Ed{\RKL{\piold}{\BQold}(S)},$$ 
then, by \Cref{prop:avg-reverse-kl}, we will be guaranteed improvement. The condition can be written as
\begin{align}
	\label{eq:av_red_start}
	&\Ed{\RKL{\piold}{\BQold}(S)}  \geq\\ &\EdBigg{\frac{\norm{Q^{\piold}_\tau(S,\cdot)}_\infty}{\tau}\sqrt{2\FKL{\pinew}{\BQold}(S)} \, - \nonumber  \Ex_{\BQold}[\log(\piold(\cdot|S))]}.
\end{align}
Our goal is to derive the assumption on $\FKL{\pinew}{\BQold}(S)$ that must be made for \Cref{eq:av_red_start} to hold. 

From rearranging to have the FKL term on one side, the following condition is equivalent to \Cref{eq:av_red_start}. 
\begin{align*}	
	&\frac{\tau}{\norm{Q^{\piold}_\tau}_\infty} \Bigg( \Ed{\RKL{\piold}{\BQold}(S) +  \Ex_{\BQold}[\log(\piold(\cdot|S))] } \Bigg) \\ 
	&\quad\quad\geq \Ed{\sqrt{2\FKL{\pinew}{\BQold}(S)}}.
\end{align*}
We now square both sides. The following is also equivalent to \Cref{eq:av_red_start} by assumption in \Cref{eq:rkl-cross-ent-positive}.
\begin{align*}	
	&\Bigg( \frac{\tau}{\norm{Q^{\piold}_\tau}_\infty} \Big( \Ed{\RKL{\piold}{\BQold}(S) + \Ex_{\BQold}[\log(\piold(\cdot|S))] } \Big) \Bigg)^2 \\ 
	&\quad\quad \geq \left(\Ed{\sqrt{2\FKL{\pinew}{\BQold}(S)}}\right)^2.
\end{align*}
Jensen's inequality applied to the RHS shows that the following \textit{implies} \Cref{eq:av_red_start}.
\begin{align*}	
	& \Bigg( \frac{\tau}{\norm{Q^{\piold}_\tau}_\infty} \Big( \Ed{\RKL{\piold}{\BQold}(S) +  \Ex_{\BQold}[\log(\piold(\cdot|S))] } \Big) \Bigg)^2 \\ 
    	&\quad\quad\geq \Ed{2\FKL{\pinew}{\BQold}(S)}.
\end{align*}
Some straightforward rearrangement yields the assumption in the statement of the Proposition. 
\begin{align*}	
	&  - \frac{1}{2}\Bigg( \frac{\tau}{\norm{Q^{\piold}_\tau}_\infty} \Big( \Ed{\RKL{\piold}{\BQold}(S) +   \Ex_{\BQold}[\log(\piold(\cdot|S))] } \Big) \Bigg)^2 \\
	&\quad\quad\quad+ \Ed{\FKL{\piold}{\BQold}(S)}\\
	&\quad\quad\leq \Ex_{d^{\pinew}}[\Delta \FKL{\piold}{\pinew}(S) ].
\end{align*}		
\end{proof}

\subsubsection{Proofs for FKL with Approximate Action-values}

We can also extend the above results to approximate action-values. We first state all the results upfront, and then provide their proofs. 

We first reformulate \Cref{prop:fklredbset} with action-value estimates.
\begin{restatable}[Sufficient Approximate FKL Reduction in Bandits]{proposition}{approxfklredbset}\label{prop:approx-fklredbset}
Let $\hat Q$ be an action-value estimate of $Q^\piold_\tau$, let $\epsilon := \hat Q - Q$ be the approximation error, and let $\bar \epsilon := \Ex_{\pinew}[\epsilon(A)] - \Ex_{\piold}[\epsilon(A)]$. As well, define
\begin{align*}
    \hat \Delta \FKL{\bpiold}{\bpinew} := \FKL{\piold}{\mathcal{B}_\tau \hat Q }  - \FKL{\pinew}{ \mathcal{B}_\tau \hat Q}.
\end{align*}

For two policies $\bpiold, \bpinew \in \R^{|\cA|}$ in the bandit setting, if 
\begin{align*}
	&\hat \Delta \FKL{\bpiold}{\bpinew} \\
	&\quad\quad\geq \max\left\{0, \FKL{\bpiold}{\mathcal{B}_\tau \hat Q} - \frac{1}{2} \left( \frac{\tau}{\norm{\mathbf{\hat Q}}_\infty} \left( \RKL{\bpiold}{\mathcal{B}_\tau \hat Q} + \mathcal{B}_\tau \hat Q^\top \log(\bpiold) \right) \right)^2\right\},
\end{align*}
\begin{align*}
    \RKL{\bpiold}{\mathcal{B}_\tau \hat Q} - \bar \epsilon + \mathcal{B}_\tau \hat Q^\top \log(\bpiold) \geq 0,
\end{align*}
and $\bar \epsilon \leq \hat \Delta \RKL{\piold}{\pinew}(s)$, then $\Delta \RKL{\bpiold}{\bpinew} \geq 0$ and $\eta_\tau(\bpinew) \geq \eta_\tau(\bpiold)$.
\end{restatable}
An analogous result to \Cref{prop:fklredall} applies, with the same proof. 
\begin{corollary}[All-state Sufficient Approximate FKL reduction]
	\label{prop:approx-fklredall} Assume the action set is finite. 
	If the assumptions in \Cref{prop:approx-fklredbset} are satisfied for all $s \in \cS$, then $Q^{\pinew}_\tau(s,a) \geq Q^{\piold}_\tau(s,a)$ for all states and actions and $\eta_\tau(\pinew) \geq \eta_\tau(\piold)$.
\end{corollary}
Finally, we note an analogous result to \Cref{prop:fklredavg}. 
\begin{restatable}[Approximate Average FKL Reduction]{proposition}{approxfklredavg}\label{prop:approx-fklredavg}
Let $\hat Q$ be an action-value estimate of $Q^\piold_\tau$, let $\epsilon := \hat Q - Q$ be the approximation error, and let $\bar \epsilon := \Ex_{d^\pinew}[\Ex_{\pinew}[\epsilon(S, A)] - \Ex_{\piold}[\epsilon(S, A)]]$. 

For two policies $\bpiold, \bpinew \in \R^{|\cA|}$ in the bandit setting, if 
\begin{align*}
	&\hat \Ex_{d^\pinew}[\Delta \FKL{\bpiold}{\bpinew}] \\
	&\quad\quad\geq \Ex_{d^\pinew} [\FKL{\bpiold}{\mathcal{B}_\tau \hat Q}] - \frac{1}{2}\left( \frac{\tau}{\norm{\mathbf{\hat Q}}_\infty} \left(  \Ex_{d^\pinew}[\RKL{\bpiold}{\mathcal{B}_\tau \hat Q} + \mathcal{B}_\tau \hat Q^\top \log(\bpiold)] \right) \right)^2,
\end{align*}
\begin{align*}
    \Ex_{d^\pinew}\left[\RKL{\bpiold}{\mathcal{B}_\tau \hat Q}+ \mathcal{B}_\tau \hat Q^\top \log(\bpiold)\right] \geq \bar \epsilon,
\end{align*}
and $\bar \epsilon \leq \Ex_{d^\pinew}[\hat \Delta \RKL{\piold}{\pinew}(s)]$, then $\eta_\tau(\bpinew) \geq \eta_\tau(\bpiold)$.
\end{restatable}

\approxfklredbset*
\begin{proof}
We will again use \Cref{lemma:intermediate_ineq}. To obtain improvement by \Cref{cor:approximate-rkl-reduction}, we need that $\RKL{\pinew}{\mathcal{B}_\tau \hat Q} \leq \RKL{\piold}{\mathcal{B}_\tau \hat Q} - \bar \epsilon$. Using \Cref{lemma:intermediate_ineq}, we will require that 
\begin{align}\label{eq:initial-approx-reduction}
    \frac{\norm{\hat Q}_\infty}{\tau} \sqrt{2 \FKL{\bpinew}{\mathcal{B}_\tau \hat Q}} - \mathcal{B}_\tau \hat Q^\top \log(\bpiold) \leq \RKL{\bpiold}{\mathcal{B}_\tau \hat Q} - \bar \epsilon.
\end{align}
With some algebraic manipulation and assuming that $\RKL{\bpiold}{\mathcal{B}_\tau \hat Q} - \bar \epsilon + \mathcal{B}_\tau \hat Q^\top \log(\bpiold) \geq 0$, the above is equivalent to
\begin{align*}
    \left(\frac{\tau}{\norm{\hat Q}_\infty} (\RKL{\bpiold}{\mathcal{B}_\tau \hat Q} - \bar \epsilon + \mathcal{B}_\tau \hat Q^\top \log(\bpiold))\right)^2 \geq   2 \FKL{\bpinew}{\mathcal{B}_\tau \hat Q}.
\end{align*}
Dividing by 2, adding $\FKL{\bpiold}{\mathcal{B}_\tau \hat Q}$, and rearranging yields that the assumption is equivalent to the following. 
\begin{gather}
	\label{eq:approx-bound1}
	\hat \Delta \FKL{\bpiold}{\bpinew} \geq \FKL{\bpiold}{\mathcal{B}_\tau \hat Q} - \frac{1}{2} \left( \frac{\tau}{\norm{\hat Q}_\infty} \left( \RKL{\bpiold}{\mathcal{B}_\tau \hat Q} + \mathcal{B}_\tau \hat Q^\top \log(\bpiold) \right) \right)^2.
\end{gather}	
\end{proof}

\approxfklredavg*
\begin{proof}
The proof closely follows that of \Cref{prop:fklredavg}. 
Applying expectations to both sides of \Cref{eq:initial-approx-reduction}, we obtain 
\begin{align}\label{eq:approx-avg-assumption}
	&\Ex_{d^\pinew}\left[\frac{\norm{\hat Q}_\infty}{\tau} \sqrt{2 \FKL{\bpinew}{\mathcal{B}_\tau \hat Q}(S)} - \mathcal{B}_\tau \hat Q(\cdot \mid S)^\top \log(\bpiold(\cdot \mid S))\right]\\
	&\quad\quad\leq \Ex_{d^\pinew}[\RKL{\bpiold}{\mathcal{B}_\tau \hat Q}(s)] - \bar \epsilon.\nonumber
\end{align}
As per the discussion in \Cref{prop:approx-fklredbset}, satisfying this assumption will result in improvement. We hence derive an assumption on the FKL that implies \Cref{eq:approx-avg-assumption}.

From rearranging to have the FKL term on one side, the following condition is equivalent to \Cref{eq:approx-avg-assumption}. 
\begin{align*}	
	&\frac{\tau}{\norm{\hat Q}_\infty} \Bigg( \Ed{\RKL{\piold}{\mathcal{B}_\tau \hat Q}(S) - \bar \epsilon +  \Ex_{\mathcal{B}_\tau \hat Q}[\log(\piold(\cdot|S))] } \Bigg) \\ 
	&\quad\quad\geq \Ed{\sqrt{2\FKL{\pinew}{\mathcal{B}_\tau \hat Q}(S)}}.
\end{align*}
We now square both sides. The following is also equivalent to \Cref{eq:approx-avg-assumption} by assumption that $\Ex_{d^\pinew}\left[\RKL{\bpiold}{\mathcal{B}_\tau \hat Q}+ \mathcal{B}_\tau \hat Q^\top \log(\bpiold)\right] \geq \bar \epsilon$.
\begin{align*}	
	&\Bigg( \frac{\tau}{\norm{\hat Q}_\infty} \Big( \Ed{\RKL{\piold}{\mathcal{B}_\tau \hat Q}(S) - \bar \epsilon + \Ex_{\BQold}[\log(\piold(\cdot|S))] } \Big) \Bigg)^2 \\ 
	&\quad\quad \geq \left(\Ed{\sqrt{2\FKL{\pinew}{\mathcal{B}_\tau \hat Q}(S)}}\right)^2.
\end{align*}
Jensen's inequality applied to the RHS shows that the following \textit{implies} \Cref{eq:approx-avg-assumption}.
\begin{align*}	
	&\Bigg( \frac{\tau}{\norm{\hat Q}_\infty} \Big( \Ed{\RKL{\piold}{\mathcal{B}_\tau \hat Q}(S) - \bar \epsilon + \Ex_{\BQold}[\log(\piold(\cdot|S))] } \Big) \Bigg)^2 \\ 
	&\quad\quad \geq\Ed{2\FKL{\pinew}{\mathcal{B}_\tau \hat Q}(S)}.
\end{align*}
Some straightforward rearrangement yields the assumption in the statement of the Proposition. 
\begin{align*}	
	&  - \frac{1}{2}\Bigg( \frac{\tau}{\norm{\hat Q}_\infty} \Big( \Ed{\RKL{\piold}{\mathcal{B}_\tau \hat Q}(S) - \bar \epsilon +   \Ex_{\mathcal{B}_\tau \hat Q}[\log(\piold(\cdot|S))] } \Big) \Bigg)^2 \\
	&\quad\quad\quad+ \Ed{\FKL{\piold}{\mathcal{B}_\tau \hat Q}(S)}\\
	&\quad\quad\leq \Ex_{d^{\pinew}}[\Delta \FKL{\piold}{\pinew}(S) ].
\end{align*}		
\end{proof}

We now turn to proving the above results.
\approxfklredbset*
\begin{proof}
We will again use \Cref{lemma:intermediate_ineq}. To obtain improvement by \Cref{cor:approximate-rkl-reduction}, we need that $\RKL{\pinew}{\mathcal{B}_\tau \hat Q} \leq \RKL{\piold}{ \mathcal{B}_\tau \hat Q} - \bar \epsilon$. Using \Cref{lemma:intermediate_ineq}, we will require that 
\begin{align}
    \frac{\norm{\hat Q}_\infty}{\tau} \sqrt{2 \FKL{\bpinew}{\mathcal{B}_\tau \hat Q}} - \mathcal{B}_\tau \hat Q^\top \log(\bpiold) \leq \RKL{\bpiold}{\mathcal{B}_\tau \hat Q} - \bar \epsilon.
\end{align}
With some algebraic manipulation and assuming that $\RKL{\bpiold}{\mathcal{B}_\tau \hat Q} - \bar \epsilon + \mathcal{B}_\tau \hat Q^\top \log(\bpiold) \geq 0$, the above is equivalent to
\begin{align*}
    \left(\frac{\tau}{\norm{\hat Q}_\infty} (\RKL{\bpiold}{\mathcal{B}_\tau \hat Q} - \bar \epsilon + \mathcal{B}_\tau \hat Q^\top \log(\bpiold))\right)^2 \geq   2 \FKL{\bpinew}{\mathcal{B}_\tau \hat Q}.
\end{align*}
Dividing by 2, adding $\FKL{\bpiold}{\mathcal{B}_\tau \hat Q}$, and rearranging yields that the assumption is equivalent to the following. 
\begin{gather}
	\hat \Delta \FKL{\bpiold}{\bpinew} \geq \FKL{\bpiold}{\mathcal{B}_\tau \hat Q} - \frac{1}{2} \left( \frac{\tau}{\norm{\hat Q}_\infty} \left( \RKL{\bpiold}{\mathcal{B}_\tau \hat Q} + \mathcal{B}_\tau \hat Q^\top \log(\bpiold) \right) \right)^2.
\end{gather}	
\end{proof}

\approxfklredavg*
\begin{proof}
The proof closely follows that of \Cref{prop:fklredavg}. 
Applying expectations to both sides of \Cref{eq:initial-approx-reduction}, we obtain 
\begin{align}\label{eq:approx-avg-assumption}
	&\Ex_{d^\pinew}\left[\frac{\norm{\hat Q}_\infty}{\tau} \sqrt{2 \FKL{\bpinew}{\mathcal{B}_\tau \hat Q}(S)} - \mathcal{B}_\tau \hat Q(\cdot \mid S)^\top \log(\bpiold(\cdot \mid S))\right]\\
	&\quad\quad\leq \Ex_{d^\pinew}[\RKL{\bpiold}{\mathcal{B}_\tau \hat Q}(S)] - \bar \epsilon.\nonumber
\end{align}
As per the discussion in \Cref{prop:approx-fklredbset}, satisfying this assumption will result in improvement. We hence derive an assumption on the FKL that implies \Cref{eq:approx-avg-assumption}.

From rearranging to have the FKL term on one side, the following condition is equivalent to \Cref{eq:approx-avg-assumption}. 
\begin{align*}	
	&\frac{\tau}{\norm{\hat Q}_\infty} \Bigg( \Ed{\RKL{\piold}{\mathcal{B}_\tau \hat Q}(S) - \bar \epsilon +  \Ex_{\mathcal{B}_\tau \hat Q}[\log(\piold(\cdot|S))] } \Bigg) \\ 
	&\quad\quad\geq \Ed{\sqrt{2\FKL{\pinew}{\mathcal{B}_\tau \hat Q}(S)}}.
\end{align*}
We now square both sides. The following is also equivalent to \Cref{eq:approx-avg-assumption} by assumption that $\Ex_{d^\pinew}\left[\RKL{\bpiold}{\mathcal{B}_\tau \hat Q}+ \mathcal{B}_\tau \hat Q^\top \log(\bpiold)\right] \geq \bar \epsilon$.
\begin{align*}	
	&\Bigg( \frac{\tau}{\norm{\hat Q}_\infty} \Big( \Ed{\RKL{\piold}{\mathcal{B}_\tau \hat Q}(S) - \bar \epsilon + \Ex_{\BQold}[\log(\piold(\cdot|S))] } \Big) \Bigg)^2 \\ 
	&\quad\quad \geq \left(\Ed{\sqrt{2\FKL{\pinew}{\mathcal{B}_\tau \hat Q}(S)}}\right)^2.
\end{align*}
Jensen's inequality applied to the RHS shows that the following \textit{implies} \Cref{eq:approx-avg-assumption}.
\begin{align*}	
	&\Bigg( \frac{\tau}{\norm{\hat Q}_\infty} \Big( \Ed{\RKL{\piold}{\mathcal{B}_\tau \hat Q}(S) - \bar \epsilon + \Ex_{\BQold}[\log(\piold(\cdot|S))] } \Big) \Bigg)^2 \\ 
	&\quad\quad \geq\Ed{2\FKL{\pinew}{\mathcal{B}_\tau \hat Q}(S)}.
\end{align*}
Some straightforward rearrangement yields the assumption in the statement of the Proposition. 
\begin{align*}	
	&  - \frac{1}{2}\Bigg( \frac{\tau}{\norm{\hat Q}_\infty} \Big( \Ed{\RKL{\piold}{\mathcal{B}_\tau \hat Q}(S) - \bar \epsilon +   \Ex_{\mathcal{B}_\tau \hat Q}[\log(\piold(\cdot|S))] } \Big) \Bigg)^2 \\
	&\quad\quad\quad+ \Ed{\FKL{\piold}{\mathcal{B}_\tau \hat Q}(S)}\\
	&\quad\quad\leq \Ex_{d^{\pinew}}[\Delta \FKL{\piold}{\pinew}(S) ].
\end{align*}		
\end{proof}
\section{Complete Results for Performance in Benchmark Environments} \label{app:main_benchmark}

In this section, we compare the KL methods on benchmark continuous and discrete-action environments, using non-linear function approximation. Here, we wish to understand (1) if our observations from the microworld experiments apply to more complicated environments, (2) if there are any new differences as a result of function approximation or increased environment complexity and (3) if any of the KL divergences is more robust to hyperparameter choices than the other.

\subsection{Implementation Details}

The agents use the API Algorithm with KL Greedification, in Algorithm \ref{alg:kl-agent}. For the discrete action environments, we use the All-Actions updates, and for the continuous action environments we use the Sampled-Actions update, for both RKL and FKL, with 128 sampled actions. When evaluating the integral of the gradient for the RKL, we tested using the log-likelihood trick as well as the reparametrization trick. Since the last outperformed the first, we report results using reparametrization. All agents use experience replay with a buffer size of $10^6$ and use batch sizes of 32. 

Hyperparameter sweeps are performed separately for each domain. We use RMSprop for both the actor and critic. In the continuous action-setting, we sweep over the actor learning rates $\{10^{-5}, 10^{-4}, 10^{-3}, 10^{-2}\}$ and critic learning rates $\{10^{-5}, 10^{-4}, 10^{-3}, 10^{-2}, 10^{-1}\}$. In the discrete-action setting we have a shared learning rate because of a shared architecture and the sweep is done over the learning rates $\{10^{-5}, 10^{-4}, 10^{-3}, 10^{-2}, 10^{-1}\}$. We sweep temperatures in $\{10^{-3}, 5  \times 10^{-3}, 10^{-2}, 5  \times 10^{-2}, 10^{-1}, 5  \times 10^{-1}, 1\}$ for the soft action-value methods and additionally include runs with the hard action-value methods. The temperature in $\boltzmannQ$ and the temperature in the soft action-value function are set to be the same value. For example, if $\tau = 0.01$, then we learn a soft action-value function with $\tau = 0.01$ and use a KL target distribution proportional to $\exp(Q(s, a) \tau^{-1})$.

On our continuous-action domains, all policy and value function networks are implemented as two-layer neural networks of size 128, with ReLU activations. On our discrete-action domains, we employ the following architectures. In the OpenAI Gym environments, the architecture is a two-layer neural network of size 128 with ReLU activations, with the policy and value functions as separate heads off of the main two-layer body.
In MinAtar, the architecture is a convolutional network into one fully-connected layer for each of the policy, action value function, and state value function. The convolutional layer has 16 3x3 convolutions with stride 1, the same as in \citet{young2019minatar}. The size of the fully-connected layer is 128, with ReLU activations used between layers. 

\subsection{Performance}

For continuous actions, we experiment on Pendulum, Reacher, Swimmer and HalfCheetah \citep{todorov2012mujoco}. For discrete actions, we experiment on  OpenAI Gym environments \citep{brockman2016openai} and MinAtar environments \citep{young2019minatar}. In this section, we plot only a summary of the performance, detailed plots showing how it varies throughout training can be found in \Cref{app:addit_perf_results}. Temperatures $[1.0, 0.5, 0.1]$ were grouped together and labeled as ``High RKL/FK'', whereas $[0.05, 0.01, 0.005, 0.001, 0.0]$\footnote{Zero was excluded for FKL with continuous actions, see \Cref{app:addit_perf_results}} were grouped together and named ``Low RKL/FKL''. For each temperature, 30 seeds were run per hyperparameter setting and the best performing 20\% settings were selected, which were then grouped together based on the temperature as ``High/Low''. \Cref{fig:sum_deep_cont,fig:sum_deep_discr} report the average of the last half of the area under the curve for each group, as well as standard errors between all runs in the group. Returns are normalized between $0$ and $1$, with $0$ corresponding to the lower limit of the returns from the curves in \Cref{app:addit_perf_results} and $1$ to the highest. There was no striking pattern regarding which temperature is best overall, the choice seems to be highly environment dependent. Furthermore, FKL and RKL seem to perform comparably overall, with no clear dominance of one over the other, although the FKL performed slightly better in the few cases they were different.

\begin{figure}[tb!]
	\centering
	\begin{tabular}{c}
	\includegraphics[width=0.25\columnwidth]{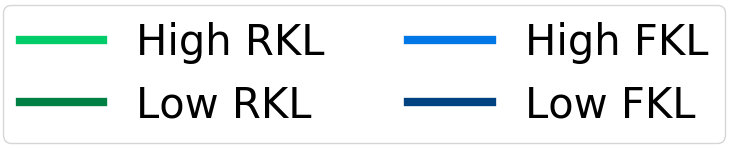}
	\end{tabular}
		\begin{tabular}{c c c c}
			\begin{subfigure}[b]{0.2\linewidth}
			\includegraphics[width=0.9\columnwidth]{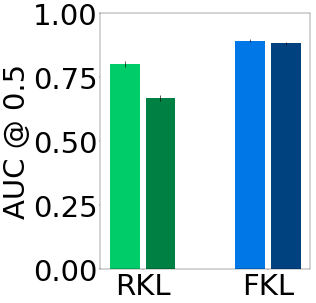} 
			\caption{Pendulum}\label{sub:sum_Pendulum}
			\end{subfigure}
			&    
			\begin{subfigure}[b]{0.2\linewidth}
			\includegraphics[width=0.9\columnwidth]{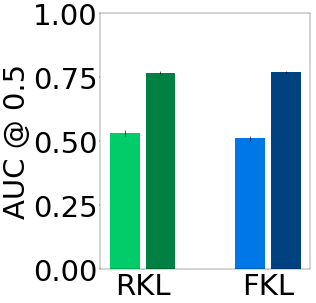} 
			\caption{Reacher}\label{sub:sum_Reacher}
			\end{subfigure}
			&
			\begin{subfigure}[b]{0.2\linewidth}
			\includegraphics[width=0.9\columnwidth]{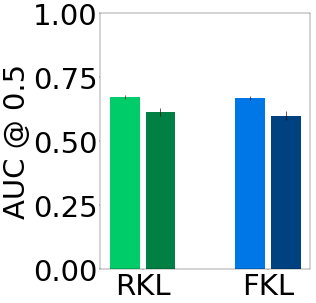} 
			\caption{Swimmer}\label{sub:sum_Swimmer}
			\end{subfigure}
			&
			\begin{subfigure}[b]{0.2\linewidth}
			\includegraphics[width=0.9\columnwidth]{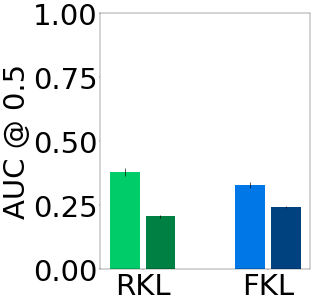} 
			\caption{HalfCheetah}\label{sub:sum_HalfCheetah}
			\end{subfigure}
		\end{tabular} 
	
	\caption{\textbf{Average return on the continuous-action environments}. The reported performance is the average return over the last half of learning, normalized between 0 and 1 and averaged over 30 runs with standard errors shown.}\label{fig:sum_deep_cont}
\end{figure}

\begin{figure}[tb!]
	\centering
	\begin{tabular}{c}
		\includegraphics[width=0.25\columnwidth]{figs/deep/summary/general/legend.png}
	\end{tabular}
\\
	\begin{tabular}{c c c c}
		\begin{subfigure}[b]{0.2\linewidth}
			\includegraphics[width=0.9\columnwidth]{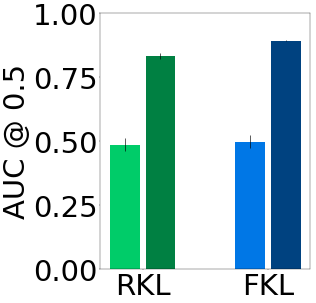} 
			\caption{Acrobot}\label{sub:sum_Acrobot}
		\end{subfigure}
		&    
		\begin{subfigure}[b]{0.2\linewidth}
			\includegraphics[width=0.9\columnwidth]{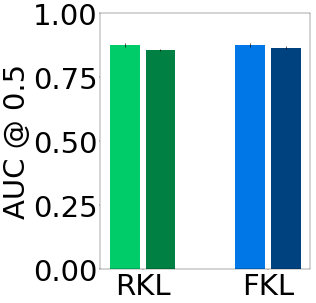} 
			\caption{CartPole}\label{sub:sum_CartPole}
		\end{subfigure}
		&
		\begin{subfigure}[b]{0.2\linewidth}
			\includegraphics[width=0.9\columnwidth]{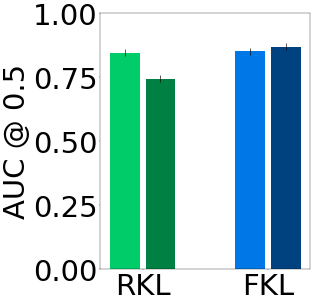} 
			\caption{LunarLander}\label{sub:sum_LunarLander}
		\end{subfigure}
		&
	\begin{subfigure}[b]{0.2\linewidth}
		\includegraphics[width=0.9\columnwidth]{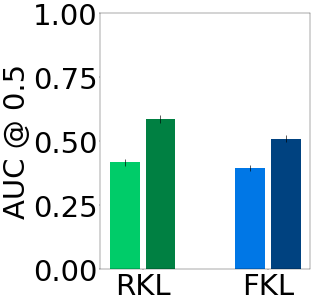} 
		\caption{Asterix}\label{sub:sum_asterix}
	\end{subfigure}		
	\end{tabular} 
\\
	\begin{tabular}{c c c c}
	\begin{subfigure}[b]{0.2\linewidth}
		\includegraphics[width=0.9\columnwidth]{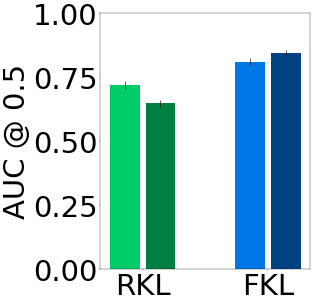} 
		\caption{Breakout}\label{sub:sum_breakout}
	\end{subfigure}
	&
	\begin{subfigure}[b]{0.2\linewidth}
		\includegraphics[width=0.9\columnwidth]{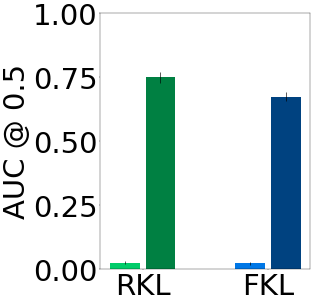} 
		\caption{Freeway}\label{sub:sum_freeway}
	\end{subfigure}
    &	
	\begin{subfigure}[b]{0.2\linewidth}
		\includegraphics[width=0.9\columnwidth]{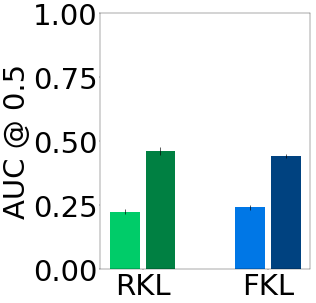} 
		\caption{Seaquest}\label{sub:sum_seaquest}
	\end{subfigure}
	&    
	\begin{subfigure}[b]{0.2\linewidth}
		\includegraphics[width=0.9\columnwidth]{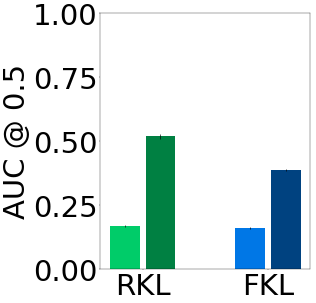} 
		\caption{Space invaders}\label{sub:sum_space_invaders}
	\end{subfigure}
\end{tabular} 
	\caption{\textbf{Average return on the discrete-action environments}. The reported performance is the average return over the last half of learning, normalized between 0 and 1 and averaged over 30 runs with standard errors shown.}\label{fig:sum_deep_discr}
\end{figure}

\begin{figure}[htb!]
  \centering
    \begin{tabular}{c c}
    \includegraphics[width=0.25\columnwidth]{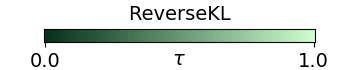} 
    &
    \includegraphics[width=0.25\columnwidth]{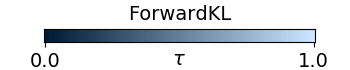}    
    \end{tabular} 

    \begin{subfigure}[b]{1.0\linewidth}
    \centering
    \begin{tabular}{c c c c}
    \includegraphics[width=0.22\columnwidth]{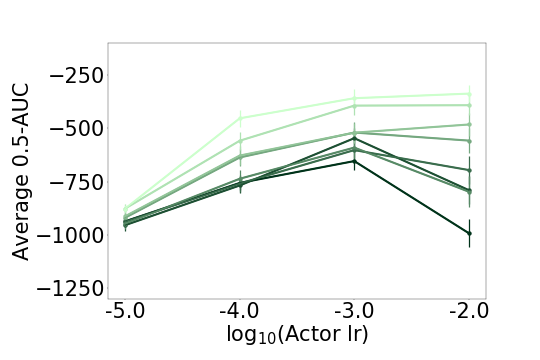} 
    &    
    \includegraphics[width=0.22\columnwidth]{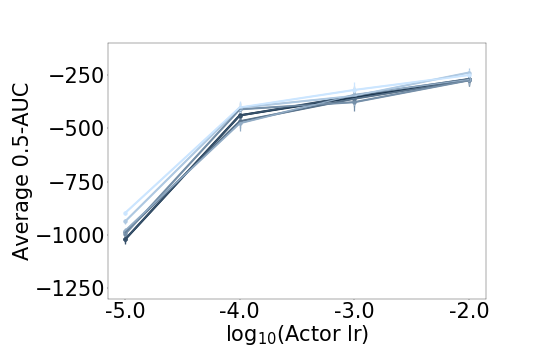} 
    &    
    \includegraphics[width=0.22\columnwidth]{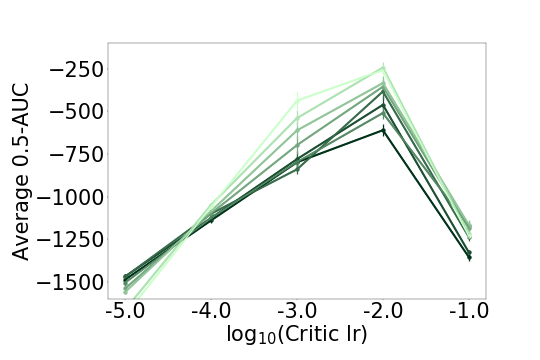} 
    &    
    \includegraphics[width=0.22\columnwidth]{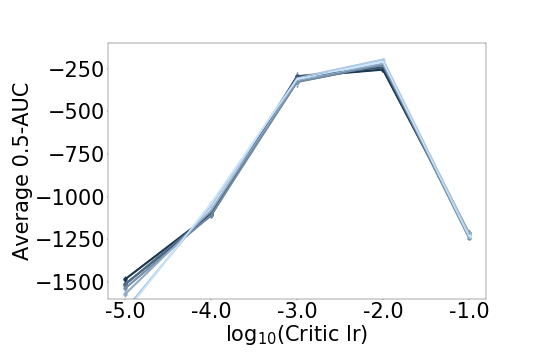} 
    \end{tabular} 
    \caption{Pendulum}\label{fig:Pendulum_sensitivity}        
  \end{subfigure}

    \begin{subfigure}[b]{1.0\linewidth}
    \centering
    \begin{tabular}{c c c c}
    \includegraphics[width=0.22\columnwidth]{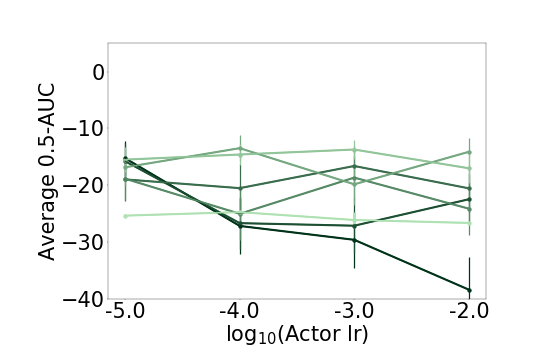} 
    &    
    \includegraphics[width=0.22\columnwidth]{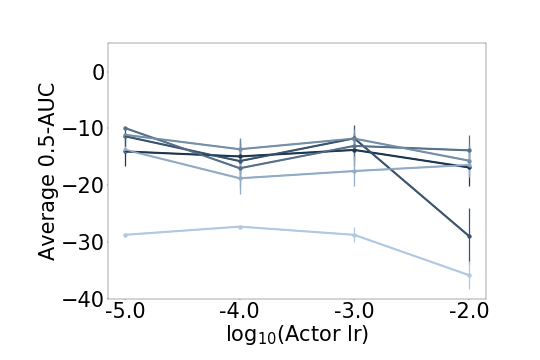} 
    &
    \includegraphics[width=0.22\columnwidth]{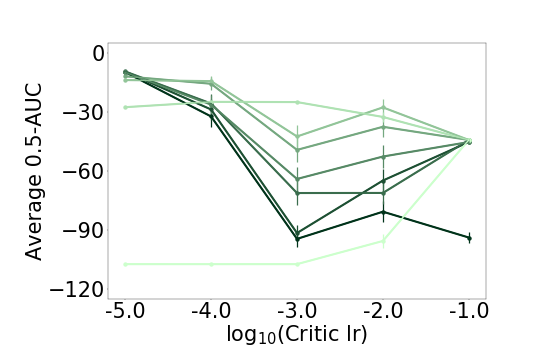} 
    &    
    \includegraphics[width=0.22\columnwidth]{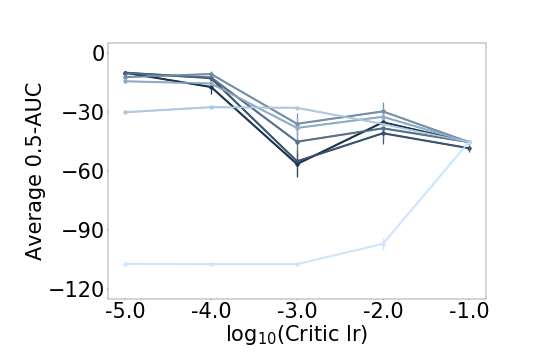}  
    \end{tabular} 
    \caption{Reacher}\label{fig:Reacher_sensitivity}
  \end{subfigure}

    \begin{subfigure}[b]{1.0\linewidth}
    \centering
    \begin{tabular}{c c c c}
    \includegraphics[width=0.22\columnwidth]{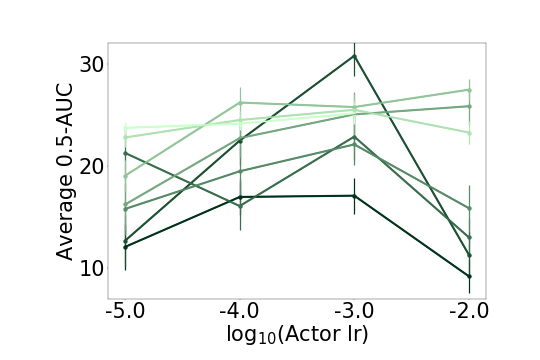} 
    &    
    \includegraphics[width=0.22\columnwidth]{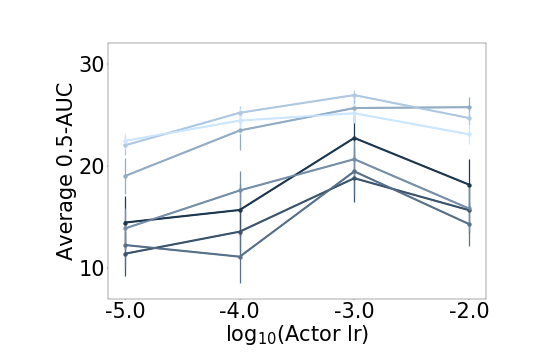} 
    &
    \includegraphics[width=0.22\columnwidth]{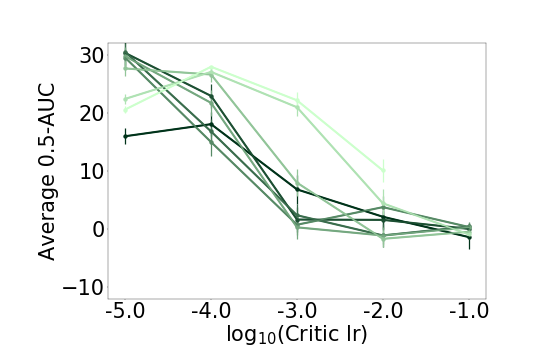} 
    &    
    \includegraphics[width=0.22\columnwidth]{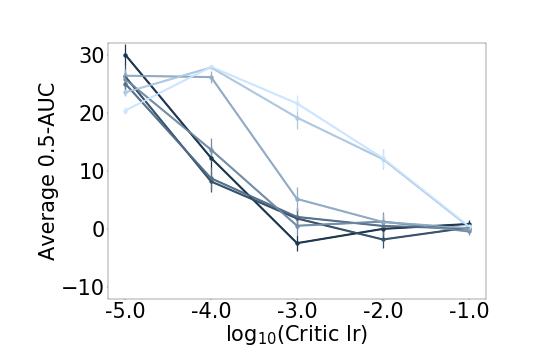}   
    \end{tabular} 
    \caption{Swimmer}\label{fig:Swimmer_sensitivity}
  \end{subfigure}    

    \begin{subfigure}[b]{1.0\linewidth}
    \centering
    \begin{tabular}{c c c c}
    \includegraphics[width=0.22\columnwidth]{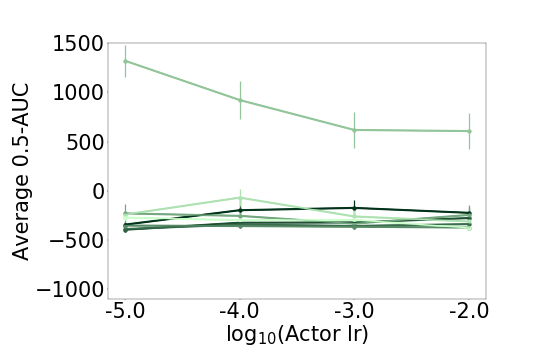} 
    &    
    \includegraphics[width=0.22\columnwidth]{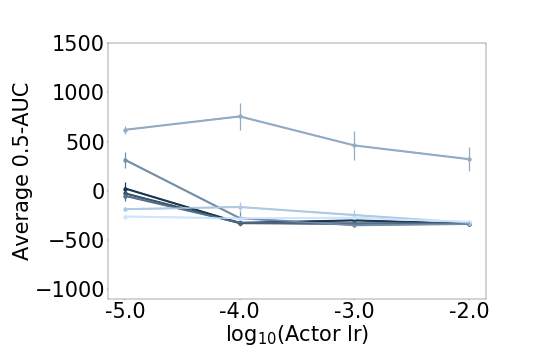} 
    &
    \includegraphics[width=0.22\columnwidth]{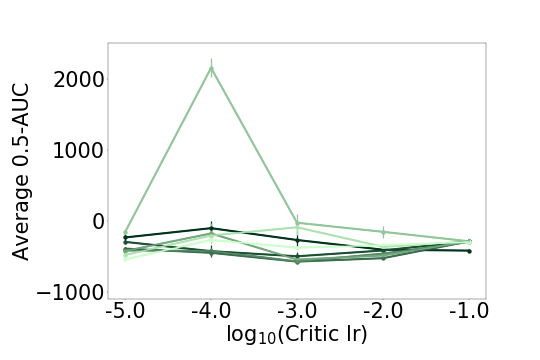} 
    &    
    \includegraphics[width=0.22\columnwidth]{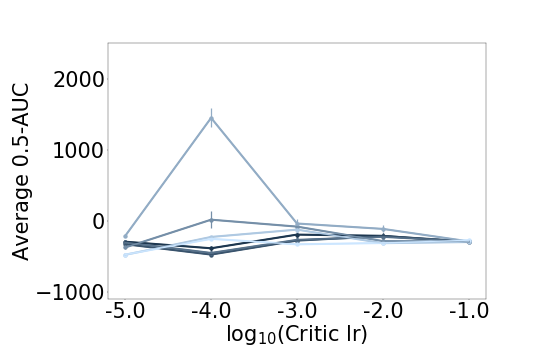}    
    \end{tabular} 
    \caption{HalfCheetah}\label{fig:HalfCheetah_sensitivity}
  \end{subfigure}

  \caption{Sensitivity plots for continuous-action environments. The x axis represent the hyperparameter values and the y axis represent the averaged area under the last half of the learning curve. Each point corresponds to the average of the top 20\% performing settings that had the corresponding learning rate and temperature}\label{fig:sensitivity_continuous}
\end{figure}

\subsection{Hyperparameter sensitivity}

We wrap up the experiments on benchmark problems by investigating the sensitivity of each divergence to hyperparameters. We focus on studying the sensitivity to the ones that seem to influence performance the most: learning rate and temperature. Particularly, we vary learning rates for each different temperature. For the continuous environments, where the actor and critic were separate networks, we have both the actor learning rate and the critic learning rate, whereas for the discrete environments we only have one learning rate.

From looking at Figures \ref{fig:sensitivity_continuous} and \ref{fig:sensitivity_discrete}, we see that both methods are sensitive to the hyperparameters, with the best parameters being highly environment dependent. For a given temperature and environment, RKL and FKL have very similar behavior: if the performance goes up for a certain learning rate for RKL, it also goes up for FKL and the same applies to decreases in performance, with very few exceptions and, even in those cases, the overall tendency is still the same for the two divergences. On Pendulum, represented in \Cref{fig:Pendulum_sensitivity}, for example, the worst learning rates perform better on FKL than in RKL, but the overall tendency is still for performance to go up as actor learning rates increase to 0.01 and as critic learning rates increase to 0.01, followed by a decrease when critic learning rate further increases to 0.1. The main takeaway is that, for a given choice of environment and temperature, learning rate sensitivity is not significantly influenced by the choice of divergence.

\begin{figure}[tb!]
  \centering
    \begin{tabular}{c c}
    \includegraphics[width=0.25\columnwidth]{figs/deep/general/colormaps_ReverseKL.png} 
    &
    \includegraphics[width=0.25\columnwidth]{figs/deep/general/colormaps_ForwardKL.png}    
    \end{tabular} 

    \begin{tabular}{c c}
    \centering
    \begin{subfigure}[b]{0.46\linewidth}
    \begin{tabular}{c c}
    \includegraphics[width=0.44\linewidth]{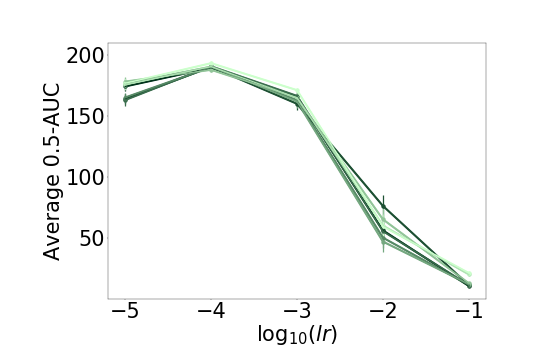}  
    &
    \includegraphics[width=0.44\linewidth]{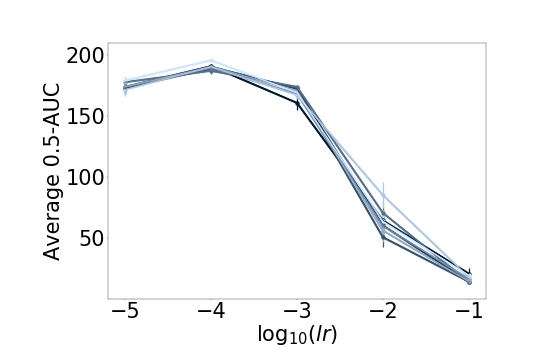} 
    \end{tabular}
    \caption{Gym:CartPole}\label{fig:sensitivity_CartPole}
    \end{subfigure}
    &    
    \begin{subfigure}[b]{0.46\linewidth}
    \begin{tabular}{c c}
    \includegraphics[width=0.44\linewidth]{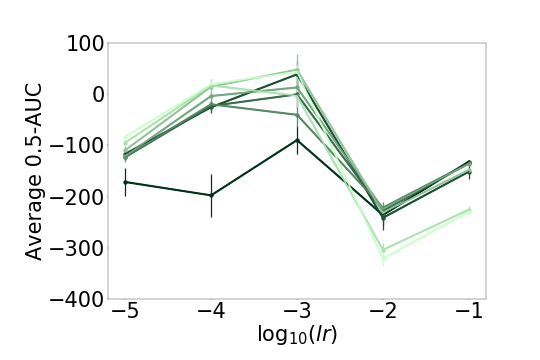}
    &
    \includegraphics[width=0.44\linewidth]{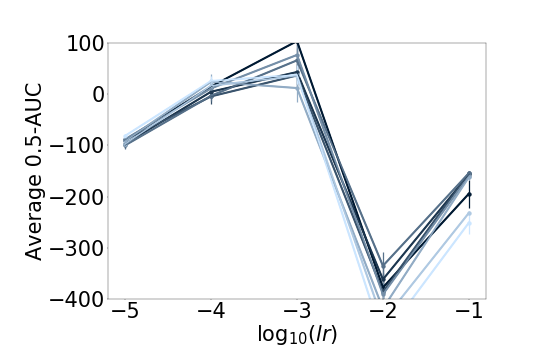} 
    \end{tabular}
    \caption{Gym:LunarLander}\label{fig:sensitivity_LunarLander}
    \end{subfigure}
    \end{tabular}

    \begin{tabular}{c c}
    \centering
    \begin{subfigure}[b]{0.46\linewidth}
    \begin{tabular}{c c}
    \includegraphics[width=0.44\linewidth]{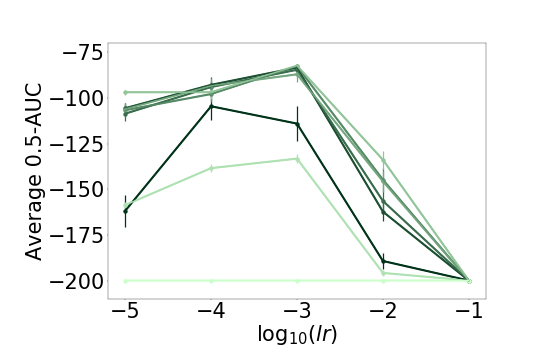}  
    &
    \includegraphics[width=0.44\linewidth]{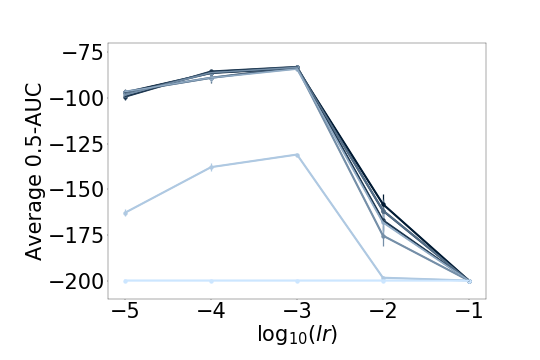} 
    \end{tabular}
    \caption{Gym:Acrobot}\label{fig:sensitivity_Acrobot}
    \end{subfigure}
    &    
    \begin{subfigure}[b]{0.46\linewidth}
    \begin{tabular}{c c}
    \includegraphics[width=0.44\linewidth]{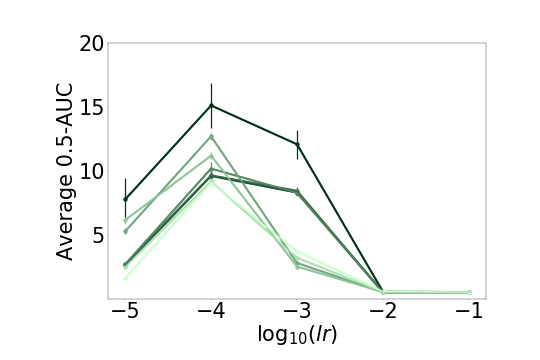}
    &
    \includegraphics[width=0.44\linewidth]{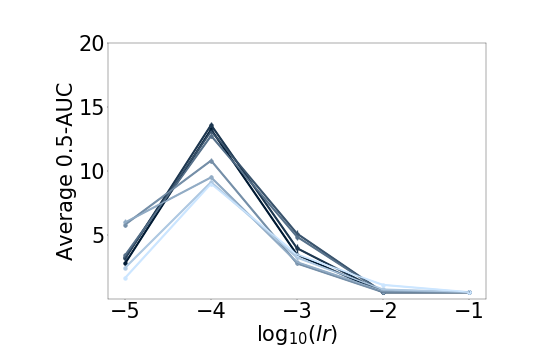} 
    \end{tabular}
    \caption{MinAtar:Asterix}\label{fig:sensitivity_asterix}
    \end{subfigure}
    \end{tabular}      

    \begin{tabular}{c c}
    \centering
    \begin{subfigure}[b]{0.46\linewidth}
    \begin{tabular}{c c}
    \includegraphics[width=0.44\linewidth]{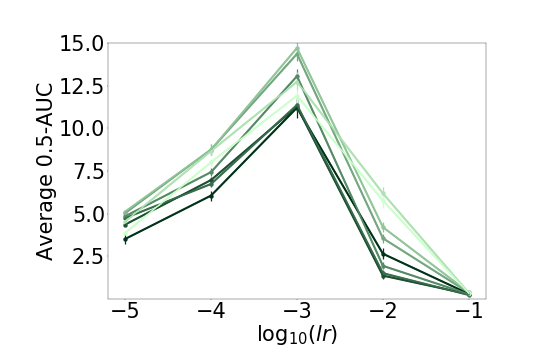}  
    &
    \includegraphics[width=0.44\linewidth]{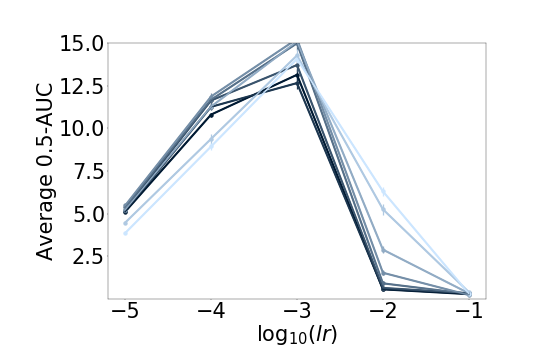} 
    \end{tabular}
    \caption{MinAtar:Breakout}\label{fig:sensitivity_breakout}
    \end{subfigure}
    &    
    \begin{subfigure}[b]{0.46\linewidth}
    \begin{tabular}{c c}
    \includegraphics[width=0.44\linewidth]{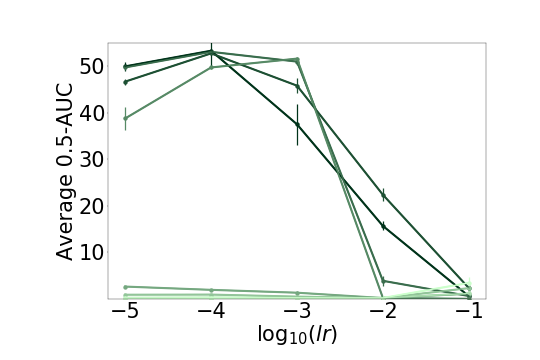}
    &
    \includegraphics[width=0.44\linewidth]{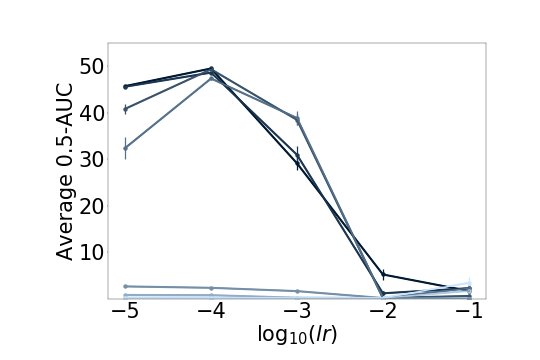} 
    \end{tabular}
    \caption{MinAtar:Freeway}\label{fig:sensitivity_freeway}
    \end{subfigure}
    \end{tabular}      

    \begin{tabular}{c c}
    \centering
    \begin{subfigure}[b]{0.46\linewidth}
    \begin{tabular}{c c}
    \includegraphics[width=0.44\linewidth]{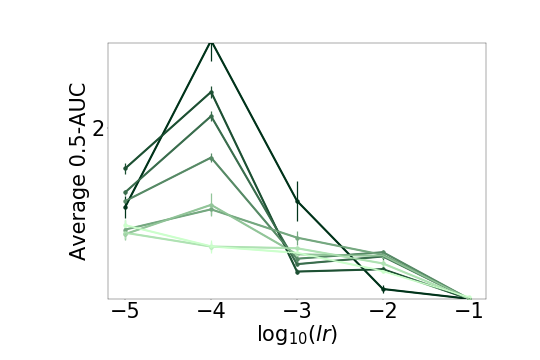}  
    &
    \includegraphics[width=0.44\linewidth]{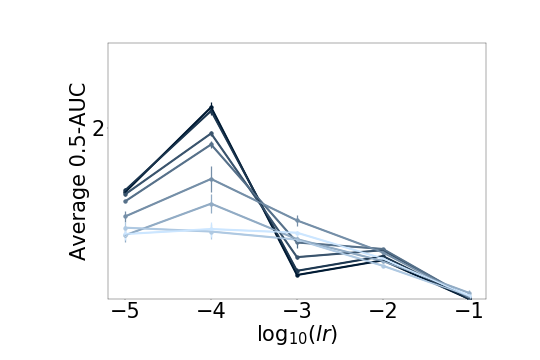} 
    \end{tabular}
    \caption{MinAtar:Seaquest}\label{fig:sensitivity_seaquest}
    \end{subfigure}
    &    
    \begin{subfigure}[b]{0.46\linewidth}
    \begin{tabular}{c c}
    \includegraphics[width=0.44\linewidth]{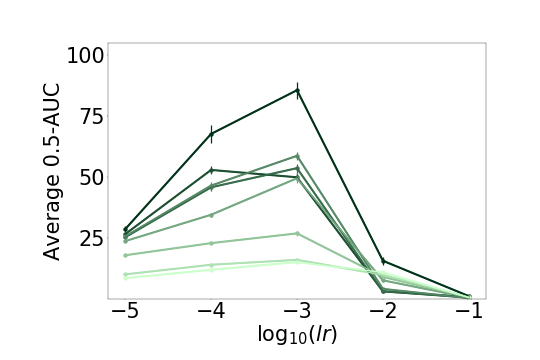}
    &
    \includegraphics[width=0.44\linewidth]{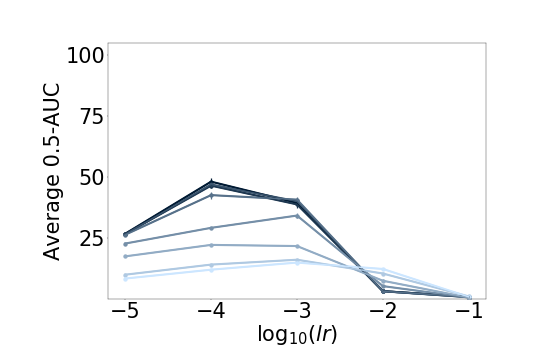} 
    \end{tabular}
    \caption{MinAtar:Space Invaders}\label{fig:sensitivity_space_invaders}
    \end{subfigure}
    \end{tabular}                  
   
  \caption{Sensitivity plots for discrete-action environments. Plot settings are identical to \Cref{fig:sensitivity_continuous}}\label{fig:sensitivity_discrete}
\end{figure}

\section{Additional Experimental Results} \label{app:addit_results}

\subsection{Exploration in a Discrete Maze}\label{app:exploration}
For these experiments, we use Gym-maze\footnote{\url{https://github.com/MattChanTK/Gym-maze}}, choosing their fixed $10\times10$ maze to plot the changes in state-visitation distribution throughout training. This way, some of the differences we may be able to inspect are: if any one of the divergences becomes deterministic quicker than the other; if any of the divergences get stuck in local optima while the other succeeds in finding the optimal policy; how spread-out the state-visitation distribution is in the early phases of training, when the agent is more likely to be exploring; if there are different priorities between regions of the maze when the agents are exploring. 

For the Gym-maze, the dynamics $\Pr(s_{t+1} \mid s_{t}, a_{t}) $ are deterministic and the actions are the four directions. The agent remains in place if it tries to move to a position where there is a wall. The reward is -0.1 divided by the total number of cells if $s_{t+1}$ is not the goal state and is 1.0 for the goal state. There is a timeout if the agent does not reach the goal after 10,000 steps. The agent is given a tabular representation: a one-hot encoding of the (x,y) position. This means that the agent has to create some sort of mental map of the maze from the positions it visited, causing the environment to be a harder exploratory problem than it might seem at first glance. We use the same agent that will be used in \Cref{sec:main_benchmark}, based on Algorithm \ref{alg:kl-agent}, which was introduced in Section \ref{sec:api-alg}.

On each iteration, one gradient descent step is performed to update the policy, for a given value function. 

\subsubsection{Exploration with True Values} \label{app:exp_true}
We want to understand differences in exploration when minimizing the KL divergences with respect to the estimated value functions. Even so, it is important to also study how the state-visitation distribution changes when using accurate value functions, which we will henceforth refer to as true values, noting that they are still only approximations of the actual values that are calculated more precisely. In this setting, the agent does not need to explore or even interact with the environment directly; the dynamics are given. These studies make clear not only how entropy affects the convergence of policies but also make it easier to disentangle which behaviors and results are due to the use of estimated values and which happen even with true values.

By hand annotating the dynamics, we can use dynamic programming to compute more accurate estimates of the value function, as opposed to the common approach of using gradient descent methods based on the Bellman equation and least squares. The gradient for the greedification step is computed over all states, as opposed to using a buffer, and the learning rate is set to $0.1$ with RMSprop. The total number of iterations is 100.

The stopping condition for dynamic programming is when the relative difference between successive Q's is less than $0.01\%$ for 10 consecutive iterations. The policy update, for either FKL or RKL, requires both Q and V. We compute V directly from Q, by summing over all four actions weighted by the current policy. The most representative timesteps are illustrated in \Cref{fig:true_q_expl}.

\begin{figure}[tb!]
  \centering
    \begin{tabular}{c c c}
    \includegraphics[width=0.22\columnwidth]{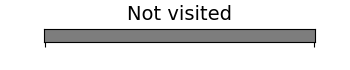}
    &
    \includegraphics[width=0.22\columnwidth]{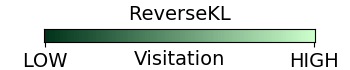} 
    &
    \includegraphics[width=0.22\columnwidth]{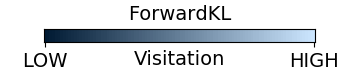} 
    \end{tabular}  

  \begin{subfigure}[b]{1.0\linewidth}
    \centering
    \begin{tabular}{c c c c}
    \begin{subfigure}{0.15\columnwidth}
    \centering
    \includegraphics[width=1.0\columnwidth]{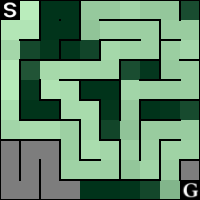} \\[3pt]
    \includegraphics[width=1.0\columnwidth]{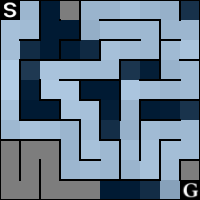}
    \caption*{$t=0$}
    \end{subfigure}
    &    
    \begin{subfigure}{0.15\columnwidth}
    \includegraphics[width=1.0\columnwidth]{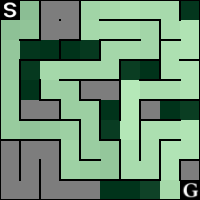} \\[3pt]
    \includegraphics[width=1.0\columnwidth]{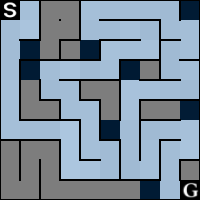}
    \caption*{$t=10$}
    \end{subfigure}
    &
    \begin{subfigure}{0.15\columnwidth}
    \includegraphics[width=1.0\columnwidth]{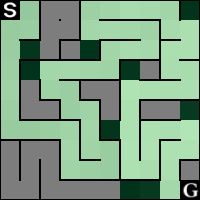} \\[3pt]
    \includegraphics[width=1.0\columnwidth]{figs/exploration/true_q/fkl_0.0_policy_iter_1.png}
    \caption*{$t=50$}
    \end{subfigure}
    &
    \begin{subfigure}{0.15\columnwidth}
    \includegraphics[width=1.0\columnwidth]{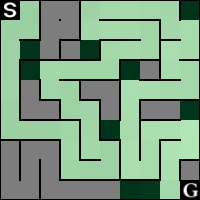} \\[3pt] 
    \includegraphics[width=1.0\columnwidth]{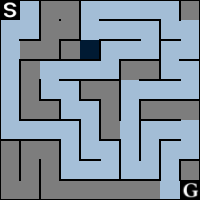}    
    \caption*{$t=99$}
    \end{subfigure}
    \end{tabular} 
 
    \caption{$\tau = 0.0$}\label{fig:expl_tau_0.0}
  \end{subfigure}
  
  \vspace{7pt}

  \begin{subfigure}[b]{1.0\linewidth}
    \centering
    \begin{tabular}{c c c c}
    \begin{subfigure}{0.15\columnwidth}
    \centering
    \includegraphics[width=1.0\columnwidth]{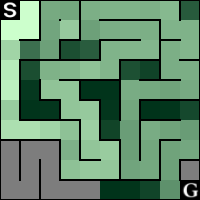} \\[3pt]
    \includegraphics[width=1.0\columnwidth]{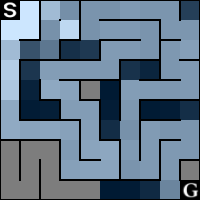}
    \caption*{$t=0$}
    \end{subfigure}
    &    
    \begin{subfigure}{0.15\columnwidth}
    \includegraphics[width=1.0\columnwidth]{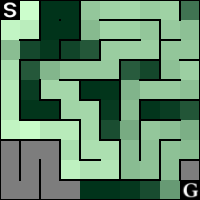} \\[3pt]
    \includegraphics[width=1.0\columnwidth]{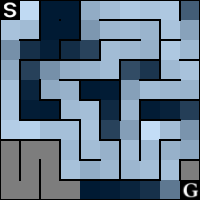}
    \caption*{$t=4$}
    \end{subfigure}
    &
    \begin{subfigure}{0.15\columnwidth}
    \includegraphics[width=1.0\columnwidth]{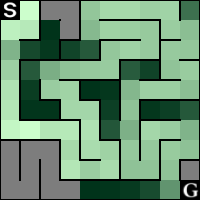} \\[3pt]
    \includegraphics[width=1.0\columnwidth]{figs/exploration/true_q/fkl_0.01_policy_iter_1.png}
    \caption*{$t=10$}
    \end{subfigure}
    &
    \begin{subfigure}{0.15\columnwidth}
    \includegraphics[width=1.0\columnwidth]{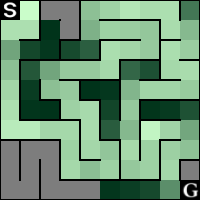} \\[3pt] 
    \includegraphics[width=1.0\columnwidth]{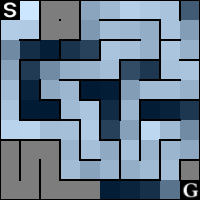}    
    \caption*{$t=99$}
    \end{subfigure}
    \end{tabular} 
 
    \caption{$\tau = 0.01$}\label{fig:expl_tau_0.01}
  \end{subfigure}
  
  \vspace{7pt}

  \begin{subfigure}[b]{1.0\linewidth}
    \centering
    \begin{tabular}{c c c c}
    \begin{subfigure}{0.15\columnwidth}
    \centering
    \includegraphics[width=1.0\columnwidth]{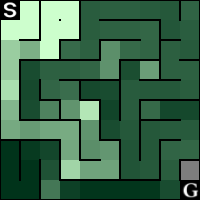} \\[3pt]
    \includegraphics[width=1.0\columnwidth]{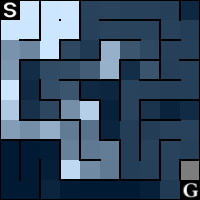}
    \caption*{$t=0$}
    \end{subfigure}
    &    
    \begin{subfigure}{0.15\columnwidth}
    \includegraphics[width=1.0\columnwidth]{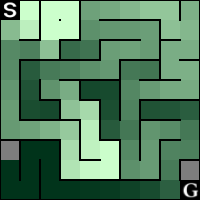} \\[3pt]
    \includegraphics[width=1.0\columnwidth]{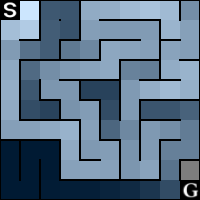}
    \caption*{$t=10$}
    \end{subfigure}
    &
    \begin{subfigure}{0.15\columnwidth}
    \includegraphics[width=1.0\columnwidth]{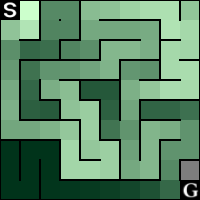} \\[3pt]
    \includegraphics[width=1.0\columnwidth]{figs/exploration/true_q/fkl_0.1_policy_iter_1.png}
    \caption*{$t=75$}
    \end{subfigure}
    &
    \begin{subfigure}{0.15\columnwidth}
    \includegraphics[width=1.0\columnwidth]{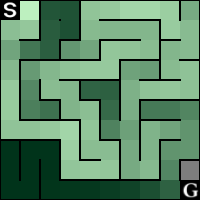} \\[3pt] 
    \includegraphics[width=1.0\columnwidth]{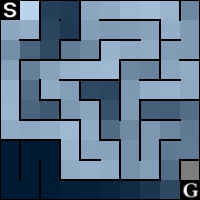}    
    \caption*{$t=99$}
    \end{subfigure}
    \end{tabular} 
 
    \caption{$\tau = 0.1$}\label{fig:expl_tau_0.1}
  \end{subfigure}    

  \caption{Evolution of state-visitation distributions throughout training for each temperature with true values.}
  \label{fig:true_q_expl}
\end{figure}

\subsubsection{Exploration with Estimated Values} \label{app:exp_estimate}


In this section, we use the more practical approach of updating value function estimates via expected semi-gradient updates. The optimizer used was RMSprop with learning rate 0.001 and the total number of iterations is 20000. We use a mini-batch of 32 states sampled from a buffer. The buffer size is 10000.

Results are illustrated in \Cref{fig:no_true_q_expl}, where we show only the subset of timesteps most representative of changes and only the plots corresponding to $\tau = 0$. 
The start state is in the top left position and the goal state is in the bottom right. To obtain these images, 30 seeds were used for each temperature-divergence combination and trained for the full number of iterations. The image at timestep $t$ corresponds to 100 trajectories generated from each of the 30 policies at that timestep. Particularly, we take the visitation counts of each of these $30\times100$ trajectories, normalize them and average them, producing an image representative of the overall exploratory behavior of that divergence-temperature combination. The figure shows that, given a certain temperature, both RKL and FKL have state-visitation distributions that evolve very similarly, as can be seen by comparing the pairwise green and blue images for each timestep. 

\begin{figure}[tb!]
	\centering
	\begin{tabular}{c c c}
		\includegraphics[width=0.22\columnwidth]{figs/exploration/general/not_visited.png}
		&
		\includegraphics[width=0.22\columnwidth]{figs/exploration/general/colormaps_ReverseKL.png} 
		&
		\includegraphics[width=0.22\columnwidth]{figs/exploration/general/colormaps_ForwardKL.png} 
	\end{tabular}  
	
		\centering
		\begin{tabular}{c c c}
			\begin{subfigure}{0.15\columnwidth}
				\centering
				\includegraphics[width=1.0\columnwidth]{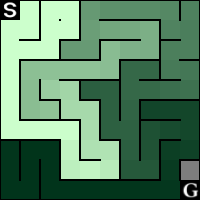} \\[3pt]
				\includegraphics[width=1.0\columnwidth]{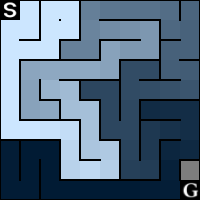}
				\caption*{t=0}
			\end{subfigure}
			&    
			\begin{subfigure}{0.15\columnwidth}
				\includegraphics[width=1.0\columnwidth]{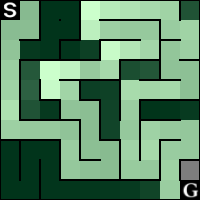} \\[3pt] 
				\includegraphics[width=1.0\columnwidth]{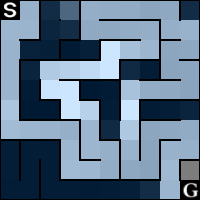}    
				\caption*{t=8399}
			\end{subfigure}    
			&
			\begin{subfigure}{0.15\columnwidth}
				\includegraphics[width=1.0\columnwidth]{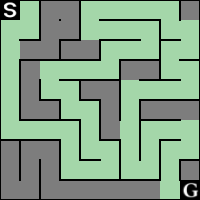} \\[3pt] 
				\includegraphics[width=1.0\columnwidth]{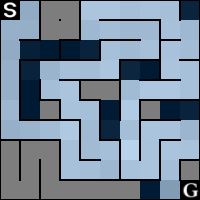}    
				\caption*{t=19999}
			\end{subfigure}
		\end{tabular} 
	
	\caption{Evolution of state-visitation distributions throughout training for $\tau = 0$}\label{fig:no_true_q_expl}
\end{figure}


\Cref{fig:app_no_true_q_expl} makes it clear that, for this setting, entropy regularization seems to do more harm than good: the agents converge to policies that tend to go to the correct trajectory, but waste time exploring when they already have all information they need. This is one example where the maximization of entropy conflicts with the true objective, ideally, one would prefer to have a deterministic policy for this environment at the end of training.

\begin{figure}[tb!]
  \centering
    \begin{tabular}{c c c}
    \includegraphics[width=0.22\columnwidth]{figs/exploration/general/not_visited.png}
    &
    \includegraphics[width=0.22\columnwidth]{figs/exploration/general/colormaps_ReverseKL.png} 
    &
    \includegraphics[width=0.22\columnwidth]{figs/exploration/general/colormaps_ForwardKL.png} 
    \end{tabular}  

  \begin{subfigure}[b]{1.0\linewidth}
    \centering
    \begin{tabular}{c c c c c}
    \begin{subfigure}{0.15\columnwidth}
    \centering
    \includegraphics[width=1.0\columnwidth]{figs/exploration/no_true_q/rkl_0.0_policy_iter_0.png} \\[3pt]
    \includegraphics[width=1.0\columnwidth]{figs/exploration/no_true_q/fkl_0.0_policy_iter_0.png}
    \caption*{t=0}
    \end{subfigure}
    &    
    \begin{subfigure}{0.15\columnwidth}
    \includegraphics[width=1.0\columnwidth]{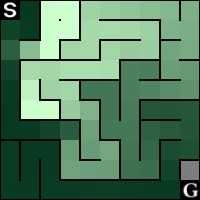} \\[3pt]
    \includegraphics[width=1.0\columnwidth]{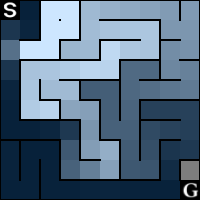}
    \caption*{t=1099}
    \end{subfigure}
    &
    \begin{subfigure}{0.15\columnwidth}
    \includegraphics[width=1.0\columnwidth]{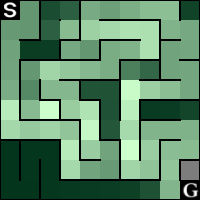} \\[3pt]
    \includegraphics[width=1.0\columnwidth]{figs/exploration/no_true_q/fkl_0.0_policy_iter_1.png}
    \caption*{t=6499}
    \end{subfigure}
    &
    \begin{subfigure}{0.15\columnwidth}
    \includegraphics[width=1.0\columnwidth]{figs/exploration/no_true_q/rkl_0.0_policy_iter_3.png} \\[3pt] 
    \includegraphics[width=1.0\columnwidth]{figs/exploration/no_true_q/fkl_0.0_policy_iter_3.png}    
    \caption*{t=8399}
    \end{subfigure}    
    &
    \begin{subfigure}{0.15\columnwidth}
    \includegraphics[width=1.0\columnwidth]{figs/exploration/no_true_q/rkl_0.0_policy_iter_4.png} \\[3pt] 
    \includegraphics[width=1.0\columnwidth]{figs/exploration/no_true_q/fkl_0.0_policy_iter_4.png}    
    \caption*{t=19999}
    \end{subfigure}
    \end{tabular} 
 
    \caption{$\tau = 0.0$}\label{fig:app_expl_ntq_tau_0.0}
  \end{subfigure}
  
   \vspace{7pt}

  \begin{subfigure}[b]{1.0\linewidth}
    \centering
    \begin{tabular}{c c c c c}
    \begin{subfigure}{0.15\columnwidth}
    \centering
    \includegraphics[width=1.0\columnwidth]{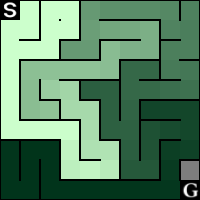} \\[3pt]
    \includegraphics[width=1.0\columnwidth]{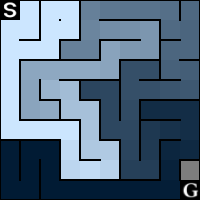}
    \caption*{t=0}
    \end{subfigure}
    &    
    \begin{subfigure}{0.15\columnwidth}
    \includegraphics[width=1.0\columnwidth]{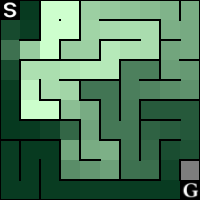} \\[3pt]
    \includegraphics[width=1.0\columnwidth]{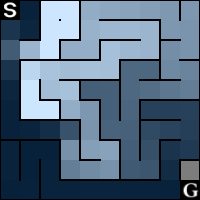}
    \caption*{t=1099}
    \end{subfigure}
    &
    \begin{subfigure}{0.15\columnwidth}
    \includegraphics[width=1.0\columnwidth]{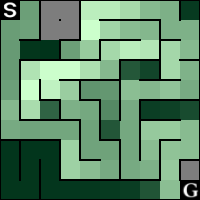} \\[3pt]
    \includegraphics[width=1.0\columnwidth]{figs/exploration/no_true_q/fkl_0.0001_policy_iter_1.png}
    \caption*{t=6499}
    \end{subfigure}
    &
    \begin{subfigure}{0.15\columnwidth}
    \includegraphics[width=1.0\columnwidth]{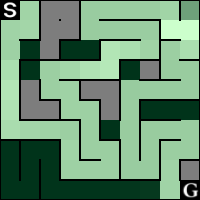} \\[3pt] 
    \includegraphics[width=1.0\columnwidth]{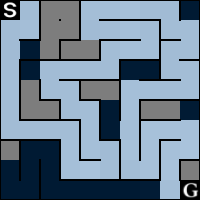}    
    \caption*{t=11699}
    \end{subfigure}    
    &
    \begin{subfigure}{0.15\columnwidth}
    \includegraphics[width=1.0\columnwidth]{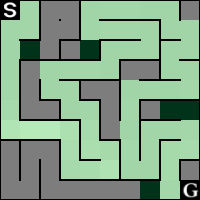} \\[3pt] 
    \includegraphics[width=1.0\columnwidth]{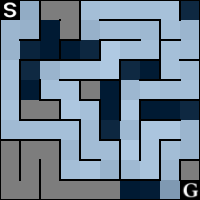}    
    \caption*{t=19999}
    \end{subfigure}
    \end{tabular} 
 
    \caption{$\tau = 0.0001$}\label{fig:app_expl_ntq_tau_0.0001}
  \end{subfigure}
  
   \vspace{7pt}

  \begin{subfigure}[b]{1.0\linewidth}
    \centering
    \begin{tabular}{c c c c c}
    \begin{subfigure}{0.15\columnwidth}
    \centering
    \includegraphics[width=1.0\columnwidth]{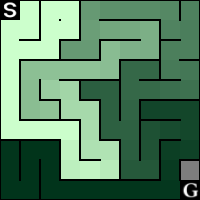} \\[3pt]
    \includegraphics[width=1.0\columnwidth]{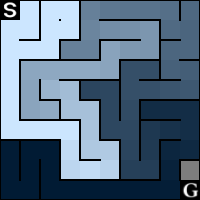}
    \caption*{t=0}
    \end{subfigure}
    &    
    \begin{subfigure}{0.15\columnwidth}
    \includegraphics[width=1.0\columnwidth]{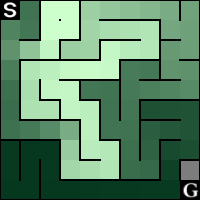} \\[3pt]
    \includegraphics[width=1.0\columnwidth]{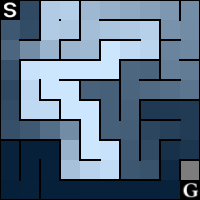}
    \caption*{t=1099}
    \end{subfigure}
    &
    \begin{subfigure}{0.15\columnwidth}
    \includegraphics[width=1.0\columnwidth]{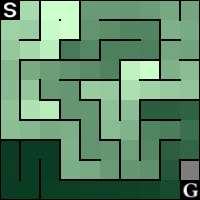} \\[3pt]
    \includegraphics[width=1.0\columnwidth]{figs/exploration/no_true_q/fkl_0.0007_policy_iter_1.png}
    \caption*{t=6499}
    \end{subfigure}
    &
    \begin{subfigure}{0.15\columnwidth}
    \includegraphics[width=1.0\columnwidth]{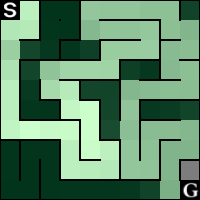} \\[3pt] 
    \includegraphics[width=1.0\columnwidth]{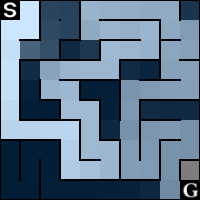}    
    \caption*{t=11999}
    \end{subfigure}    
    &
    \begin{subfigure}{0.15\columnwidth}
    \includegraphics[width=1.0\columnwidth]{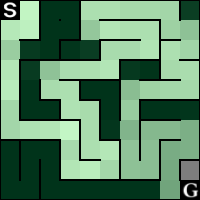} \\[3pt] 
    \includegraphics[width=1.0\columnwidth]{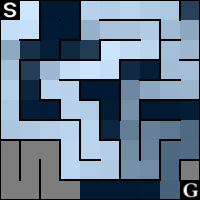}    
    \caption*{t=19999}
    \end{subfigure}
    \end{tabular} 
 
    \caption{$\tau = 0.0007$}\label{fig:app_expl_ntq_tau_0.0007}
  \end{subfigure}      

  \caption{Evolution of state-visitation distributions throughout training for each temperature with estimated values.}
  \label{fig:app_no_true_q_expl}
\end{figure}

\subsection{Exploration in a Continuous Maze} \label{app:exp_cont}

This section gives a more in-depth view of the experiments from \Cref{sec:expl_maze_cont}. Figure \ref{fig:cont_maze_detail} plots the cumulative number of times both the misleading and the correct exits are reached throughout training using 30 seeds and 2M steps, instead of 500k. For $\tau = 1000$, RKL and FKL have very similar curves; for $\tau = 100$ and $\tau=10$ the RKL visits the misleading exit more and the correct exit less. For the remaining temperatures, FKL visits both exits more than the RKL. For lower temperatures, the misleading exit is visited orders of magnitude more than the regular one. These all corroborate with the conclusion that the FKL is more exploratory in this setting.

\begin{figure}[tb!]
	\centering
	\begin{tabular}{c}
		\includegraphics[width=0.35\columnwidth]{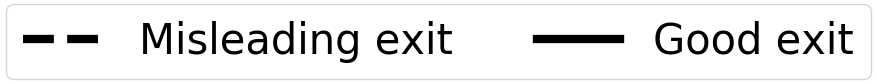}
	\end{tabular}  
	\centering
	\begin{tabular}{c c}
		\begin{subfigure}{0.43\columnwidth}
			\begin{tabular}{c c}
			\includegraphics[width=0.43\columnwidth]{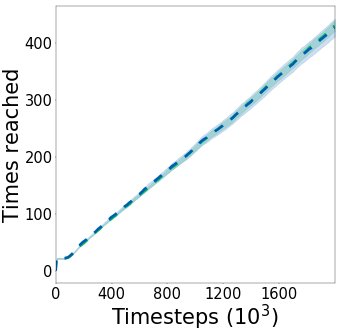} & 
			\includegraphics[width=0.43\columnwidth]{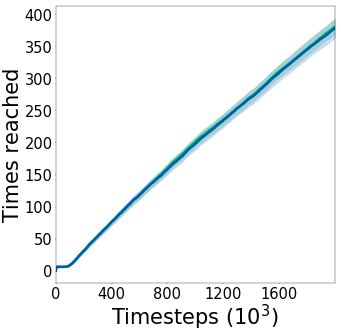}
			\end{tabular}
		\caption{$\tau=1000$}
		\end{subfigure}  
		&		
		\begin{subfigure}{0.43\columnwidth}
			\begin{tabular}{c c}
			\includegraphics[width=0.43\columnwidth]{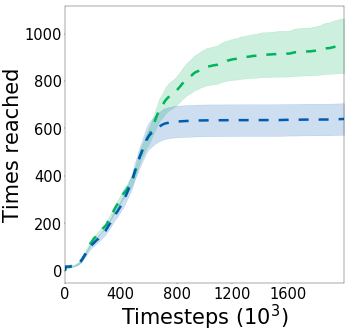} &
			\includegraphics[width=0.43\columnwidth]{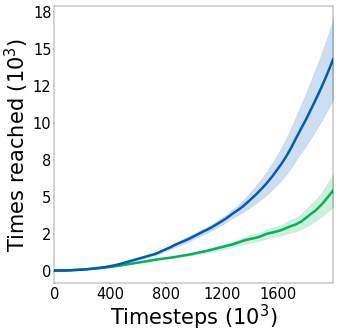}
			\end{tabular}
		\caption{$\tau=100$}		
		\end{subfigure} 
	\end{tabular}		
	\begin{tabular}{c c}
		\begin{subfigure}{0.43\columnwidth}
			\begin{tabular}{c c}
				\includegraphics[width=0.43\columnwidth]{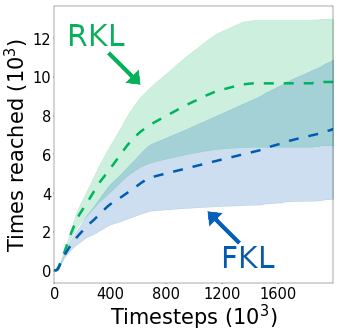} & 
				\includegraphics[width=0.43\columnwidth]{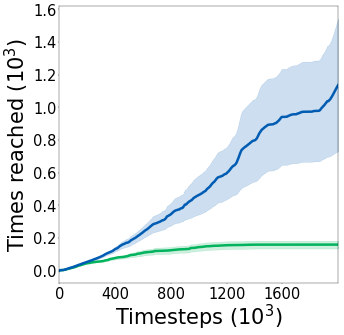}
			\end{tabular}
		\caption{$\tau=10$}		
		\end{subfigure}  
		&		
		\begin{subfigure}{0.43\columnwidth}
			\begin{tabular}{c c}
				\includegraphics[width=0.43\columnwidth]{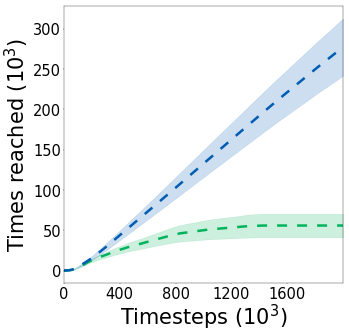} &
				\includegraphics[width=0.43\columnwidth]{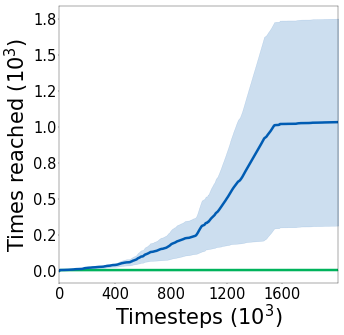}
			\end{tabular}
		\caption{$\tau=1$}		
		\end{subfigure} 
	\end{tabular}	

	\begin{subfigure}{0.43\columnwidth}
		\begin{tabular}{c c}
			\includegraphics[width=0.43\columnwidth]{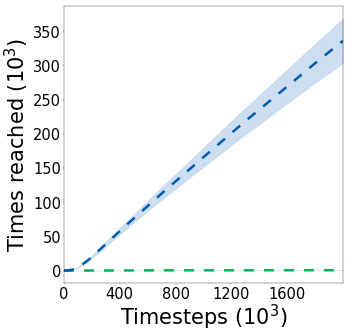} & 
			\includegraphics[width=0.43\columnwidth]{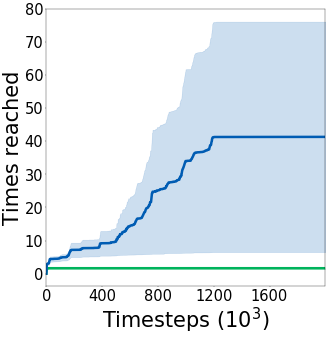}
		\end{tabular}
		\caption{$\tau=0.1$}	
	\end{subfigure}  
	\caption{Cumulative number of times the each exit is reached throughout training for multiple temperatures plotted separately. RKL corresponds to green curves and FKL to blue curves.}
	\label{fig:cont_maze_detail}
\end{figure}

\subsection{Performance} \label{app:addit_perf_results}
Implementation details are the same as in \Cref{sec:main_benchmark}. We perform 30 runs for all hyperparameter settings and plot the mean return averaged over the past 20 episodes. Shaded areas represent standard errors. 

\subsubsection{Continuous-Actions Results}

\begin{figure}[tb!]
  \centering
    \begin{tabular}{c c c}
    \includegraphics[width=0.3\columnwidth]{figs/deep/general/colormaps_ReverseKL.png} 
    &
    \includegraphics[width=0.3\columnwidth]{figs/deep/general/colormaps_ForwardKL.png}
    &
    \hspace{0.02\columnwidth}
    \includegraphics[width=0.26\columnwidth]{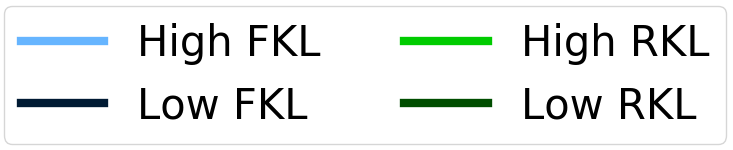}
    \hspace{0.02\columnwidth}
    \end{tabular}  
  \begin{subfigure}[b]{1.0\linewidth}
    \centering
    \begin{tabular}{c c c}
    \includegraphics[width=0.3\columnwidth]{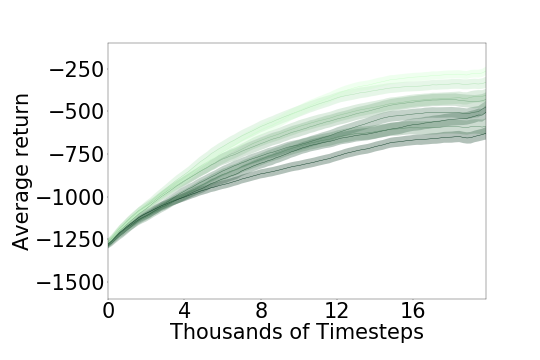} 
    &    
    \includegraphics[width=0.3\columnwidth]{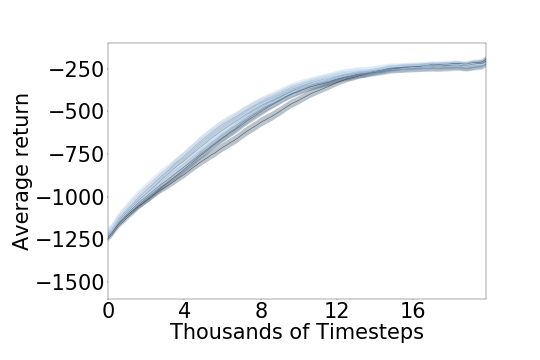}
    &
    \includegraphics[width=0.3\columnwidth]{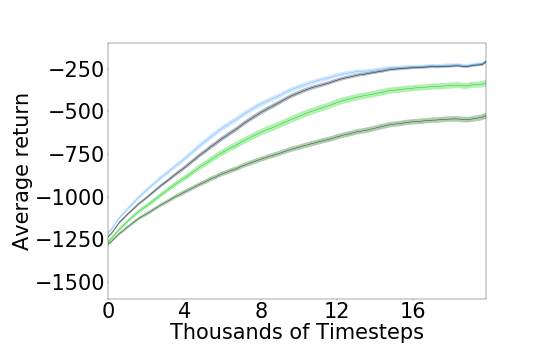}
    \end{tabular} 
    \caption{Pendulum}\label{fig:Pendulum}
  \end{subfigure}
  
   \begin{subfigure}[b]{1.0\linewidth}
    \centering
    \begin{tabular}{c c c}
    \includegraphics[width=0.3\columnwidth]{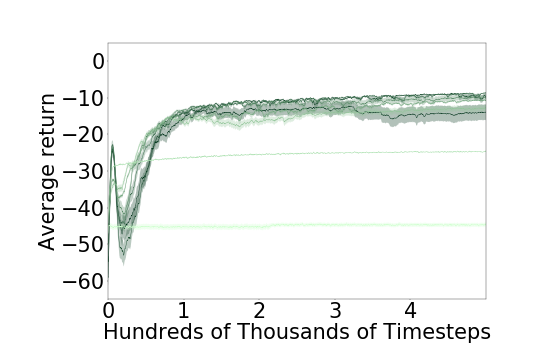} 
    &    
    \includegraphics[width=0.3\columnwidth]{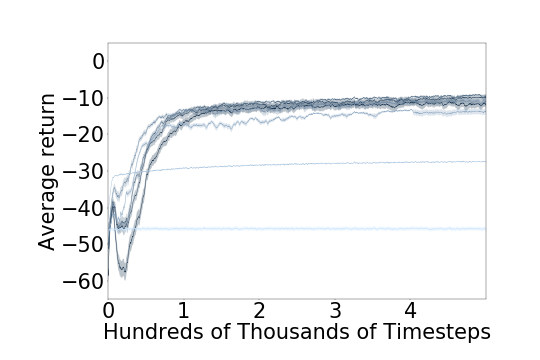}
    &
    \includegraphics[width=0.3\columnwidth]{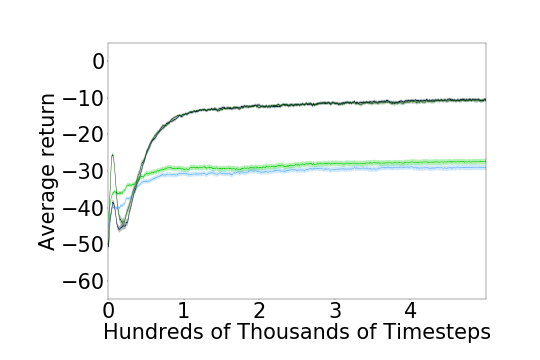}
    \end{tabular} 
    \caption{Reacher}\label{fig:Reacher}
  \end{subfigure}

    \begin{subfigure}[b]{1.0\linewidth}
    \centering
    \begin{tabular}{c c c}
    \includegraphics[width=0.3\columnwidth]{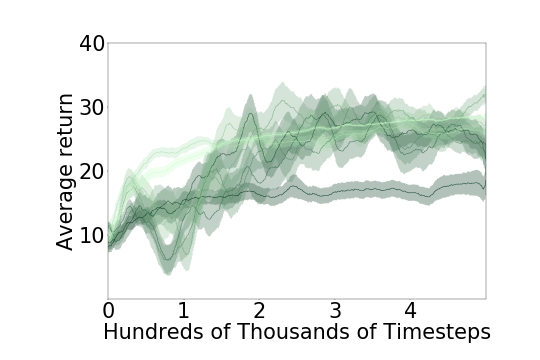} 
    &    
    \includegraphics[width=0.3\columnwidth]{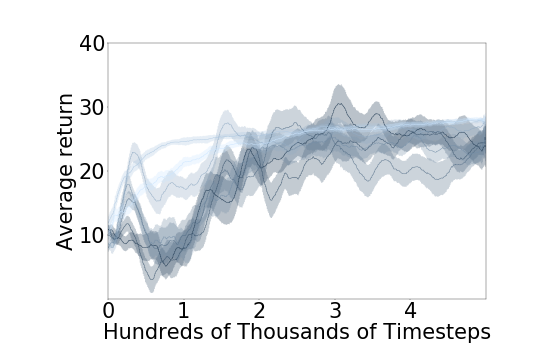}
    &
    \includegraphics[width=0.3\columnwidth]{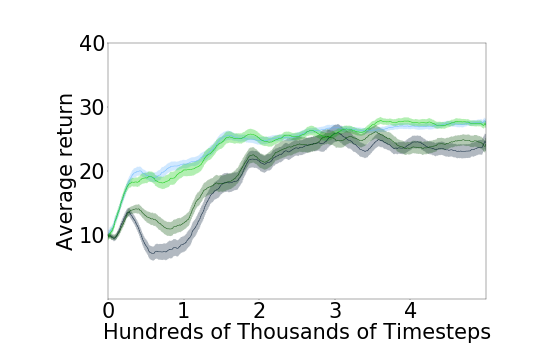}
    \end{tabular} 
    \caption{Swimmer}\label{fig:Swimmer}
  \end{subfigure}

    \begin{subfigure}[b]{1.0\linewidth}
    \centering
    \begin{tabular}{c c c}
    \includegraphics[width=0.3\columnwidth]{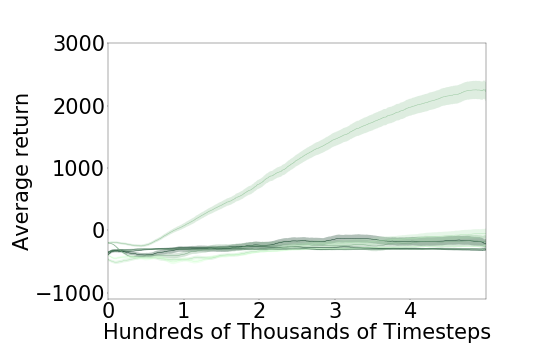} 
    &    
    \includegraphics[width=0.3\columnwidth]{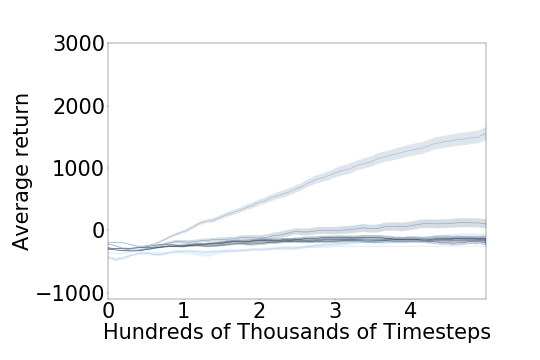}
    &
    \includegraphics[width=0.3\columnwidth]{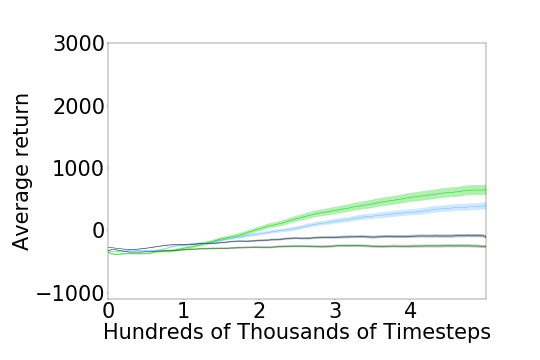}
    \end{tabular} 
    \caption{HalfCheetah}\label{fig:HalfCheetah}
  \end{subfigure} 

  \caption{Continuous-action environments. Shown are the averaged top 20\% performing hyperparameter settings for each algorithm, which are selected by largest area under the last half of the learning curve. In the left we have RKL, in the middle we have FKL and in the right we grouped the high and low temperatures for both FKL and RKL. Temperatures for the curves in the left and middle plots correspond to the respective colormaps and in the rightmost plot the curves correspond to the legend}\label{fig:deep_cont}
\end{figure}

We compare agents on Pendulum \citep{brockman2016openai}, Reacher, Swimmer and HalfCheetah \citep{todorov2012mujoco}, with results shown in \Cref{fig:deep_cont}. We exclude Hard FKL in our comparison since it requires access to $\max_a Q(s,a)$, which is difficult to obtain with continuous actions. The leftmost plot shows all temperatures from RKL, the middle plot shows all temperatures for FKL and the rightmost plot averages all high temperatures and all low temperatures for each divergence. Temperatures $[1.0, 0.5, 0.1]$ were considered high and $[0.05, 0.01, 0.005, 0.001, 0.0]$ were considered low.

Except for Pendulum, which is the simplest environment and where FKL seems to perform more consistently across temperatures, the overall behavior of both divergences is very similar, with performance being much more dependent on the choice of temperature than on the choice of divergence. On Reacher, Swimmer and Pendulum, FKL and RKL with high temperatures have the worst performance, but on HalfCheetah this pattern is reversed, meaning the benefits of entropy regularization are environment dependent.

It is difficult to comment on the importance of the policy parameterization for these experiments relative to our microworld experiments. Any influence from the Gaussian policy parameterization is conflated with function approximation. Moreover, as we will see below, no stark pattern seems to divide continuous and discrete action settings, as one did in our microworld experiments. 

\subsubsection{Discrete-Actions Results}

\begin{figure}[!htb]
  \centering
    \begin{tabular}{c c c}
    \includegraphics[width=0.3\columnwidth]{figs/deep/general/colormaps_ReverseKL.png} 
    &
    \includegraphics[width=0.3\columnwidth]{figs/deep/general/colormaps_ForwardKL.png}
    &
    \hspace{0.02\columnwidth}
    \includegraphics[width=0.26\columnwidth]{figs/deep/general/legend.png}
    \hspace{0.02\columnwidth}
    \end{tabular}  
  \begin{subfigure}[b]{1.0\linewidth}
    \centering
    \begin{tabular}{c c c}
    \includegraphics[width=0.3\columnwidth]{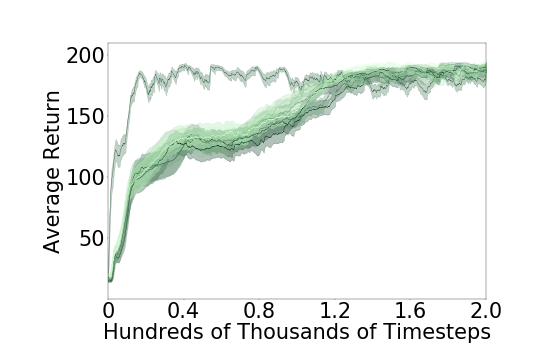} 
    &    
    \includegraphics[width=0.3\columnwidth]{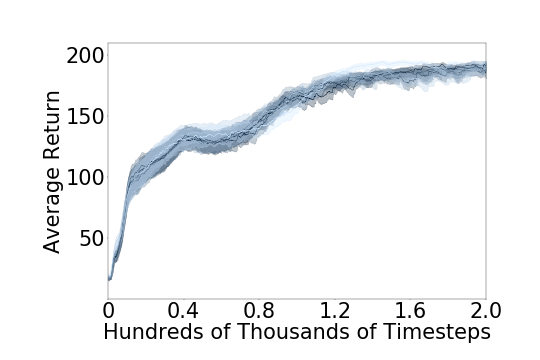}
    &
    \includegraphics[width=0.3\columnwidth]{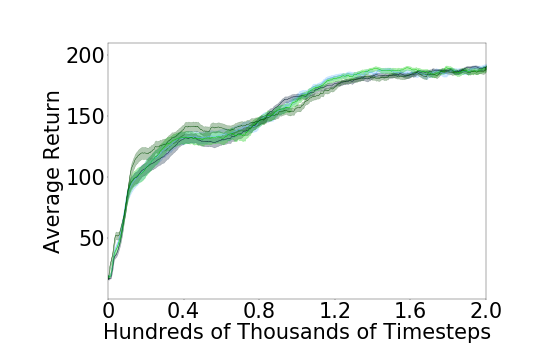}
    \end{tabular} 
    \caption{CartPole}\label{fig:CartPole}
  \end{subfigure}
  
   \begin{subfigure}[b]{1.0\linewidth}
    \centering
    \begin{tabular}{c c c}
    \includegraphics[width=0.3\columnwidth]{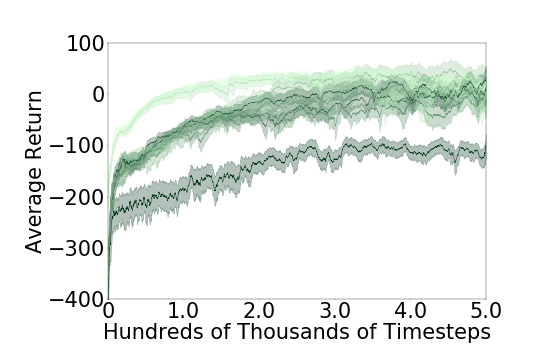} 
    &    
    \includegraphics[width=0.3\columnwidth]{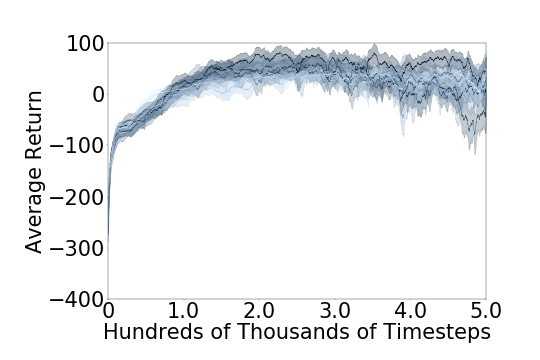}
    &
    \includegraphics[width=0.3\columnwidth]{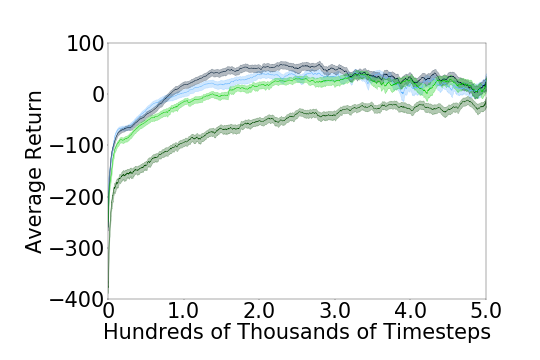}
    \end{tabular} 
    \caption{LunarLander}\label{fig:LunarLander}
  \end{subfigure}

    \begin{subfigure}[b]{1.0\linewidth}
    \centering
    \begin{tabular}{c c c}
    \includegraphics[width=0.3\columnwidth]{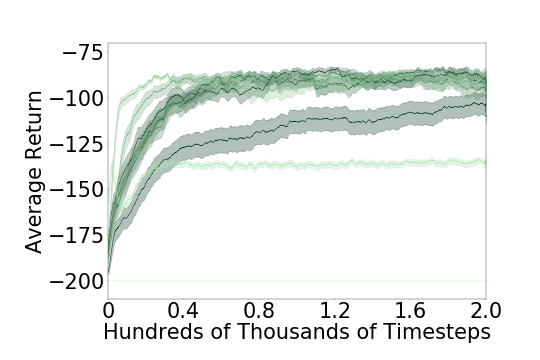} 
    &    
    \includegraphics[width=0.3\columnwidth]{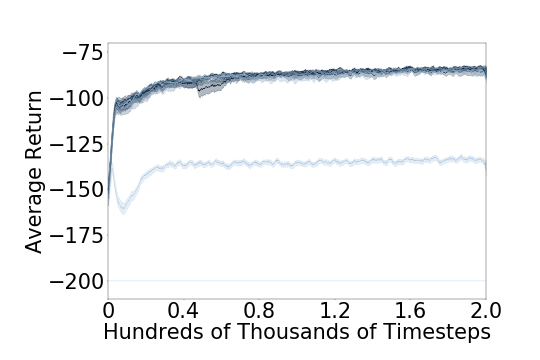}
    &
    \includegraphics[width=0.3\columnwidth]{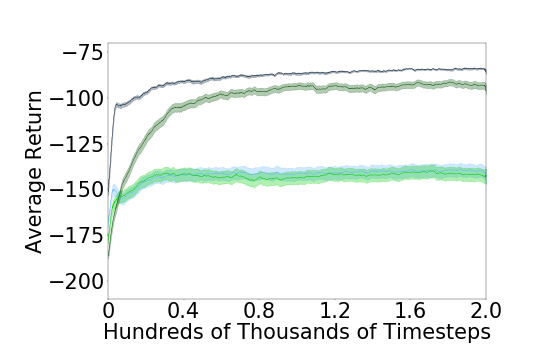}
    \end{tabular} 
    \caption{Acrobot}\label{fig:Acrobot}
  \end{subfigure}

\caption{OpenAI Gym discrete-action environments. Plot settings are identical to Figure \ref{fig:deep_cont}.}\label{fig:open-ai}
\end{figure}

We report results for environments from the OpenAI Gym \citep{brockman2016openai} and MinAtar \citep{young2019minatar}. 
Analogously to the continuous action setting, the OpenAI Gym results reported in \Cref{fig:open-ai} show that both FKL and RKL behave similarly for any given choice of temperature, with the left and middle plots being similar to one another. The main difference is that, for some of the non-optimal temperatures for each problem, FKL seemed to learn faster than RKL, this becomes specially clear on Acrobot. The higher temperatures performed better on CartPole and LunarLander, but worse on Acrobot, confirming that the influence of this hyperparameter is highly environment dependent.

\begin{figure}[th]
  \centering
    \begin{tabular}{c c c}
    \includegraphics[width=0.3\columnwidth]{figs/deep/general/colormaps_ReverseKL.png} 
    &
    \includegraphics[width=0.3\columnwidth]{figs/deep/general/colormaps_ForwardKL.png}
    &
    \includegraphics[width=0.26\columnwidth]{figs/deep/general/legend.png}
    \end{tabular}  
  \begin{subfigure}[b]{1.0\linewidth}
    \centering
    \begin{tabular}{c c c}
    \includegraphics[width=0.3\columnwidth]{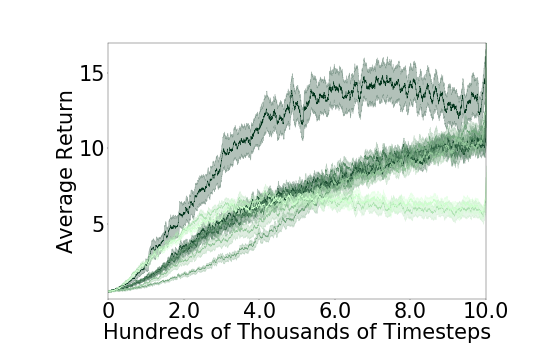} 
    &    
    \includegraphics[width=0.3\columnwidth]{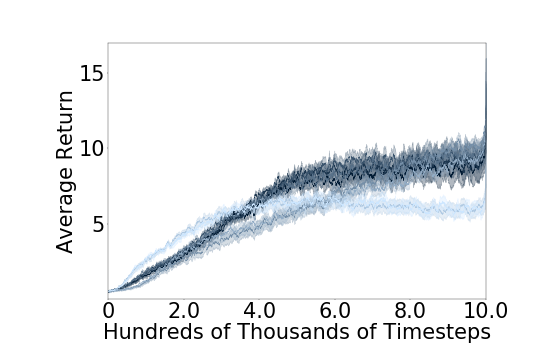}
    &
    \includegraphics[width=0.3\columnwidth]{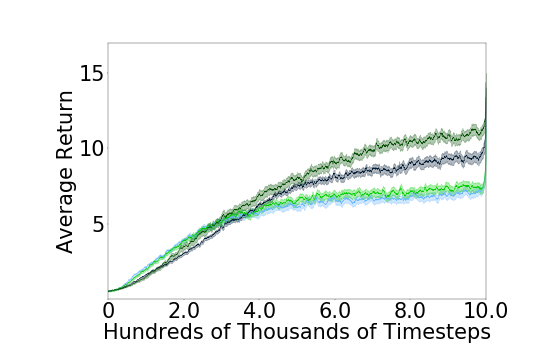}
    \end{tabular} 
    \caption{Asterix.}\label{fig:asterix}
  \end{subfigure}
  
   \begin{subfigure}[b]{1.0\linewidth}
    \centering
    \begin{tabular}{c c c}
    \includegraphics[width=0.3\columnwidth]{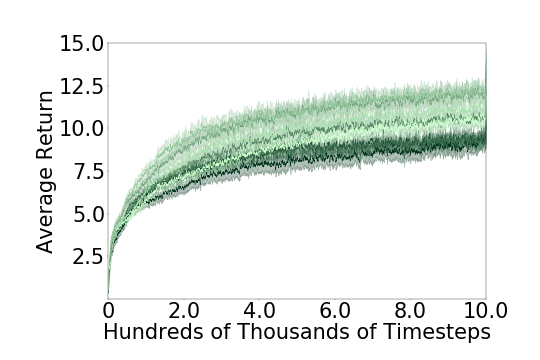} 
    &    
    \includegraphics[width=0.3\columnwidth]{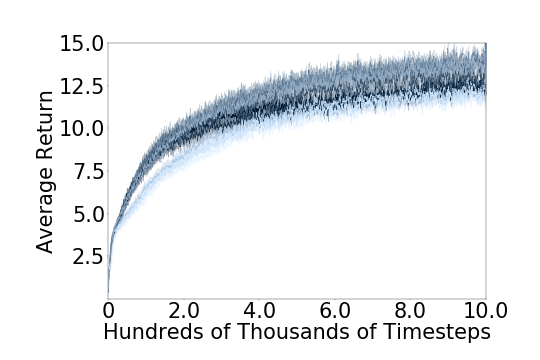}
    &
    \includegraphics[width=0.3\columnwidth]{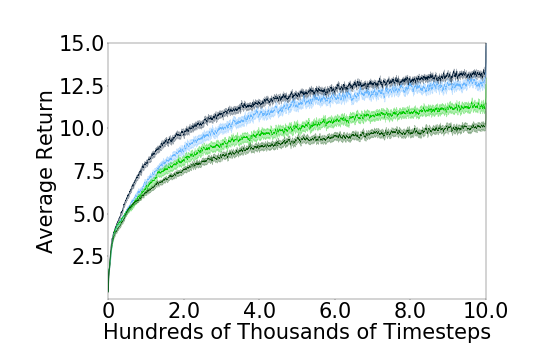}
    \end{tabular} 
    \caption{Breakout.}\label{fig:breakout}
  \end{subfigure}

    \begin{subfigure}[b]{1.0\linewidth}
    \centering
    \begin{tabular}{c c c}
    \includegraphics[width=0.3\columnwidth]{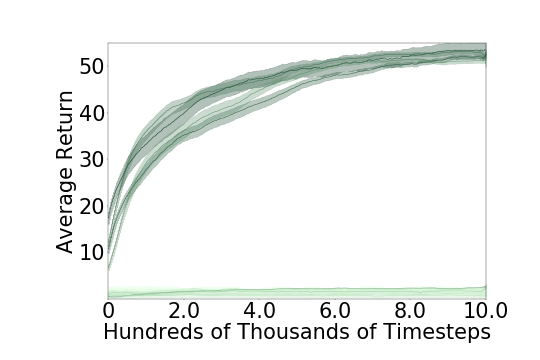} 
    &    
    \includegraphics[width=0.3\columnwidth]{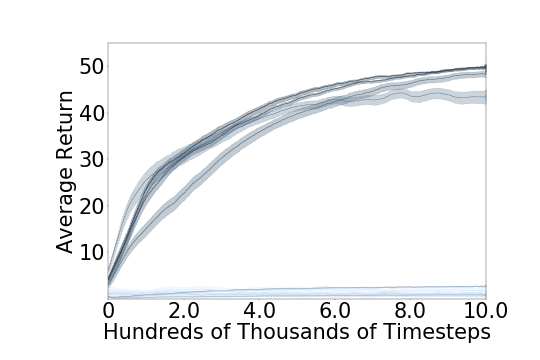}
    &
    \includegraphics[width=0.3\columnwidth]{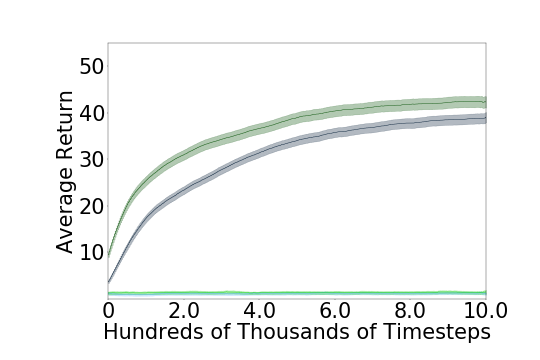}
    \end{tabular} 
    \caption{Freeway.}\label{fig:freeway}
  \end{subfigure}
    \begin{subfigure}[b]{1.0\linewidth}
    \centering
    \begin{tabular}{c c c}
    \includegraphics[width=0.3\columnwidth]{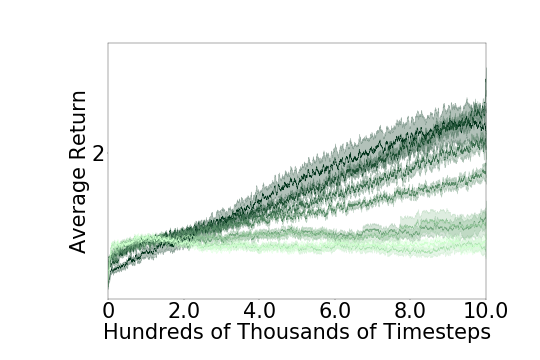} 
    &    
    \includegraphics[width=0.3\columnwidth]{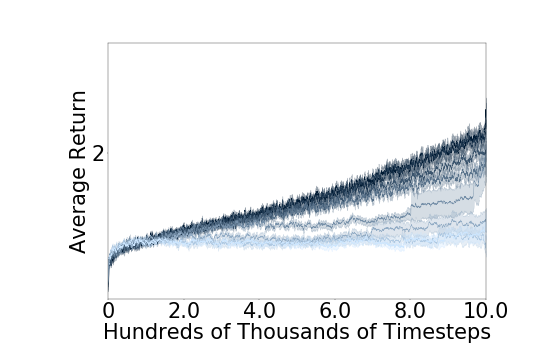}
    &
    \includegraphics[width=0.3\columnwidth]{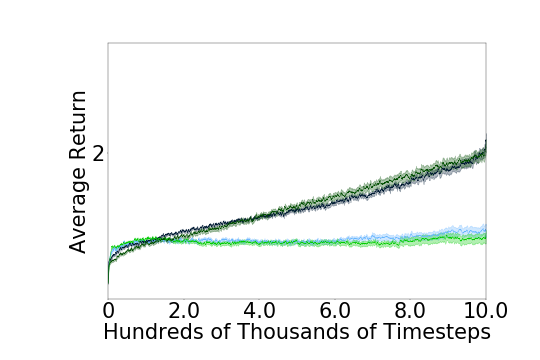}
    \end{tabular} 
    \caption{Seaquest.}\label{fig:seaquest}
  \end{subfigure}

    \begin{subfigure}[b]{1.0\linewidth}
    \centering
    \begin{tabular}{c c c}
    \includegraphics[width=0.3\columnwidth]{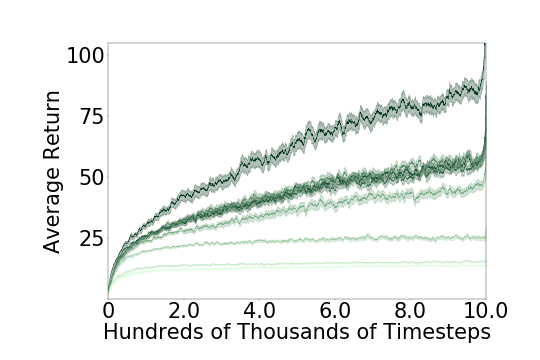} 
    &    
    \includegraphics[width=0.3\columnwidth]{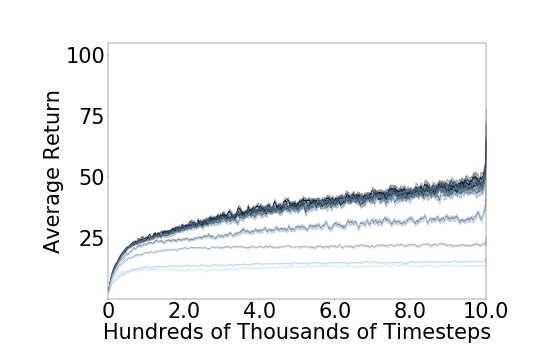}
    &
    \includegraphics[width=0.3\columnwidth]{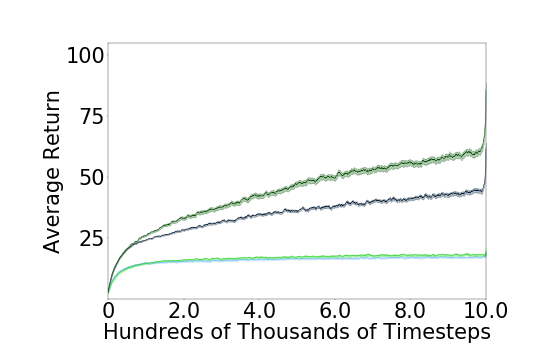}
    \end{tabular} 
    \caption{Space Invaders}\label{fig:space_invaders}
  \end{subfigure}  
  \caption{MinAtar discrete-action environments. Plot settings are identical to those in Figure \ref{fig:deep_cont}. }\label{fig:minatar}
\end{figure}

Finally, for the MinAtar results represented in \Cref{fig:minatar}, there is once more no consistent dominance of either KL over the other: the plots of both FKL and RKL are again highly similar to each other. The slight superiority of FKL for non-optimal temperatures is present only on Breakout and Seaquest. On Asterix, Freeway, Seaquest and Space Invaders highest temperatures performed the worse, but they performed the best on Breakout, showing a pattern opposite of the one seen in the continuous environments of \Cref{fig:deep_cont} and confirming that the optimal temperature is going to vary according to the environment.

\end{document}